\pdfoutput=1
\documentclass{article}
\usepackage{utils/utils}
\usepackage[accepted]{utils/icml2026}

\usepackage{xurl}
\usepackage{tabularx}
\usepackage{array}

\begin{document}
\twocolumn[
\icmltitle{Planar Symmetric Pattern Generation}

\icmlsetsymbol{equal}{*}

\begin{icmlauthorlist}
\icmlauthor{Ning Lin}{equal,sch}
\icmlauthor{Luxi Chen}{equal,sch}
\icmlauthor{Huaguan Chen}{sch}  
\icmlauthor{Jiacheng Cen}{sch} 
\icmlauthor{Chongxuan Li}{sch} 
\icmlauthor{Wenbing Huang}{sch} 
\icmlauthor{Hao Sun}{sch} 
\end{icmlauthorlist}

\icmlaffiliation{sch}{Gaoling School of Artificial Intelligence, Renmin University of China}
\icmlcorrespondingauthor{Hao Sun}{haosun@ruc.edu.cn}

\vskip 0.3in
]

\printAffiliationsAndNotice{\icmlEqualContribution} %

\begin{abstract}
Generating objects with specific symmetries is essential in various real-world scenarios. However, adapting existing 2D continuous representations to enforce planar group symmetry remains a challenge, as the transformation of non-reflective group elements may disrupt continuity.
To overcome this limitation, we propose a symmetrization framework for arbitrary planar groups. Our method transforms any 2D continuous representation into a symmetric one while preserving continuity. We provide the mathematical formulation of this representation, demonstrate its approximation capability for symmetric functions, and detail the construction methodology.
We validate our approach through three visual design tasks (pattern design, paper-cutting design and stylized topology design) and one material design task. Experiments confirm that our representation enables effective symmetry control and demonstrate its broader applicability. Code is available at \href{https://github.com/GLAD-RUC//Sym2D}{https://github.com/GLAD-RUC/Sym2D}.
\end{abstract}

\begin{figure}[!t]
    \centering
    \includegraphics[width=0.95\linewidth]{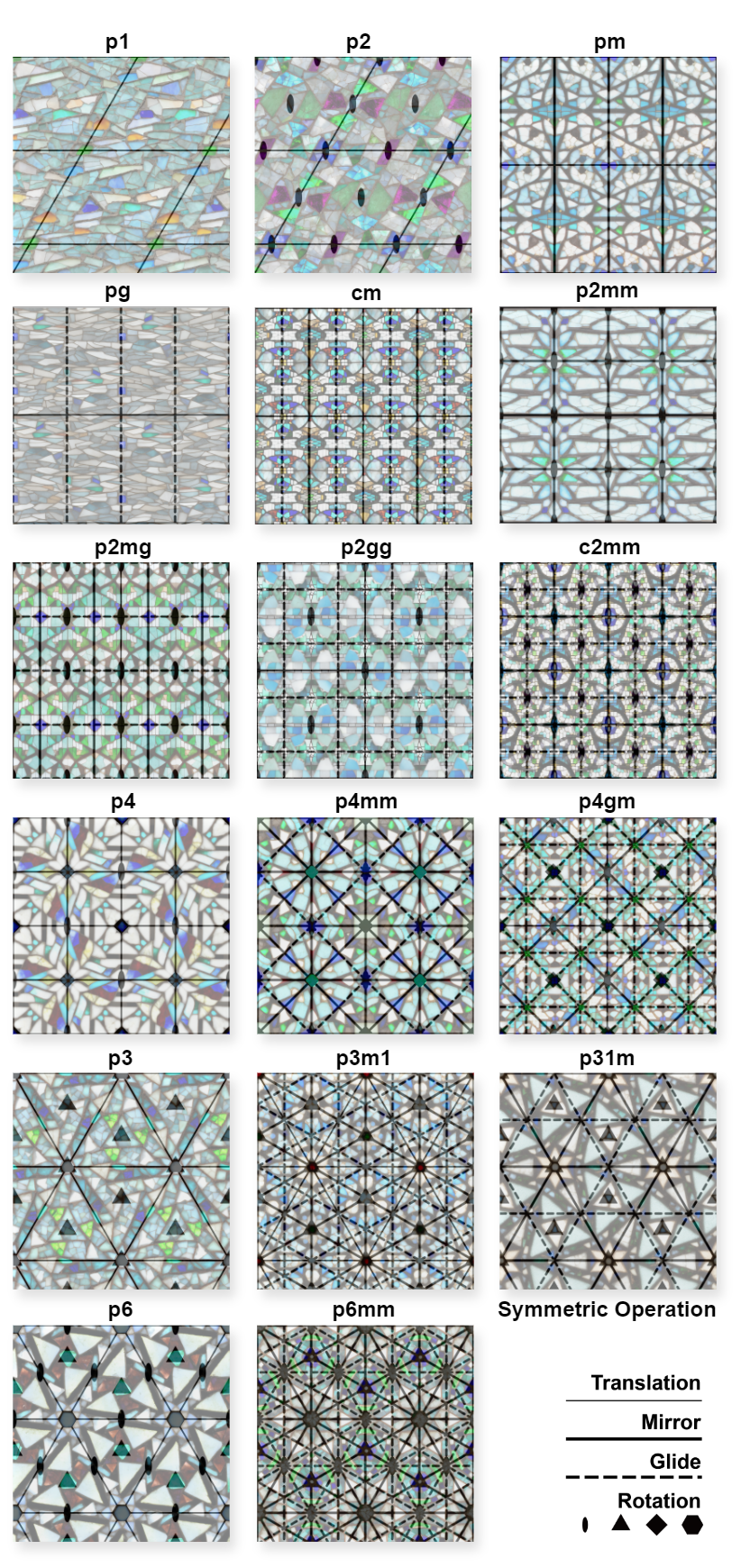}
    \caption{Generated images for the 17 planar groups using the prompt \emph{stained-glass mosaic fragments...} Annotated symmetry transformations demonstrate that our patterns exhibit perceptual preservation of the target symmetries.}
    \label{fig:title_fig}
\end{figure}

\section{Introduction}
\label{sec:introduction}

In visual arts and manufacturing engineering, symmetry plays an important role. It shapes geometric aesthetics and serves as a geometric prior that facilitates manufacturing processes, for example by enabling mold reuse, while also helping satisfy physical requirements. Classic examples range from wallpaper art to reflection-symmetric patterns in paper-cutting crafts formed by folding and unfolding, as well as periodic designs in architectural ornaments and lattice structures. All planar symmetry groups are illustrated in \cref{fig:title_fig}. In these tasks, designers often seek to impose constraints associated with planar symmetry groups while preserving spatial continuity and avoiding abrupt transitions.

However, existing multimodal large language models (MLLMs) fail to produce strictly symmetric images even with rich textual and visual guidance. Previous work such as \citet{bergmannLearningTextureManifolds2017a} studies periodic texture generation, but does not provide an approach with exact symmetry across planar groups. We therefore seek symmetry-embedded representations and begin with the symmetrization of general 2D representations. A naive approach is to define the representation only on the asymmetric unit and extend it to the full plane via group transformations.

Unfortunately, for continuous representation, extending a function from the asymmetric unit by symmetry can introduce boundary discontinuities for non-reflective transformations: points just inside the boundary remain fixed, while points just outside are abruptly mapped to the opposite side. To resolve this, we propose a symmetric continuous representation framework that embeds any planar group into an affine reflection group to preserve boundary continuity, and constructs a continuous $G$-invariant field via a combination of high-symmetry coefficients and low-symmetry bases.

Based on this representation, we develop a unified pipeline for symmetry-constrained controllable generation with diffusion priors. Given a target planar group, we optimize the parameters of the proposed symmetric representation with respect to the loss functions. This separates symmetry constraints from other task-specific objectives: symmetry is handled by the representation, while other task-specific objectives can be imposed through loss functions. As a result, the same framework can be applied across diverse design tasks without collecting symmetric data or training symmetry-specific generative models.

We validate the versatility of our framework through three visual design tasks: 
(i) Pattern Design: Generating symmetric RGB images that align with given text descriptions (\textit{Visual Semantic Constraints}).
(ii) Paper-Cutting Design: Generating globally connected, symmetric binary masks subject to volume constraints and text semantics (\textit{Visual + Connectivity Constraints}).
(iii) Topology Design: Generating connected binary masks optimized for mechanical properties under volume constraints, alongside text-aligned stylized images (\textit{Visual + Connectivity + Mechanical Constraints}). In addition to visual design, we further extend the same idea of symmetry control to metamaterial design based on diffusion prior. Experimental results demonstrate that our framework achieves stable symmetry control under physical constraints.

\textbf{Organization.}
\cref{sec:formulation} details the theoretical construction of the symmetric representation based on affine reflection group embeddings. \cref{sec:computation} presents the computation of the required bases. \cref{sec:gen_obj_and_loss} introduces the unified generative framework and loss functions. \cref{sec:experiment} elaborates on the implementation and results of the four applications.

\section{Preliminaries}
\label{sec:preliminaries}

\textbf{Planar groups.}
In $n$-dimensional space, a Euclidean transformation is defined as the composition of an orthogonal transformation and a translation. The set of all Euclidean transformations constitutes the \textbf{Euclidean group}, denoted as $E(n)$. A \textbf{crystallographic group} $G$ is a discrete subgroup of $E(n)$ that contains $n$ linearly independent translations. Specifically, crystallographic groups are referred to as \textbf{planar groups} when $n=2$, and \textbf{space groups} when $n=3$. We primarily focus on the case where $n=2$.

A classical result states that planar groups are classified into 17 types up to affine coordinate transformations. The simplest planar group is $p1$, generated by translations along two linearly independent directions. For any other planar group $G$, its elements contain nontrivial orthogonal components in addition to translations, and the quotient $G/T(G)$ is finite, where $T(G)$ denotes the translation subgroup of $G$. An important subclass is formed by \textbf{affine reflection groups}, which are generated by affine reflections. There are four such groups: $p2mm$, $p4mm$, $p3m1$, and $p6mm$.

\textbf{Symmetric functions.}
For a planar group $G$, a function $f$ on $\mathbb{R}^2$ is called $G$-invariant if $f(g(\mathbf{x}))=f(\mathbf{x})$ for all $g\in G$. We denote by $C_G(\mathbb{R}^2)$ the space of continuous $G$-invariant functions. Consider a $p1$-invariant function $f$ periodic with respect to two fundamental translations $\mathbf{a}_1$ and $\mathbf{a}_2$. Let $S = \{m\mathbf{b}_1+n\mathbf{b}_2\}$ be the reciprocal lattice, where  $\langle\mathbf{a}_i,\mathbf{b}_j\rangle=2\pi\delta_{ij}$. We choose one representative from each non-zero pair $(\mathbf{k},-\mathbf{k})$. Denoting such a half reciprocal lattice by $S^+$, then $f$ can be written as a Fourier series
\begin{equation}
    f(\mathbf{x}) = \frac{a_{0}}{2}+\sum_{\mathbf{k} \in S^+} [a_{\mathbf{k}}\cos(\langle \mathbf{k}, \mathbf{x}\rangle) + b_{\mathbf{k}}\sin(\langle \mathbf{k}, \mathbf{x}\rangle)].
\end{equation}

\textbf{Asymmetric unit.}
A symmetric function is determined by its values on the asymmetric unit. The asymmetric unit represents the minimal, non-redundant portion of space from which the full periodic structure is obtained through the symmetry operations of the group. For $p1$, any fundamental period constitutes an asymmetric unit. However, for other planar groups, the existence of additional transformations implies that the asymmetric unit is only a portion of the period. Specifically for affine reflection groups, any point in space can be mapped into the asymmetric unit via a finite sequence of reflection transformations.

\begin{figure*}[!t]
    \centering
    \includegraphics[width=1\linewidth]{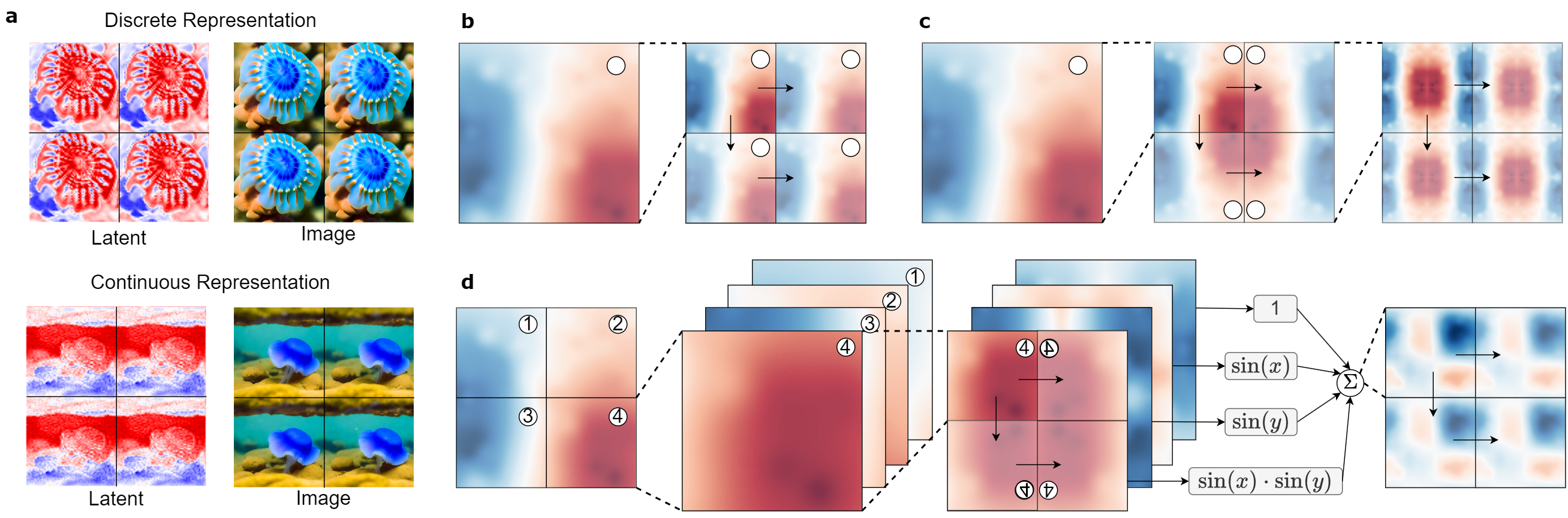}
    \caption{
\textbf{Naive symmetrization vs.\ our symmetrization.}
(a) A discrete representation extended by group transformations can induce abrupt latent transitions and local clustering, whereas continuous fields yield smoother result.
(b) For continuous representations, naive extension may introduce seams near boundaries of asymmetric units for non-reflective groups.
(c) For affine reflection groups, the same strategy is valid because reflections preserve boundary continuity.
(d) Our construction combines high-symmetry coefficients and low-symmetry bases to produce a continuous symmetric function which is seamless across boundaries of asymmetric units.}
    \label{fig:decompostion}
\end{figure*}

\section{Symmetric Continuous Representations}
\label{sec:formulation}

Continuous representations play a critical role in modeling 2D and 3D objects. Continuity prior prevents abrupt variations during optimization, as illustrated in \cref{fig:decompostion}. In this section, we discuss how to symmetrize 2D continuous representations with respect to planar groups.

\subsection{Representations for Affine Reflection Groups}
\label{sec:formulation_reflection_groups}

For planar groups, a naive strategy to adapt general representations to symmetric ones involves directly parameterizing the minimal asymmetric unit and extending it to the rest of the plane via group transformations. 

In the specific case of affine reflection groups, since any point in the plane can be mapped into the asymmetric unit through a finite sequence of reflections and mirror reflections preserve boundary continuity, this approach can directly transform general continuous representations into symmetric ones. However, for other planar groups, this approach disrupts boundary continuity, as illustrated in \cref{fig:decompostion}. To address this, we require a method that bridges continuous modeling across different symmetries.

\begin{example}[Decomposition for $p1$ cases]
\label{ex:decomposition_for_p1}

Consider the $p1$ symmetry and its Fourier bases. Let $c_1,s_1$ correspond to $\mathbf{k}=(1,0)$ and $c_2,s_2$ to $\mathbf{k} = (0,1)$. Using the multiple-angle formulas $\cos(nx)=T_n(\cos x)$ and $\sin(nx)=U_{n-1}(\cos x)\sin x$ ($T_{n}$ and $U_{n}$ denote the Chebyshev polynomials of the first and second kind) and trigonometric identities, truncated Fourier series can be written as a polynomial in $\mathbb{R}[c_1,c_2,s_1,s_2]$, denoted by $P_{\mathrm{torus}}$. With $\eta_1=1$, $\eta_2=s_1$, $\eta_3=s_2$, and $\eta_4=s_1s_2$, every $f\in P_{\mathrm{torus}}$ admits
\begin{equation}
    f(\mathbf{x})=\sum_{i=1}^4 p_i(c_1,c_2)\eta_i(\mathbf{x}),
\end{equation}
where $p_i$ are bivariate polynomials.

Considering the unit basis vectors $\mathbf{e}_{1}, \mathbf{e}_{2}$ as fundamental translations for the $p1$ symmetric function, the coefficient functions $p_{i}(c_{1},c_{2})$ actually exhibit higher $p2mm$ symmetry, while the basis terms $g_{i}$ correspond to the lower $p1$ symmetry. Consequently, the low-symmetry function is successfully decomposed into a combination of high-symmetry coefficients and low-symmetry bases.
\end{example}

This example provides the intuition for constructing symmetric continuous representations. To generalize this method to arbitrary planar groups via affine reflection groups, we address three key questions:
(i) \textbf{Existence of Higher Symmetry:} For any planar group $G$, does there always exist an affine reflection group $W_{a}$ serving as a supergroup?
(ii) \textbf{Existence of Decomposition:} If such a group exists, does the decomposition into high-symmetry coefficients and low-symmetry bases always hold for any $G$-invariant continuous function?
 (iii) \textbf{Computation of Bases:} How do we explicitly construct the low-symmetry bases?

We answer the first two existence questions using rigorous mathematical theory in \cref{sec:hironaka_decomposition_r2}, and present the computational approach for the basis in \cref{sec:computation}.

\subsection{Hironaka Decomposition}
\label{sec:hironaka_decomposition_r2}

We first address the problem of symmetry group inclusion. While the general answer is negative, we can transform any planar group into a subgroup of an affine reflection group up to conjugation by an invertible linear transformation.

\begin{theorem}
\label{thm:existence_of_higher_Wa_n=2}
Any planar group $G$ is conjugate to a subgroup of some affine reflection group $W_{a}$, i.e., there exists $A\in \GL(2,\mathbb{R})$ such that $AGA^{-1}\subseteq W_{a}$.
\end{theorem}

In the subsequent discussion, we will focus solely on the symmetries obtained after such a transformation, as the corresponding parameterizations can be mutually converted via invertible transformations.

We can construct $G$-invariant continuous representations by representation of affine reflection group $W_a$. Concretely, we start from an underlying continuous representation and extend it via group transformations to obtain a $W_a$-invariant coefficient function which is continuous by construction.
Given $r$ fixed $G$-invariant functions $\eta_1,\dots,\eta_r$ and $W_a$-invariant functions $\{h_i\}_{i=1}^r$, we obtain a $G$-symmetric continuous representation in the form
\begin{equation}
\label{eq:hironaka_param}
f(\mathbf{x}) \;=\; \sum_{i=1}^{r} h_i(\mathbf{x})\,\eta_i(\mathbf{x}).
\end{equation}
For $G$-invariant continuous functions, the existence of a high-symmetry coefficients and low-symmetry bases decomposition implies the approximation power of the form \cref{eq:hironaka_param}. Due to the algebraic structure of continuous functions, an exact decomposition may not exist in general. Instead, we establish an approximation guarantee.

\begin{theorem}
\label{thm:hironaka_decomposition_func}
Let $r=[W_a:G]$, and let $\Omega$ denote the unit cell. 
Then there exist $r$ fixed $G$-invariant basis functions $\eta_1,\dots,\eta_r$ such that
for any $f\in C_G(\mathbb{R}^2)$ and any $\epsilon>0$, there exist
$h_1,\dots,h_r\in C_{W_a}(\mathbb{R}^2)$ satisfying
\begin{equation}
\int_{\Omega}\bigg| f(\mathbf{x})-\sum_{i=1}^{r} h_i(\mathbf{x})\,\eta_i(\mathbf{x}) \bigg|^2\,\dd\mathbf{x} \;<\; \epsilon.
\end{equation}
\end{theorem}

Consequently, given $r$ parameterized continuous functions $\varphi_{\theta_i}\in C(\mathbb{R}^2)$, we can obtain an expressive
parameterization of $G$-invariant continuous functions. See \cref{alg:eval_G} for the evaluation procedure.

\begin{algorithm}[!t]
\caption{Parameterized planar $G$-invariant continuous function $f$ evaluation at query point}
\label{alg:eval_G}
\KwIn{Continuous parameterization $\theta \mapsto \varphi_{\theta}$; parameters $\{\theta_i\}_{i=1}^{r}$; query point $\mathbf{p}=(x,y)$}
\KwOut{G-inv func $f(\mathbf{p})$}
\tcp{Map $\mathbf{p}$ to asym unit of $W_a$}
$\hat{\mathbf{p}} \leftarrow \texttt{reflect\_to\_asym\_cell}(\mathbf{p})$\;
$f \leftarrow 0$ \;
\For{$i=1$ \KwTo $r$}{
    \tcp{Get coefs on asym unit}
    $h_i \leftarrow \varphi_{\theta_i}(\hat{\mathbf{p}})$\;
    \tcp{Comb with fixed bases}
    $f \leftarrow f + h_i \cdot \eta_i(\mathbf{p})$\;
}
\Return{$f$}
\end{algorithm}

Regarding the construction of bases, we analyze trigonometric polynomials and then extend the results to continuous functions. Recall that the polynomial ring $P_{\mathrm{torus}}$ is constructed specifically to encode the periodicity of the translation subgroup $T(G)$. For $T(G)$ is a normal subgroup of $G$, the $G$-action is closed on $P_{\mathrm{torus}}$, and since $T(G)$ acts trivially, the invariant subring $P_{\mathrm{torus}}^{G}$ (invariant truncated Fourier series) is determined entirely by the induced action of the finite quotient group $G/T(G)$. Although $P_{\mathrm{torus}}$ is a quotient ring and the action is non-linear, it remains a commutative domain. The fact allows us to apply results analogous to classical invariant theory (see, e.g., \citet{sturmfels2008algorithms}) if we relax the homogeneity requirement.

The Hironaka decomposition guarantees that there exist primary invariants $\theta_1,\theta_2$ and a finite set of secondary invariants $\eta_1,\dots,\eta_r$ such that any $f\in P_{\mathrm{torus}}^{G}$ admits the unique decomposition
$
f=\sum_{i=1}^{r} p_i(\theta_1,\theta_2)\,\eta_i.
$
One can choose $\theta_{1}, \theta_{2}$ to be the primary invariants of $P_{\mathrm{torus}}^{W_{a}}$, which in turn yields the desired
low-symmetry bases $\{\eta_i\}$. \cref{ex:decomposition_for_p1} illustrates this construction in the $p1$ case.

\subsection{General Theory in \texorpdfstring{$\mathbb{R}^n$}{Rn}}
\label{sec:genral_theory_in_rn}

The polynomial Hironaka decomposition extends to arbitrary dimensions
(see \cref{thm:hironaka_decomposition_poly}). By the same argument as
in the $n=2$ case, the approximation result of
\cref{thm:hironaka_decomposition_func} also holds whenever $G$ admits,
up to conjugation, an affine reflection supergroup $W_a$
(see \cref{thm:hironaka_func_with_higher_sym}). Thus, the only
dimension-dependent issue is the existence of such a supergroup,
which is governed by the conjugacy classes of finite subgroups of
$\GL(n,\mathbb{Q})$ (see \cref{thm:dimension_judge}).

However, every crystallographic group admits such a supergroup for $n\leq 3$
(see \cref{thm:affine_reflection_supergroup_n_eq_2_3}), whereas this
fails in general for $n>3$ (see
\cref{thm:no_affine_reflection_supergroup_n_neq_7_8,thm:no_affine_reflection_supergroup_n_eq_7_8}).
This restriction is sufficient for our practical applications, which
concern periodic structures in dimensions $n\leq 3$.

\section{Computation of Basis}
\label{sec:computation}

For a given planar group $G$, according to \cref{thm:hironaka_decomposition_func}, to obtain a continuous representation, we first need to identify the high-symmetry group $W_{a}$. For the case of $n = 2$, this can be determined via table lookup. Next, we need to calculate the secondary invariants of $P_{\mathrm{torus}}^G$ with respect to the low-symmetry bases of $W_{a}$. A construction via Fourier series can be found in \cref{ex:decomposition_for_p1}. We next discuss how to perform the computation using existing results.

Assume that $G$ is a subgroup of an affine reflection group $W_{a}$, and $G$ contains a subgroup $H$ such that $G$ is generated by $H$ and an element $\tau$. Suppose we have obtained the bases for $H$ with respect to $W_{a}$. If the action of $\tau$ maps each secondary invariant $\eta_{i}$ of $H$ to either itself or its negation, then the secondary invariants for $G$ consist of precisely those $\eta_{i}$ that remain invariant under $\tau$.

Calculations for the primary and secondary invariants of symmorphic planar groups (planar groups without glide reflections) can be found in Tab.~7 in \citet{kimRingInvariantReal2001}. We focus on the treatment of non-symmorphic planar groups. Here, the additional generator $\tau$ is a glide reflection, and the subgroup $H$ consists of all orientation-preserving operations in $G$. The computation procedure is detailed in \cref{sec:sec_inv_non-symmorphic} and \cref{sec:adj_conventions}, and the resulting invariants are summarized in \cref{sec:sec_inv_table}.

\section{Generative Objectives and Loss Constraints}
\label{sec:gen_obj_and_loss}

\begin{figure}[!t]
    \centering
    \includegraphics[width=1\linewidth]{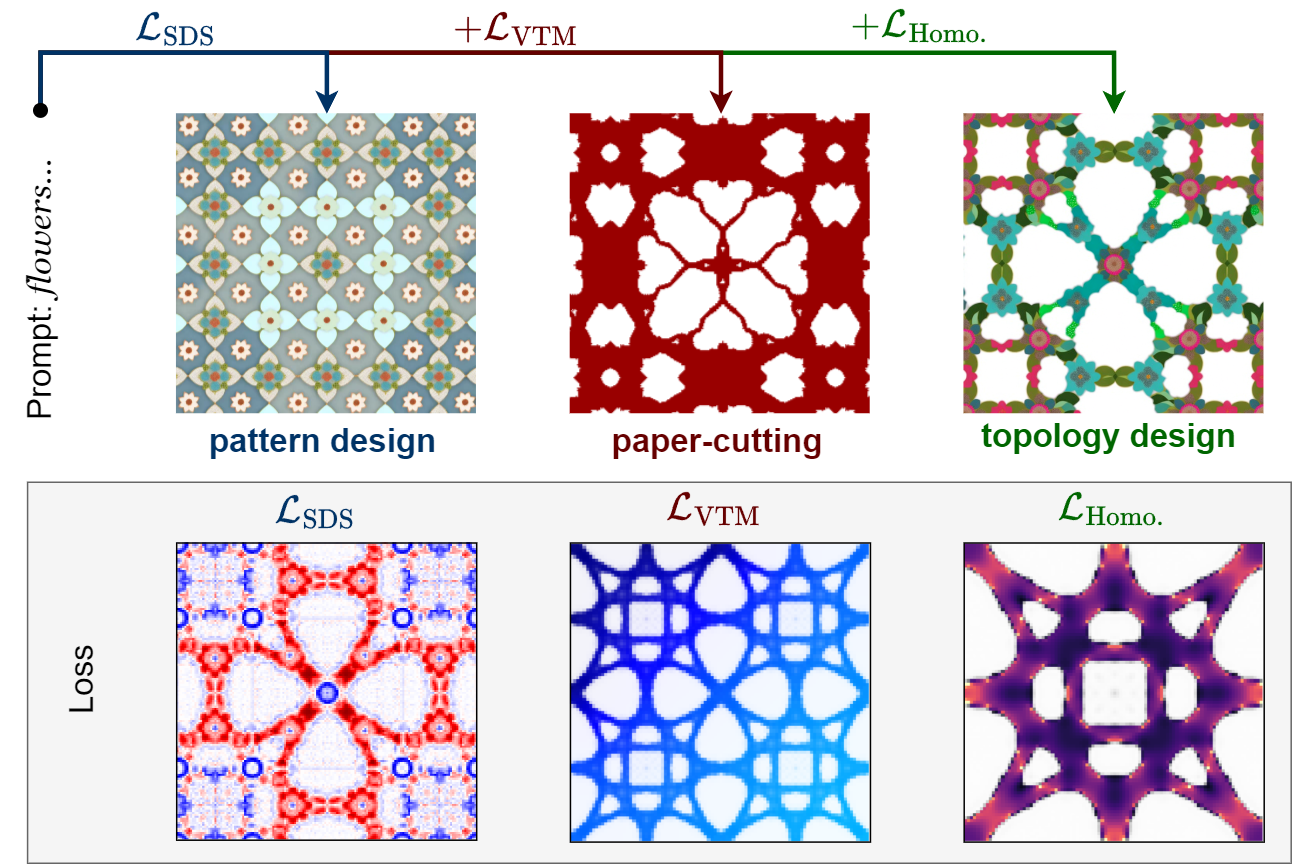}
    \caption{\textbf{Unified framework across three tasks.} Given the same prompt, pattern design optimizes $\mathcal{L}_{\mathrm{SDS}}$. Paper-cutting adds the VTM penalty $\mathcal{L}_{\mathrm{VTM}}$ to encourage connectivity. Topology design further incorporates $\mathcal{L}_{\mathrm{Homo}}$ to maximize the effective bulk modulus evaluated through homogenization.}
    \label{fig:loss}
\end{figure}

In this section, we formulate downstream generation over our symmetric continuous representation, combining diffusion priors and other task-specific constraints. \cref{fig:loss} summarizes our progressively complex downstream tasks and the corresponding losses used to impose different constraints.

\subsection{Pattern Design}
\label{sec:pattern_design}

\textbf{Objective.}
Given a planar group $G$ and a text condition $y$, we consider the general task of symmetric pattern generation: generating 2D patterns $I$ that satisfy specific group symmetry constraints. We construct symmetric continuous representations within the latent space. Leveraging the symmetrization technique proposed in the previous section, we transform a general continuous representation into a continuous $G$-invariant vector field $f_\theta:\mathbb{R}^2\rightarrow \mathbb{R}^{C}$ where $C$ denotes the number of latent channels, and $\theta$ represents the learnable parameters. We obtain the finite-dimensional latent tensor $z_{\theta} \in\mathbb{R}^{H\times W\times C}$ by evaluating the symmetric vector field on a regular grid.

\textbf{Distribution Constraint.}
In the latent space, we employ Score Distillation Sampling (SDS) \citep{poole2023dreamfusion} as the generative objective to align the generated image with the conditional distribution of the diffusion model given $y$. The gradient of the SDS loss is formulated as
\begin{equation}
    \nabla_\theta \mathcal{L}_{\mathrm{SDS}}
    \ \propto\
    \mathbb{E}_{t,\epsilon}\left[
    w(t)\left(\widehat{\epsilon}_\phi(z_t,t,y)-\epsilon\right)\,
    \nabla_\theta z_\theta
    \right],
\end{equation}
where $w(t)$ is a weighting term, $z_t = \alpha_t z_\theta + \sigma_t \epsilon$ represents the result of forward noise injection on $z_{\theta}$ at step $t$, $\alpha_t$ and $\sigma_t$ are the noise schedule parameters, and $\widehat{\epsilon}_\phi(z_t,t,y)$ is computed via classifier-free guidance as a weighted combination of the unconditional and conditional scores.

\textbf{Latent Vector Decoding.}
In the output phase, we utilize a convolutional decoder to decode $z_{\theta}$ into an image. Before decoding, we perturb $z_{\theta}$ with moderate noise and denoise it using the diffusion model to improve generation quality. Since convolutional operations are translation equivariant and the learned diffusion prior encourages coherent symmetric completions, the planar group symmetry of the image constrained by the symmetric initialization is approximately preserved after denoising and decoding, as experimentally evaluated in \cref{sec:pattern_results}. Compared with pixel-space modeling, latent-space parameterization is more computationally and memory efficient, supports higher-resolution generation, and often leads to more stable optimization in practice.

\subsection{Paper-Cutting Design}
\label{sec:paper_cutting_design}

\begin{figure}[!t]
    \centering
    \includegraphics[width=1\linewidth]{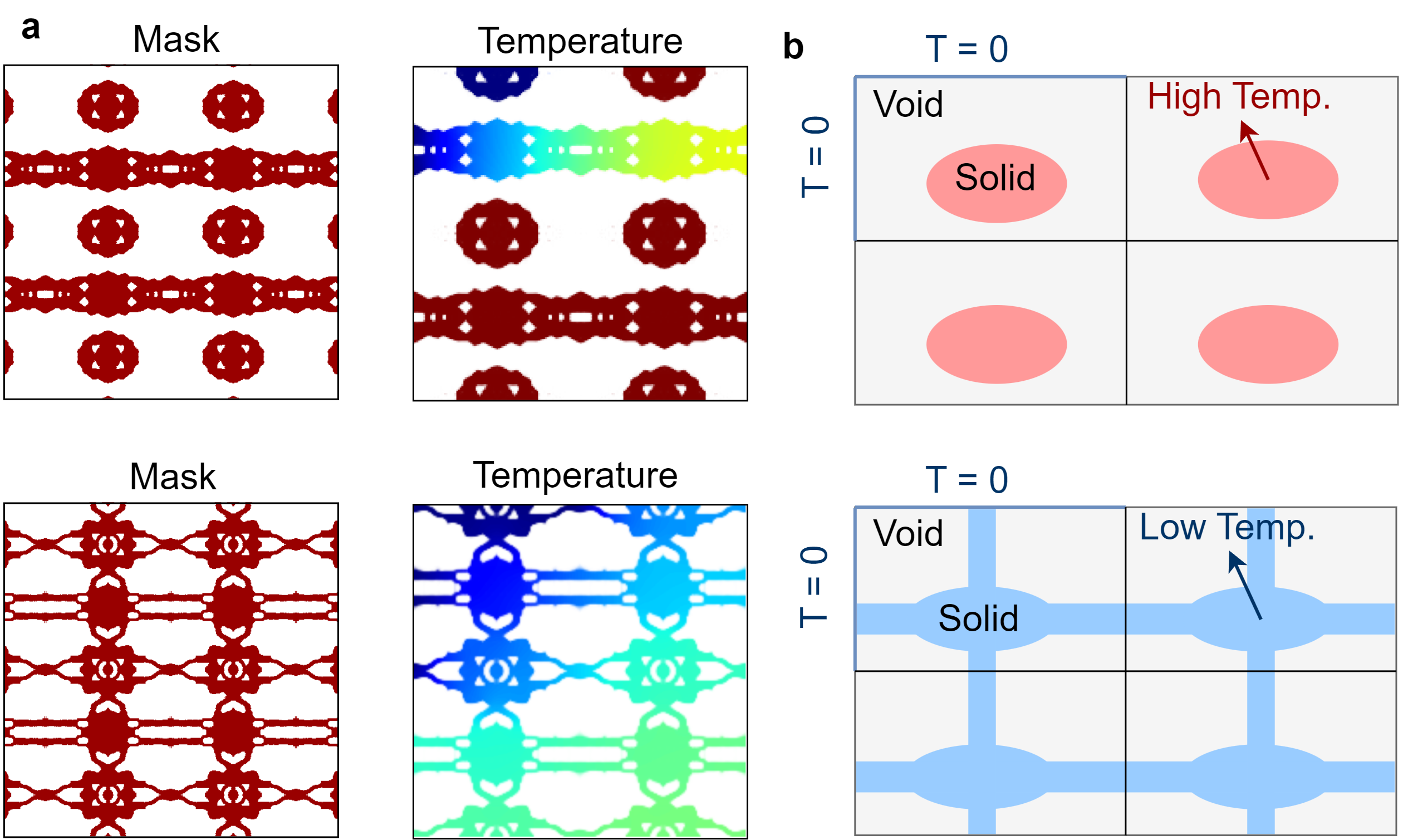}
    \caption{\textbf{VTM for connectivity.} Disconnected solid islands cannot dissipate heat to the sink on $\Gamma$ ($T=0$) and thus become high-temperature regions: (a) generative examples and (b) an abstract schematic. Minimizing the VTM loss penalizes these hot components, promoting global connectivity.}
    \label{fig:vtm}
\end{figure}

\textbf{Objective.} Given a planar group $G$ and a text condition $y$, we extend the framework to generate periodic binary masks $m$ for paper-cutting.
Here, we specifically restrict $G$ to the four affine reflection groups, as they correspond to valid fold-and-cut operations in physical fabrication. 
The goal is to optimize a mask $m$ where $m=1$ (solid) and $m=0$ (void) satisfy both semantic condition $y$ and maintain global connectivity to prevent structural detachment.

\textbf{Mask Parameterization.}
We obtain the structural mask through differentiable segmentation. We define fixed solid and void colors, $c_{\mathrm{solid}}$ and $c_{\mathrm{void}}$, and the channel-averaged distance
$d(f_\theta,c_{\mathrm{void}})=\frac{1}{C}\sum_{k=1}^{C}\lvert(f_\theta)_k-(c_{\mathrm{void}})_k\rvert$.
The soft density is
$\rho_\theta=\sigma((d(f_\theta,c_{\mathrm{void}})-d_0)/\tau)$,
where $\tau$ controls sharpness and $d_0$ the threshold. Thus, values close to $c_{\mathrm{void}}$ yield $\rho_\theta\approx0$, whereas distant values yield $\rho_\theta\approx1$. Unlike pattern design, we bypass the pixel-space decoder and optimize the mask directly in latent space, leveraging the geometric consistency in \cref{sec:pattern_design}.

To guide visual style, we construct the pseudo-binary latent vector
$z_\theta^{\mathrm{bin}}
=c_{\mathrm{solid}}\rho_\theta+c_{\mathrm{void}}(1-\rho_\theta)$.
Applying SDS directly to binary renderings often destabilizes optimization. We therefore evaluate the SDS loss on a progressively weighted mixture of $z_\theta$ and $z_\theta^{\mathrm{bin}}$. We further adapt the pre-trained diffusion model to a paper-cutting dataset using Low-Rank Adaptation (LoRA) \citep{hu2022lora}, aligning the diffusion prior with the domain distribution.

\textbf{Connectivity Constraint.} 
To prevent unmanufacturable isolated islands, we employ the virtual temperature method (VTM) \citep{liStructuralTopologyOptimization2016} to promote global connectivity. We model the material distribution as a heat conduction system within a domain $\Omega$: solid regions act as conductors that generate and transport heat, while void regions function as insulators. By placing a heat sink at a boundary $\Gamma$, any disconnected component unable to dissipate heat will accumulate high temperature. Therefore, minimizing the maximum temperature penalizes solid islands. 

Formally, the virtual temperature field $T$ is governed by the steady-state Poisson equation $-\nabla \cdot (k(\rho_\theta) \nabla T) = s(\rho_\theta)$ subject to Dirichlet boundary conditions $T=0$ on $\Gamma$ and adiabatic conditions elsewhere. The thermal conductivity $k(\rho_\theta)$ and heat source $s(\rho_\theta)$ are coupled to the density via SIMP interpolation, e.g., $k(\rho)=k_{\min}+(k_0-k_{\min})\rho^p$, where $k_0$ denotes the conductivity of the solid material and $k_{\min}$ is a small conductivity used to avoid degeneracy in void regions. The penalization $p>1$ encourages $\rho$ to converge toward binary values. We penalize the maximum temperature to enforce connectivity
\begin{equation}
    \mathcal{L}_{\mathrm{VTM}} = \left( \frac{1}{|\Omega|} \int_{\Omega} T(\rho_{\theta})(\mathbf{x})^p \dd\mathbf{x} \right)^{1/p},
\end{equation}
where the $p$-norm approximates the max operator, and gradients are efficiently computed via the adjoint method.

To ensure global connectivity of a periodic mask, the choice of the domain $\Omega$ and the boundary $\Gamma$ is crucial. Enforcing connectivity only within a single unit cell is insufficient, as it may still permit disconnected strip-like patterns. We therefore set $\Omega$ to be a $2\times2$ supercell and impose an intersection constraint with $\Gamma$ inside one constituent cell. In \cref{thm:2x2_gamma_connectivity}, we prove that this strategy promotes the global connectivity of the mask. Choosing $\Gamma$ as an arbitrary interior point of $B$ would unnecessarily force the mask to pass through a prescribed location. To avoid this over-constraint, we take $\Gamma$ to be two boundary segments of the unit cell. As illustrated in \cref{fig:vtm}, isolated solid islands in a periodic mask exhibit high temperatures, and the VTM loss penalizes these hot regions, thereby promoting global connectivity.

\textbf{Volume Constraint.}
To regulate material usage and control the pattern's sparsity, we introduce a volume constraint. Denoting the target volume fraction as $\rho_0 \in (0,1)$, we define the loss over the unit cell $\Omega$ as 
\begin{equation}
    \mathcal{L}_{\mathrm{vol}} = \left( \frac{1}{|\Omega|} \int_{\Omega}\rho_\theta(\textbf{x}) \,\dd\mathbf{x} - \rho_0 \right)^2.
\end{equation}
This term enables systematic control over the solid-to-void ratio, allowing for stylistic adjustments.

\subsection{Topology Design}
\label{sec:topo_design}

\textbf{Objective.}
Distinct from the previous tasks, we must simultaneously synthesize (i) a binary structural mask $m$ defining the geometric topology, and (ii) the texture content $I$ within the mask, effectively treating $m$ as an alpha transparency channel. The objective is to maximize the effective mechanical properties under symmetry and volume constraints while maintaining aesthetic appeal and stylistic consistency. 

Provided the background color, the image parameterization follows the pattern design, while the mask parameterization adopts the segmentation approach from \cref{sec:paper_cutting_design}. Next, we discuss the mechanical loss on the mask.

\textbf{Mechanical Constraints.} 
For periodic structures, stiffness is typically evaluated via the effective elastic properties. We employ homogenization-based topology optimization, calculating the equivalent elastic tensor
\begin{equation}
    E^H_{ijkl}(\rho ) = \frac{1}{|\Omega|}\int _{\Omega } E_{pqrs} \epsilon _{pq}^{A(ij)}\epsilon _{rs}^{A(kl)} \dd\Omega,
\end{equation}
where $\epsilon _{rs}^{A(kl)}$ represents the actual strain field induced within the unit cell $\Omega$ under the $(kl)$-th unit test strain. With the unit test strains and periodic boundary conditions, the structural equilibrium equations in $\Omega$ are given by
\begin{equation}
    \boldsymbol{\nabla}_{S}^{\top} \boldsymbol{\sigma} = \mathbf{0}, \quad \boldsymbol{\epsilon} = \boldsymbol{\nabla}_{S}\mathbf{u}, \quad \boldsymbol{\sigma} = \mathbf{D}(\rho_{\theta})\boldsymbol{\epsilon}.
\end{equation}

Here, $\boldsymbol{\nabla}_{S}$ is the symmetric gradient operator and
$\mathbf{D}(\rho)=E(\rho)\mathbf{D}_{0}$, where
$E(\rho)=E_{\min}+\rho^p(E_0-E_{\min})$ and $\mathbf{D}_{0}$ is the
plane-stress constitutive matrix for unit Young's modulus. 
In the 2D case, using Voigt notation
$(11\to 1,22\to 2,12\to 3)$, maximizing the effective bulk
modulus corresponds to maximizing
$c=E^H_{11}+E^H_{12}+E^H_{21}+E^H_{22}$. We therefore use the loss
\begin{equation}
    \mathcal{L}_{\mathrm{Homo}} = -c(\rho_{\theta}),
\end{equation}
 and gradients can be computed using the adjoint method.

\textbf{Masked Latent Decoding.}
To ensure texture content appears exclusively in solid regions, we use the mask to exclude background latents prior to decoding. Following denoising and decoding, the mask is interpolated to the image resolution and serves as the alpha transparency channel.

\begin{figure*}[!th]
    \centering
    \includegraphics[width=1\linewidth]{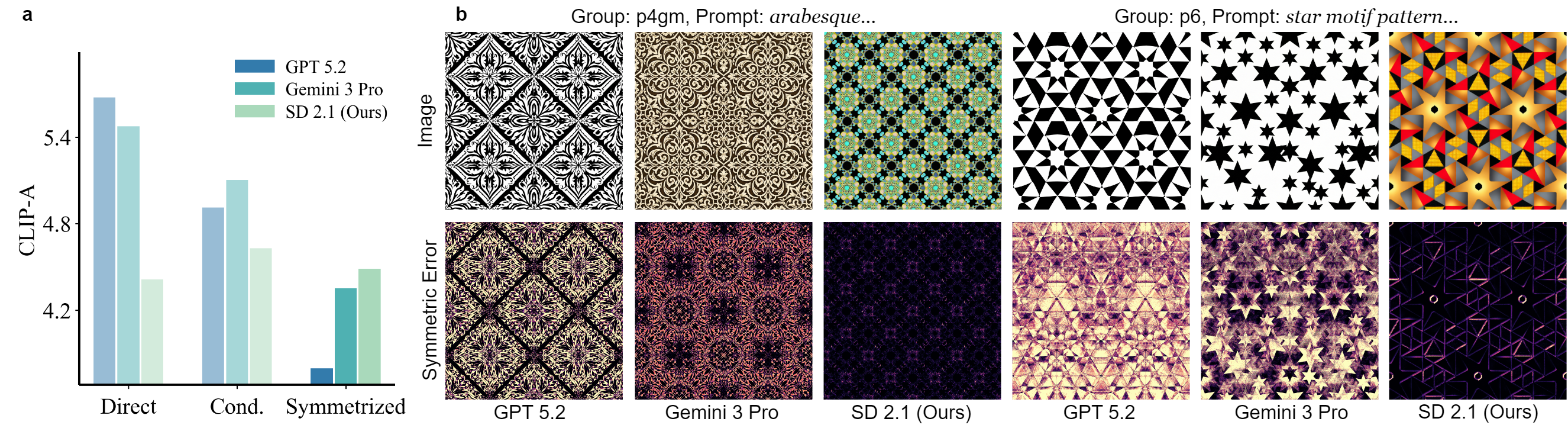}
    \caption{\textbf{Results of symmetric pattern design.}
(a) CLIP-A scores (higher is better) across 17 prompts under direct generation, conditional generation, and post-symmetrization. Our method achieves the highest score in the post-symmetrization setting.
(b) Results for groups $p4gm$ and $p6$. For each group, the top row shows conditional-generation images, while the bottom row shows the pixel-wise MSE between conditional-generation and post-symmetrized images; darker regions indicate smaller errors. This comparison highlights geometric inconsistencies introduced by strict symmetry enforcement in the baseline.}
    \label{fig:exp_pattern}
\end{figure*}

\begin{table}[t]
    \centering
    \caption{Comparison with basis projection (BP) for the $p1$ group under different retained fractions $\alpha$ of Nyquist frequencies and latent resolutions $n$ (see \cref{sec:other_symmetrization})}
    \label{tab:representation_comparison}
    \begin{tabular}{lccc}
        \toprule
        Method & $n=64$ & $n=128$ & $n=256$ \\
        \midrule
        BP ($\alpha=0.25$) & 3.81 & 3.07 & 3.46 \\
        BP ($\alpha=0.50$) & 3.98 & 3.38 & 3.61 \\
        BP ($\alpha=1.00$) & 4.05 & 3.42 & 3.48 \\
        Ours & \textbf{4.30} & \textbf{4.20} & \textbf{3.99} \\
        \bottomrule
    \end{tabular}
\end{table}

\section{Experiment}
\label{sec:experiment}
In this section, we present the experimental settings and results for three visual design tasks (pattern design, paper-cutting design,
and topology design) and one material design task (metamaterial
design). All symmetric parameterizations are implemented via symmetrization of a hash-coded bilinear interpolation scheme \citep{mullerInstantNeuralGraphics2022}. More implementation details are provided in \cref{sec:exp_detail}.

For \cref{sec:topology_results} and \cref{sec:metamaterial_results}, reported bulk-modulus values are unnormalized quantities which do not affect method comparisons within the same lattice setting.

\subsection{Results of Pattern Design}
\label{sec:pattern_results}

For the general symmetric pattern design task, we employ Stable Diffusion 2.1 (SD 2.1) \citep{rombach2022high} as our base model. For our method, we utilize the generative process described in \cref{sec:pattern_design}.

\textbf{Visualization.} 
To verify the stability and diversity of our approach, we generate
patterns for all 17 symmetry groups using one text prompt, as shown
in \cref{fig:title_fig}. Unmarked results are provided in
\cref{fig:sd_1_to_17}. The annotations demonstrate that the generated patterns remain visually consistent with the prescribed symmetry operations.

\textbf{Comparison with Text-conditioned Generation.}
To benchmark our approach against state-of-the-art generative capabilities, we use MLLMs, i.e., GPT-5.2 \citep{openai2025gpt52} and Gemini 3 Pro \citep{googledeepmind2025gemini}, as baselines. We conduct experiments across three distinct settings to assess geometric precision and aesthetic quality: 
(i) Direct Generation: Models generate images directly from the text prompts without additional constraints.
(ii) Conditional Generation: Generate images with symmetry control. For MLLMs, we provide the text prompt alongside an auxiliary visual instruction.
(iii) Post-Symmetrization: To evaluate geometric consistency, we fit images generated in the second setting with our symmetric parameterization in \cref{sec:hironaka_decomposition_r2} using MSE loss. This process enforces strict symmetry on the generated images.

To ensure a comprehensive evaluation, we constructed a test set of 17 text prompts, one corresponding to each symmetry group. The prompts were adapted from the examples in Tab. 9 of \citet{shubnikovSymmetryScienceArt1974} and were labeled and simplified using Gemini. We evaluate the aesthetic quality of generated patterns using the CLIP aesthetic (CLIP-A) metric \citep{schuhmann2022laion}.

The quantitative results in \cref{fig:exp_pattern} show that state-of-the-art MLLMs achieve strong aesthetic scores in direct generation, but their performance degrades once symmetry constraints are imposed. In particular, post-symmetrization enforces exact symmetry but substantially reduces the aesthetic quality of MLLM-generated images, indicating that these models produce visually appealing images without faithfully satisfying symmetry constraints. In contrast, our method, despite using the weaker SD 2.1 backbone, maintains higher aesthetic quality under symmetry enforcement and outperforms the baselines in the post-symmetrization setting. This trend is further supported by the qualitative MSE results in \cref{fig:mse_detail}. For full generated results of our method, see \cref{fig:sd_17_to_17}.

\textbf{Comparison with Other Symmetrization.} 
There are also several alternative approaches to symmetrization.
For example, basis projection maps an asymmetric function onto a selected invariant subspace while preserving continuity.
We compare our method with this projection-based baseline for pattern generation in \cref{tab:representation_comparison} under the same generation pipeline for the $p1$ group, our method achieves higher CLIP-A scores than basis projection at all tested resolutions. A detailed analysis is provided in \cref{sec:other_symmetrization}.

\begin{figure}[!t]
    \centering
    \includegraphics[width=1\linewidth]{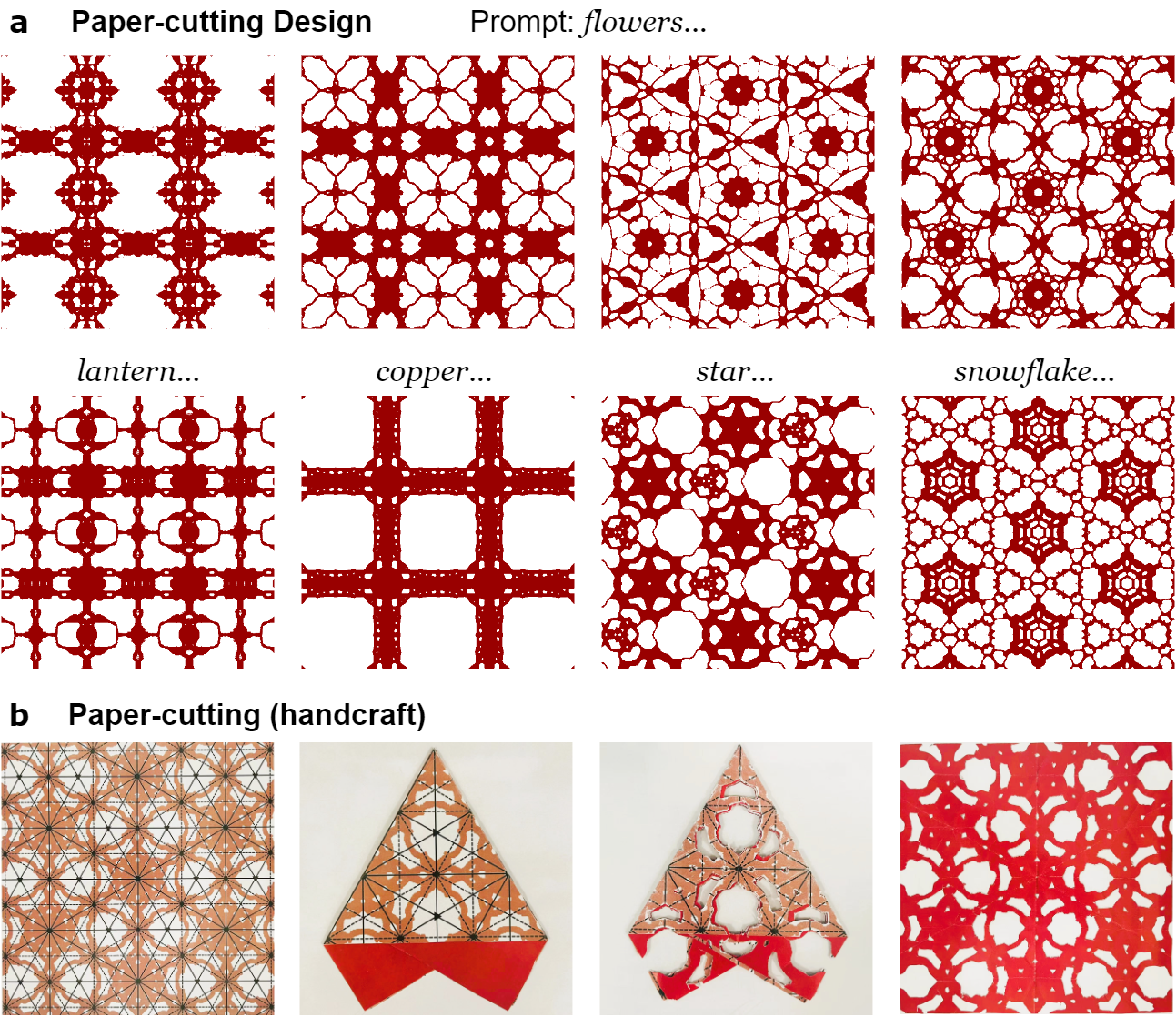}
    \caption{\textbf{Qualitative results of paper-cutting design.} (a) Diverse patterns generated by our method under varying text prompts while satisfying the prescribed symmetry and producing connected structures. (b) Physical realization demonstrates the digital pattern, folding plan, crafting process, and final paper-cutting result to verify structural integrity and manufacturability.}
    \label{fig:exp_papercut}
\end{figure}

\subsection{Results of Paper-Cutting Design}
\label{sec:paper_cutting_results}

For the paper-cutting task, we adopt Stable Diffusion XL base 1.0 (SDXL 1.0) \citep{podellsdxl} as the backbone and fine-tune a LoRA module on the dataset of \citet{wang2025harmonycut} using \texttt{diffusers} library~\cite{von-platen-etal-2022-diffusers}, following the procedure described in \cref{sec:paper_cutting_design}.

We report qualitative results in \cref{fig:exp_papercut}. Our method generates semantically aligned patterns while preserving connectivity. We further validate manufacturability by fabricating a generated mask. In practice, we dilate thin structures for cuttability, fold the sheet along the reflection axis, cut through the folded layers, and then unfold to obtain the final paper-cutting. This end-to-end fabrication verifies the practical feasibility of our designs.

\begin{figure*}[!t]
    \centering
    \includegraphics[width=1\linewidth]{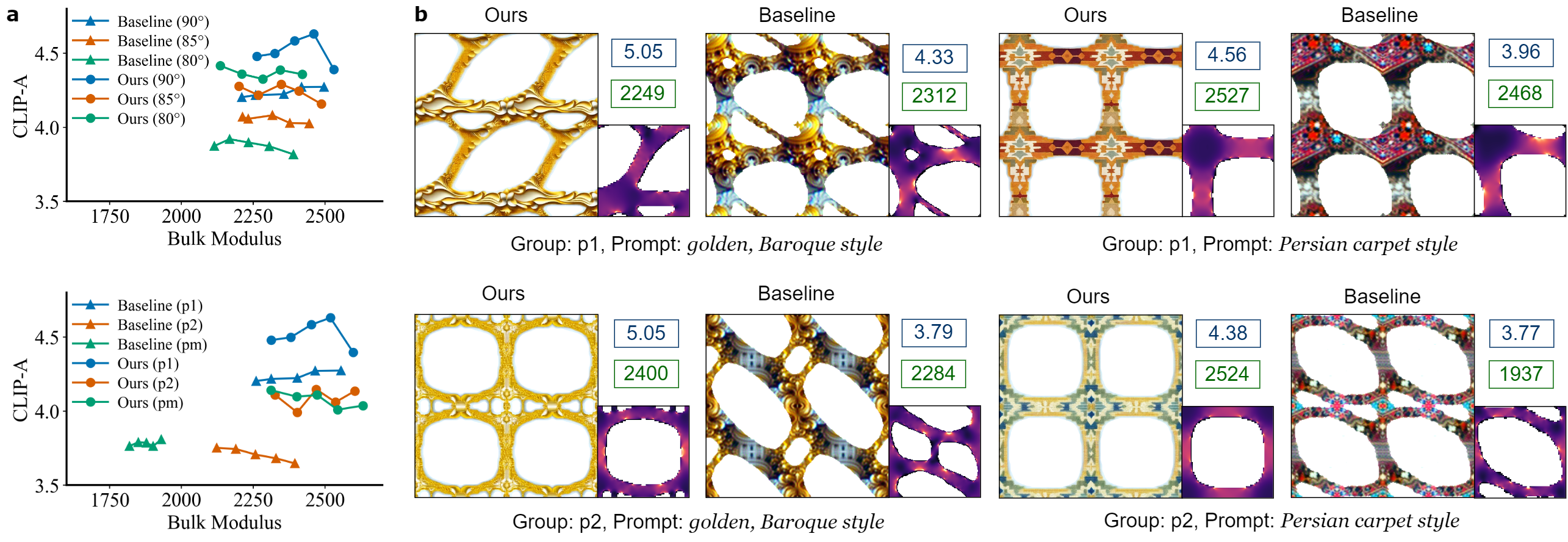}
    \caption{\textbf{Quantitative and qualitative comparison on topology design.} (a) The top row compares performance under different angles $\gamma$ for the $p1$ group, while the bottom row compares different symmetry groups. In the reported runs, our method achieves higher CLIP-A while matching or exceeding the baseline's mechanical performance. (b) The blue and green boxes indicate the value of CLIP-A and Bulk modulus, respectively. Our method generates structures with clearer semantic patterns and better mechanical performance.}
    \label{fig:exp_topo}
\end{figure*}

\begin{figure}[!t]
    \centering
    \includegraphics[width=0.96\linewidth]{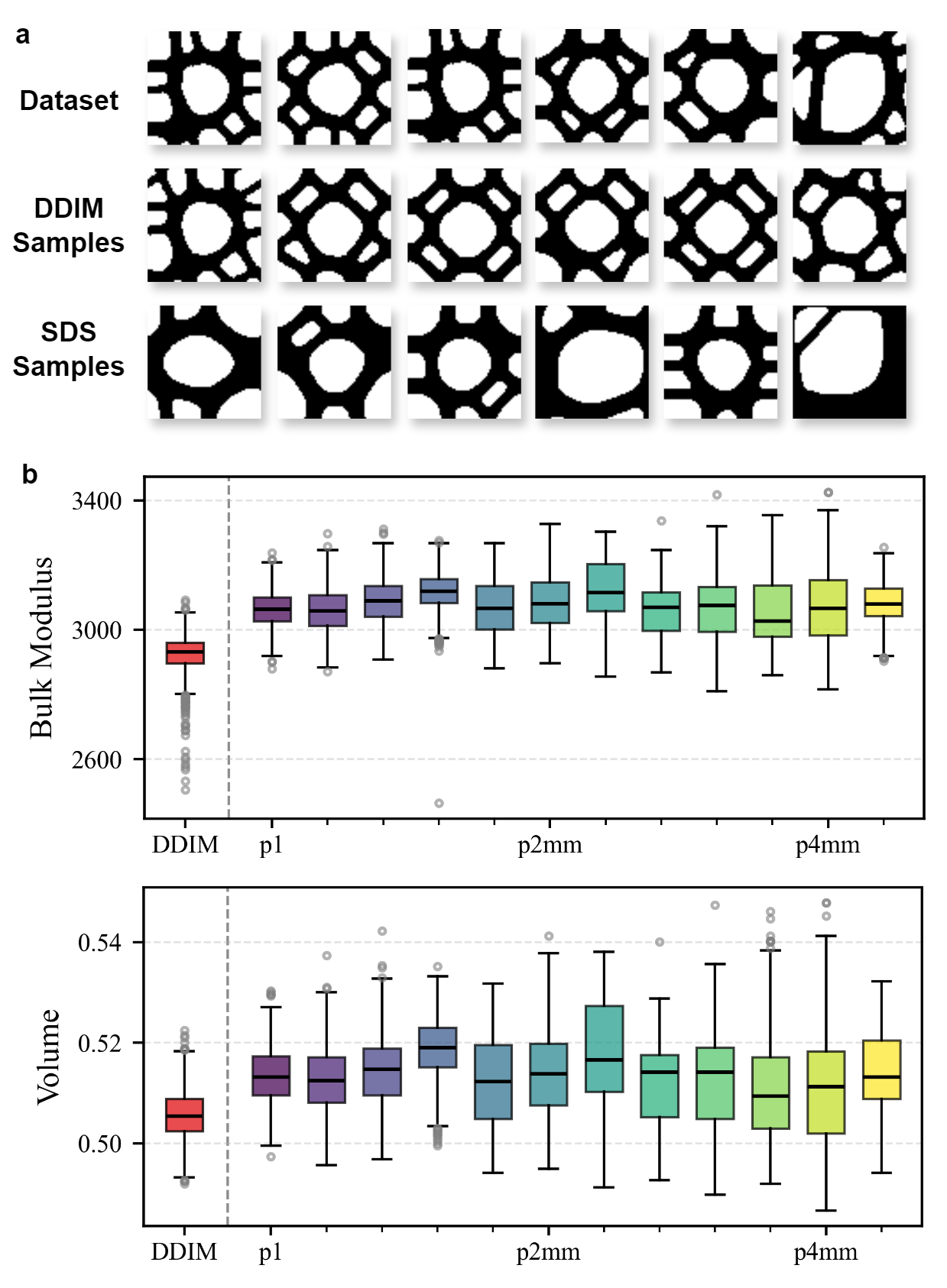}
    \caption[Generated Sample Distribution under Symmetry Constraints]{\textbf{Generated Sample Distribution under Symmetry Constraints.} (a) Samples of unit cell from the training dataset, DDIM generation, and our SDS with $p1$ constraint. (b) Bulk-modulus and volume distributions of 1000 samples from DDIM and SDS under the first 12 planar groups. Our method produces higher-performing structures while remaining close to the target volume fraction $0.5$ without symmetry-specific training data.}
    \label{fig:mat_box}
\end{figure}

\subsection{Results of Topology Design}
\label{sec:topology_results}
For the topology design task, we employ the SDXL 1.0 as our backbone. We benchmark our approach against the CLIP-based topology optimization method proposed by \citet{zhongTopologyOptimizationTextguided2023}. Following their protocol, we utilize their set of 12 test prompts of the stylization gallery. Aesthetic quality is evaluated using the CLIP-A. When testing the baseline, we retain its original mask modeling method and apply our symmetrization procedure to construct the corresponding symmetric continuous representation.

For comparison, we optimize the two representations using the baseline's CLIP loss and our SDS loss, respectively. The target volume fraction is set to 0.45. We conduct experiments across two distinct settings: (1) varying the lattice angle $\gamma \in \{90^\circ, 85^\circ, 80^\circ\}$ under the $p1$ symmetry group, and (2) varying the symmetry constraints across $p1$, $p2$, and $pm$ groups. We sample results at five volume fraction intervals ($0.43, 0.44, 0.45, 0.46, 0.47$) by adjusting the threshold applied to the mask for each setting.

The results are illustrated in \cref{fig:exp_topo}. In the variable angle setting, our method achieves mechanical properties comparable to the baseline while maintaining a substantial lead in aesthetic quality. Notably, in the variable symmetry setting, particularly for $p2$ and $pm$ groups, our method outperforms the baseline in both mechanical performance and aesthetic scores, highlighting its improved stability when optimizing under higher-symmetry constraints. The results demonstrate that our SDS-based optimization successfully synthesizes structures that are not only mechanically robust but also visually consistent with the text prompts. Additional results of our method across different symmetry groups are provided in \cref{fig:sdto_p1,fig:sdto_p2,fig:sdto_pm}.

\subsection{Extension to Metamaterial Design}
\label{sec:metamaterial_results}

We further extend our framework to mechanical metamaterial design. Structures are represented as binary unit-cell patterns, whose symmetries are closely related to their mechanical properties. We consider a zero-shot symmetry-control setting: the base diffusion model is trained only on topology-optimized unit cells with the basic $p1$ symmetry, while additional planar-group symmetries sharing the same lattice are imposed at sampling time. The goal is therefore to generate symmetric unit-cell structures guided by a mechanically informed diffusion prior, without symmetry-specific training data.

The training dataset contains 36,000 binary samples generated by homogenization method using the optimality criteria method \citep{xia2015design}. Each sample is initialized from random perturbations and optimized to maximize the bulk modulus subject to a prescribed volume constraint. The diffusion model is parameterized by a convolutional U-Net and trained for 100 epochs. 

We apply the same parametric symmetric representation as in the previous visual design tasks. The SDS loss guides the optimization in the symmetric representation space toward the high-performance structure distribution captured by the pretrained diffusion model, while the representation itself ensures that the generated binary masks satisfy the prescribed planar-group symmetry.

We evaluate the method on the first 12 planar groups with 1,000 generated samples for each setting. As shown in \cref{fig:mat_box}, the proposed method produces diverse symmetric structures with strong mechanical performance. Quantitatively, the generated samples achieve high bulk modulus with a relative volume-fraction MAE below $1.5\%$. These results demonstrate that our framework is not limited to visual pattern generation, but can also be applied to material design problems where symmetry constraints should be enforced while maintaining structural performance.

\section{Conclusion}

We introduce a general symmetrization method that transforms any continuous 2D representation into an exactly symmetric and spatially continuous one under arbitrary planar groups, avoiding the boundary discontinuities of naive asymmetric-unit extensions. By embedding planar groups into affine reflection groups, our method constructs a differentiable representation from high-symmetry coefficients and low-symmetry bases. This further yields a zero-shot controllable generation framework for symmetry-constrained design: symmetry is enforced directly in the representation space, while other task-specific objectives guide optimization. Experiments on three visual design tasks and one material design task demonstrate the broad applicability of our framework, as well as its ability to maintain symmetry control under diverse geometric and physical constraints.

\section*{Acknowledgment}

We would like to acknowledge that the work is financially supported by the National Natural Science Foundation of China (No. 62276269, 62376276), the Beijing Natural Science Foundation (No. F261002), and the Beijing Nova Program (No. 20230484278).

\section*{Author Contributions}

Ning Lin organized the project and led the theoretical development in \cref{sec:preliminaries}--\cref{sec:computation}, including the corresponding theoretical proofs. Ning Lin and Luxi Chen jointly led the framework design in \cref{sec:gen_obj_and_loss} and the experimental studies in \cref{sec:experiment}. Specifically, Ning Lin mainly contributed to the paper-cutting and topology design tasks, together with their controllable generation experiments. Luxi Chen mainly contributed to pattern design task, LoRA fine-tuning in \cref{sec:paper_cutting_results}, diffusion model training in \cref{sec:metamaterial_results}. Huaguan Chen and Jiacheng Cen contributed to parameter tuning, figure preparation, and visualization. Chongxuan Li, Wenbing Huang, and Hao Sun jointly supervised and guided the project. All authors contributed to writing and revising the manuscript.

\section*{Impact Statement}

This paper presents work whose goal is to advance the field of machine learning. There are many potential societal consequences of our work, none of which we feel must be specifically highlighted here.

\bibliographystyle{utils/icml2026}
\bibliography{Reference}

\clearpage
\appendix
\onecolumn
\tableofappendix

\section{Background}
\label{sec:background}

\subsection{Controllable Generation}
\label{sec:pipline_of_controllable_generation} 

\textbf{Zero-shot Controllable Generation.} Diffusion model \citep{ho2020denoising} and related flow matching models \citep{lipmanflow} have been widely applied across diverse domains, including image synthesis~\citep{rombach2022high, podellsdxl, lu2025dpm}, video generation~\citep{openai2024sora, guan2025taming}, 3D object generation~\citep{poole2023dreamfusion, chen2025microdreamer}, and material design~\citep{jiao2023crystal, wu2026dmflow,li2026unipath}. Building on these advances, controllable generation aims to guide these generative models toward samples that satisfy user-specified conditions or constraints, such as text prompts~\citep{chung2023diffusion, domingo2025adjoint, ye2024tfg}, spatial constraints~\citep{zhang2023adding}, or physical properties~\citep{vatani2025tripoptimizer, wu2026dao, chen2026optimization}. 

Given a fixed constraint cost model, existing controllable generation methods can be broadly categorized into training-based and training-free approaches. Training-based methods, such as \citet{domingo2025adjoint}, rely on additional constraint-related data to fine-tune the base generative model. In contrast, training-free, or zero-shot, methods, such as \citet{chung2023diffusion}, require no extra training data and instead incorporate the cost model as a plug-in guidance module during generation. 

Our work belongs to the training-free paradigm of controllable generation, focusing on the incorporation of exact symmetry constraints into diffusion-based optimization. Symmetric pattern generation is naturally a zero-shot problem, since real-world datasets rarely contain samples that strictly satisfy prescribed symmetry constraints. We decouple exact symmetry from task-specific objectives by enforcing symmetry through the representation, while imposing other requirements via loss functions. Starting from randomly initialized design parameters, we generate designs by optimizing the symmetric representation with the corresponding task-specific losses. 

The overall generation pipeline and its correspondence to the main technical components are illustrated in \cref{fig:pipeline}.

\begin{figure*}[!htbp]
    \centering
    \includegraphics[width=0.8\linewidth]{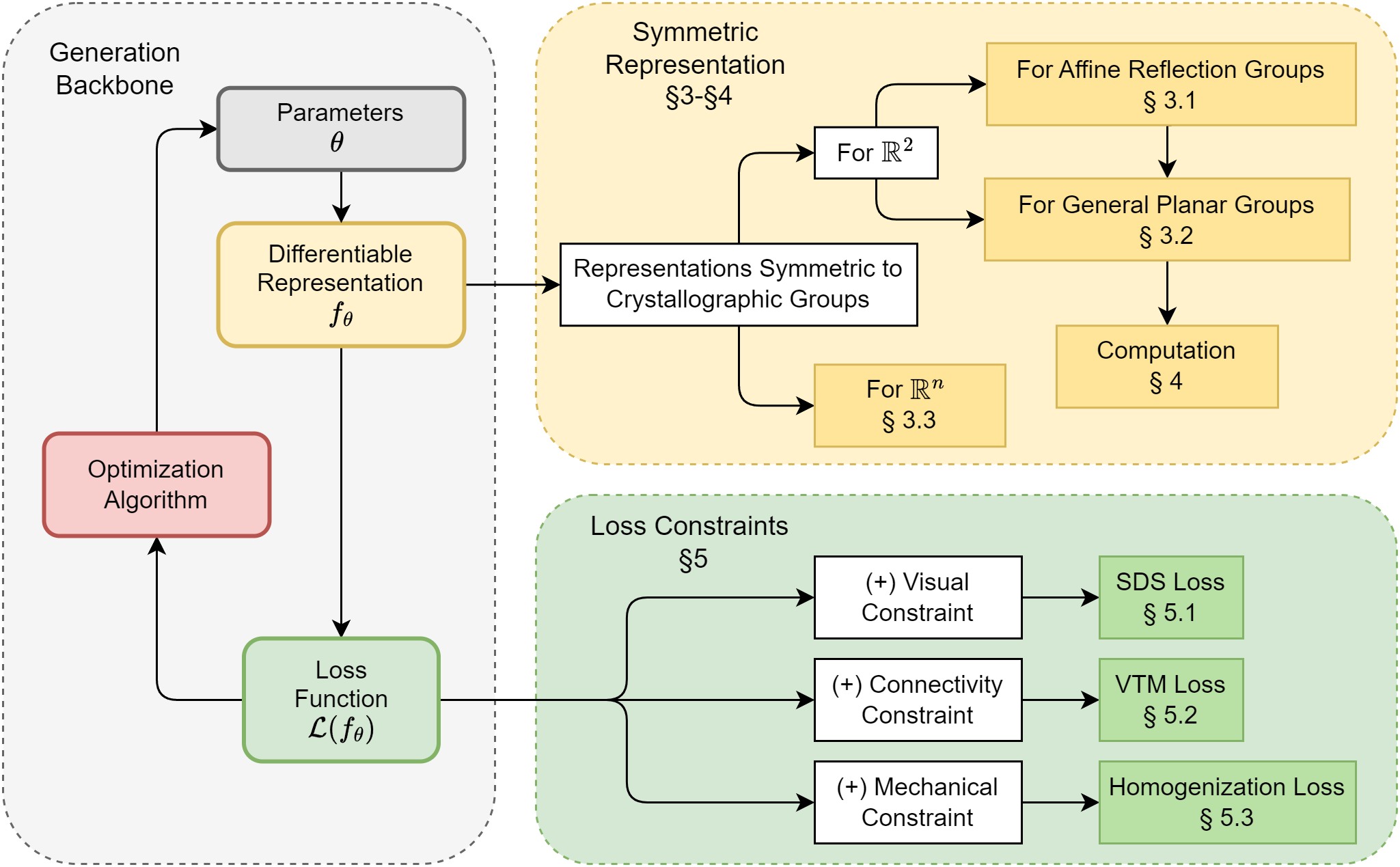}
    \caption{Overview of our training-free controllable generation pipeline.}
    \label{fig:pipeline}
\end{figure*}

\textbf{Stylized Topology Generation.}
In mechanical engineering, topology optimization improves structural performance by optimizing the spatial distribution of materials within a prescribed design domain \citep{bendsoe2013topology}, with representative algorithms including the solid isotropic material with penalization (SIMP) method \citep{andreassen2011efficient}. In periodic material design, topology optimization is often combined with homogenization methods, which evaluate the effective macroscopic properties of the optimized microstructures \citep{bendsoe2013topology}.

Several recent works have explored the trade-off between mechanical performance and visual appearance in material design \citep{martinez2015structure, zhongTopologyOptimizationTextguided2023}. The most related work is \citet{zhongTopologyOptimizationTextguided2023}, which jointly optimizes the mechanical performance of a binary material mask, represented by the alpha channel, and its semantic similarity to a text description, measured through the RGB image channel. This approach relies on pre-trained vision-language models, which provide a shared embedding space for visual and textual representations \citep{radford2021learning, chen2023vlp}. Unlike our approach in \cref{sec:topo_design}, it models the material mask as an additional channel rather than extracting it through segmentation. In addition, connectivity is encouraged by penalizing the area of small connected components.

However, such channel-based mask modeling can be vulnerable to the generative prior. Since the diffusion prior favors visually realistic RGB content, the optimized image may fill the masked material region with background-like textures or semantic content, weakening the correspondence between the visual appearance and the actual material layout. Moreover, the connectivity penalty only suppresses disconnected components rather than actively reconnecting them to the main structure, making it difficult to repair or refine the optimized topology.

Our method instead focuses on periodic structure design and derives the material layout from the generated appearance through segmentation, which better aligns the visual content with the optimized mask structure. Motivated by diffusion-prior-based optimization methods such as DreamFusion \citep{poole2023dreamfusion}, we introduce a diffusion prior to improve visual quality while optimizing the periodic material representation. Furthermore, rather than merely penalizing small isolated components, we incorporate a differentiable VTM loss into the objective, which provides a more direct mechanism for controlling and improving structural connectivity.

\textbf{Material Design and Symmetry Constraints.}
Planar- and space-group symmetries provide a natural language for describing periodic structures. In materials science, crystals are a representative class of discrete periodic structures: their unit cells are microscopic atomic arrangements, often modeled as graphs whose nodes correspond to atoms. Metamaterials, in contrast, provide a continuous counterpart: their unit cells are mesoscopic structures composed of continuum materials and are represented by binary masks, where the values indicate the absence or presence of material. For both crystals and metamaterials, symmetry plays a crucial role in determining material properties \citep{levy2025symmcd, mao2020designing}.

Recent works in crystal and metamaterial generation have therefore studied how to ensure that generated structures exactly satisfy prescribed symmetries. In crystal generation, \citet{jiaospace} considers generation under a given space-group type together with the site symmetry of atoms, which can be regarded as a symmetry template. From the perspective of controllable generation, \citet{jiaospace} puts symmetry constraints into the generation process through this prescribed template. \citet{levy2025symmcd} follows a conditional-generation formulation and further relaxes the dependence on predefined templates by treating the atomic site symmetries as generative variables. Given a space-group type and the number of orbit representatives, it learns the site-symmetry assignments from data and reconstructs the full crystal by replicating the generated asymmetric unit.

In metamaterial design, \citet{mao2020designing} study the generation of two-dimensional metamaterials under prescribed planar-group symmetries. Their method requires training a symmetry-specific generative model for each plane group using symmetric unit-cell masks. In contrast, when applied to symmetric metamaterial generation, our method does not require symmetry training data or a separately trained model for each symmetry group. Instead, our method only needs a generator trained on $p1$ periodic patterns, i.e., patterns without additional symmetry. To generate a pattern with a desired symmetry, we optimize our symmetric representation using the SDS loss. Therefore, changing the target symmetry does not require collecting symmetry-specific data or retraining the generator.

\subsection{Other Symmetrization.}
\label{sec:other_symmetrization}

We compare our method with several natural alternatives for enforcing planar-group symmetry. These alternatives can be divided into two categories: enforcing symmetry directly in the parameter space of a representation, and transforming the output function after evaluation.

\textbf{Constraining the Parameter Space.}
A seemingly direct approach is to impose symmetry constraints on the parameters of a continuous representation. However, this strategy is generally impractical for two reasons. First, plane groups are infinite because they contain translations, which would in principle induce infinitely many equality constraints. Solving such constraints numerically is therefore intractable without additional truncation or approximation. Second, for modern continuous representations such as NeRF~\citep{mildenhall2021nerf} and InstantNGP~\citep{mullerInstantNeuralGraphics2022}, the mapping from parameters to function values is highly nonlinear. As a result, even if the symmetry constraints are linear, there may not be an explicit or numerical tractable solution of parametrization of the corresponding symmetric parameter subspace.

\textbf{Transforming the Output Representation.}
Another class of methods enforces symmetry at the value of function. Besides our method, we consider three alternative approachs. The first is group averaging, which averages the representation output over all symmetry-transformed inputs. Although this produces an invariant function in principle, it is infeasible for infinite plane groups, which only yields approximate symmetry. The second is asymmetric-unit extension, which parametrizes the function only on an asymmetric unit and extends it to the whole domain via group actions. This approach is efficient and direct, but it can introduce boundary discontinuities when the asymmetric unit is not a reflection chamber or when adjacent regions are related by non-reflective transformations. The third is basis projection, which projects a general representation onto an invariant function space spanned by symmetric basis functions. Below we analyze the third method projected to invariant space spanned by Fourier basis functions.

\textbf{Complexity of Basis Projection.}
Let the target image resolution be $n \times n$. A natural choice of basis for plane-group-symmetric functions is the trigonometric basis. By the Nyquist sampling theorem, the number of recoverable frequency components is on the same order as the number of pixels in the asymmetric unit. Let $\alpha$ denote the fraction of frequency components retained, and let $\beta$ denote the number of asymmetric-unit copies in the full unit cell. Then the number of basis functions is approximately $m \approx \alpha n^2 / \beta$. 

A direct implementation requires an
$\mathcal{O}(Nm^2)=\mathcal{O}(n^6)$ one-time QR factorization and
$\mathcal{O}(Nm)=\mathcal{O}(n^4)$ time per projection.
Using the Fourier structure, the same projection can instead be
implemented by FFT-based coefficient truncation and symmetry-orbit
averaging in $\mathcal{O}(n^2\log n)$ time.
For each fixed plane group, our method evaluates a constant number of
coefficient fields and basis functions per pixel, requiring
$\mathcal{O}(n^2)$ time.

In contrast, our method introduces only a constant per-pixel time cost from reflection operations and the evaluation of the fixed basis functions $\eta_i$. The overall evaluation cost remains linear in the number of pixels, namely $\mathcal{O}(n^2)$, making the method substantially more efficient in both time and memory.

\textbf{Effect on Generation Quality.}
Basis projection on invariant space may have an unfavorable effect on generation. For locally supported representations, pixel values in different unit cells are initialized independently. After projection, values related by symmetry are aggregated, which reduces variance and produces an overly smooth initialization. 

This is problematic for SDS-based optimization, whose performance is sensitive to initialization. In particular, approximately Gaussian-like random initializations tend to provide richer high-frequency content and more diverse optimization trajectories, whereas projected initializations are biased toward low-frequency smooth patterns. As a result, projection-based symmetrization often degrades visual quality. 

In our experiments in \cref{sec:pattern_results}, our method consistently achieves better generation quality than the projection-based baseline, as reflected by higher CLIP-A scores across different resolutions and different choice of Fourier basis functions.

\clearpage
\section{Higher Reflectional Symmetry}
\label{sec:higher_sym_rn}

\subsection{Crystallographic Groups and the Integral General Linear Group}
\label{sec:GLnZ}

We begin with some necessary preliminaries by defining crystallographic groups in $\mathbb{R}^n$.

\begin{definition}
\label{def:crystallographic_group}
A \textbf{crystallographic group} $G$ in $\mathbb{R}^n$ is a discrete subgroup of the Euclidean group $E(n)$ that contains $n$ linearly independent translations.
\end{definition}

We also consider an alternative definition attributed to Bieberbach: a crystallographic group $G$ is a discrete subgroup of $E(n)$ such that the quotient space $\mathbb{R}^n / G$ is compact. Bieberbach's First Theorem establishes the equivalence between these two definitions; see (3.2) in \citet{hillerCrystallographyCohomologyGroups1986}. This definition leads to an important property regarding the finiteness of the point group.

\begin{proposition}[\citet{hillerCrystallographyCohomologyGroups1986}, Thm. 3.2.1]
Let $G$ be a crystallographic group. The translation subgroup $T(G) = G \cap \mathbb{R}^n$ is isomorphic to $\mathbb{Z}^n$ and is a normal subgroup of $G$. Furthermore, the point group $P(G) = G / T(G)$ is finite.
\end{proposition}

For a crystallographic group, the conjugation action of any group element on a translation element always yields a translation element:
\begin{equation}
    (A, \mathbf{t}_2)(O, \mathbf{t}_{1})(A, \mathbf{t}_{2})^{-1}(\mathbf{x}) = \mathbf{x} + A(\mathbf{t}_{1}) .
\end{equation}
That is, $(A, \mathbf{t}_2)(O, \mathbf{t}_{1})(A, \mathbf{t}_{2})^{-1} = (O, A(\mathbf{t}_{1}))$. Thus, the translation subgroup is a normal subgroup, and the action of the point group $P(G)$ on a translation vector results in another translation vector.

Noting that all translation vectors are generated by a set of fundamental translations, the point group $P(G)$ can be represented by integer matrices under the basis determined by these fundamental translations. Specifically, the point group can be represented by elements of the general linear group over integers, $\GL(n, \mathbb{Z})$. Under this construction, the crystallographic group is represented as a finite subgroup of $\GL(n, \mathbb{Z})$, to be specific, as a group of automorphisms of the lattice defined by the translation subgroup $T(G) \cong \mathbb{Z}^n$.

The choice of different fundamental translations results in matrix representations of the subgroup that differ by conjugation via elements of $\GL(n, \mathbb{Z})$. This allows us to determine equivalence classes of crystallographic groups, known as $\mathbb{Z}$-classes or arithmetic classes. Consequently, the discussion of $\mathbb{Z}$-classes is equivalent to the discussion of conjugacy classes of finite subgroups of $\GL(n, \mathbb{Z})$. 

If the representations differ by conjugation via elements of $\GL(n, \mathbb{Q})$, we refer to such equivalence classes as $\mathbb{Q}$-classes. It can be shown that any conjugacy class of finite subgroups of $\GL(n, \mathbb{Q})$ always contains a finite subgroup of $\GL(n, \mathbb{Z})$. Therefore, the discussion of $\mathbb{Q}$-classes is equivalent to the discussion of conjugacy classes of finite subgroups of $\GL(n, \mathbb{Q})$ (see the discussion in Sec. 1.10.1 of \citet{lorenzMultiplicativeInvariantTheory2005}).

There exists a crystallographic group that is isomorphic to the semidirect product $P \ltimes T$ corresponding to each arithmetic class. Such crystallographic groups are called symmorphic crystallographic groups. Otherwise, $G$ is called a non-symmorphic crystallographic group. Euclidean transformations associated with non-symmorphic crystallographic groups involve non-integer translations. That is, the translational component cannot be expressed as an integer linear combination of the fundamental translations. In the case of $n = 2$, these transformations are exclusively glide reflections, formed by the composition of a reflection and a non-integer translation. For $n = 3$, in addition to glide reflections, such transformations may also be screw rotations, which are the composition of a rotation and a non-integer translation.

We now introduce a very important class of symmorphic crystallographic groups: \textbf{affine reflection groups}. An affine reflection group is a discrete subgroup of $E(n)$ generated by reflections. These groups possess a classification theorem determined by root systems, thus transforming the classification problem of reflection groups into that of root systems. In subsequent discussions, references to root systems imply reduced or crystallographic root systems (see Sec. 2.8 in \citet{humphreysReflectionGroupsCoxeter1992}).

\begin{proposition}[\citet{humphreysReflectionGroupsCoxeter1992}, Sec. 4.10]
For an affine reflection group $G$ in $\mathbb{R}^n$, there exists a root system $\Phi$ in $\mathbb{R}^n$ such that $G = W \ltimes L(\Phi^{\vee})$, where $W$ is the Weyl group of $\Phi$, and $L(\Phi^{\vee})$ is the translation subgroup corresponding to the coroot lattice of $\Phi$. Conversely, if for a crystallographic group $G$ there exists a root system $\Phi$ in $\mathbb{R}^n$ such that $G = W(\Phi) \ltimes L(\Phi^{\vee})$, then $G$ is an affine reflection group.
\end{proposition}

For an affine reflection group, its action on the translation subgroup $T(G) = L(\Phi^{\vee})$ is equivalent to the action of the Weyl group on the coroot lattice. This property facilitates the identification of affine reflection groups within the conjugacy classes of finite subgroups of $\GL(n, \mathbb{Z})$ and $\GL(n, \mathbb{Q})$.

\subsection{Subgroups of Crystallographic Group}
\label{sec:subgroups_of_crystallographic_group}

In \cref{thm:existence_of_higher_Wa_n=2}, we need to investigate the conditions under which a crystallographic group is conjugate to a subgroup of an affine reflection group. We first address the problem of when a crystallographic group is a subgroup of an affine reflection group. The following lemma suggests that when considering inclusion problems, it suffices to consider symmorphic crystallographic groups.

\begin{lemma}
A non-symmorphic crystallographic group is a subgroup of a symmorphic crystallographic group.
\end{lemma}

\begin{proof}

Let $G\leq E(n)$ be a crystallographic group, with translational
subgroup $T(G)$ and point group $P(G)$. Set $m=\lvert P(G)\rvert$ and define the enlarged translational subgroup
\begin{equation}
    \widetilde{T}
    :=
    \frac{1}{m}T(G)
    =
    \left\{
        \left(e,\frac{t}{m}\right)
        \;\middle|\;
        (e,t)\in T(G)
    \right\}.
\end{equation}
Since $T(G)$ is invariant under the action of $P(G)$, $\widetilde{T}$ is also invariant under the action of $P(G)$. 

Next, let
\begin{equation}
    \widehat{P(G)}
    :=
    \left\{
        (g,0)
        \;\middle|\;
        g\in P(G)
    \right\},
\end{equation} 
we prove that $G \subset \tilde{G}$, where
\begin{equation}
    \widetilde{G}
    :=
    \tau_c
    \left(
        \widehat{P(G)}\ltimes\widetilde{T}
    \right)
    \tau_{-c}
\end{equation}

and $\tau_c=(e,c)$ denote translation by some vector $c$. For $\widetilde{G}$ is symmorphic, $\tilde{G}$ is the required symmorphic crystallographic group.

For each $g\in P(G)$, choose an element
\begin{equation}
    \gamma_g=(g,a_g)\in G,
\end{equation}
with $\gamma_e=(e,0)$. For any $g,h\in P(G)$, the elements
$\gamma_h\gamma_g$ and $\gamma_{hg}$ have the same orthogonal part.
Therefore, there exists a translation vector $\ell_{h,g}$ such that
\begin{equation}
    \gamma_h\gamma_g
    =
    (e,\ell_{h,g})\gamma_{hg}.
\end{equation}
Using the multiplication rule for affine transformations, we obtain
\begin{equation}
    a_h+ha_g-a_{hg}
    =
    \ell_{h,g}.
\end{equation}

Define
\begin{equation}
    c
    =
    \frac{1}{m}\sum_{g\in P(G)}a_g.
\end{equation}
For every $h\in P(G)$, the map $g\mapsto hg$ permutes the elements of
$P(G)$. It follows that
\begin{equation}
\begin{aligned}
    hc+a_h-c
    &=
    \frac{1}{m}
    \sum_{g\in P(G)}
    \left(ha_g+a_h-a_{hg}\right) \\
    &=
    \frac{1}{m}
    \sum_{g\in P(G)}
    \ell_{h,g}.
\end{aligned}
\end{equation}
Since each $(e,\ell_{h,g})$ belongs to $T(G)$, we have
\begin{equation}
    \left(e,hc+a_h-c\right)\in\widetilde{T}.
\end{equation}
Conjugating $\gamma_h$ by $\tau_c$ gives
\begin{equation}
    \tau_{-c}\gamma_h\tau_c
    =
    \left(h,hc+a_h-c\right)
    \in
    \widehat{P(G)}\ltimes\widetilde{T}.
\end{equation}
Every element of $G$ can be written as $(e,t)\gamma_h$ for some $(e,t)\in T(G)$ and $h\in P(G)$. Since $T(G) \subset \widetilde{T}$, we conclude that
\begin{equation}
    \tau_{-c}G\tau_c
    \leq
    \widehat{P(G)}\ltimes\widetilde{T}.
\end{equation}

Then $G \subset \tilde{G}$ holds.

\end{proof}

\begin{lemma}
\label{lem:Qclass_embedding_symmorphic}
(Embedding of Symmorphic Group)
A symmorphic crystallographic group $G$ is conjugate, via an invertible linear transformation, to a subgroup of another symmorphic crystallographic group $K$
if and only if the $\mathbb{Q}$-classes $(G)$ and $(K)$ associated with the actions on the translation subgroups satisfy $(G)\le (K)$, i.e., there exists an element of $\GL(n,\mathbb{Q})$ such that the finite subgroup of $\GL(n,\mathbb{Q})$ induced by $G$ is conjugate to a finite subgroup induced by $K$.
\end{lemma}

\begin{proof}
Assume throughout that $G$ and $K$ are written in standard
symmorphic form, namely,
\begin{equation}
    H=P(H)\ltimes T(H)
    =
    \left\{
        (A,\mathbf{v})
        :
        A\in P(H),\ \mathbf{v}\in T(H)
    \right\},
    \quad H\in\{G,K\}.
\end{equation}
In particular, $(A,\mathbf{0})\in H$ for every $A\in P(H)$.
After choosing $\mathbb{Z}$-bases of the translation lattices, let
$\rho_H(P(H))\le\GL(n,\mathbb{Z})$ denote the corresponding point-group
representation.

We first fix notation and the meaning of $(\cdot)$ and $\le$.
For a symmorphic crystallographic group $H\le E(n)$, symmorphicity means that $H$ splits as a semidirect product
\begin{equation}
    H \cong P(H)\ltimes T(H), 
\end{equation}
and the conjugation action of $H$ on $T(H)$ factors through $P(H)$, giving a faithful action of $P(H)$ on the lattice $T(H)\cong \mathbb{Z}^n$.
After choosing a $\mathbb{Z}$-basis of $T(H)$, this action is represented by a finite subgroup
\begin{equation}
    \rho_H\bigl(P(H)\bigr)\ \le\ \GL(n,\mathbb{Z})\ \subset\ \GL(n,\mathbb{Q}). 
\end{equation}
We define $(H)$ to be the $\GL(n,\mathbb{Q})$-conjugacy class of $\rho_H(P(H))$.
Moreover, we write $(G)\le (K)$ if there exists $M\in \GL(n,\mathbb{Q})$ such that
\begin{equation}
    M\,\rho_G\bigl(P(G)\bigr)\,M^{-1}\ \le\ \rho_K\bigl(P(K)\bigr)
\end{equation}
as subgroups of $\GL(n,\mathbb{Q})$.

\smallskip
\noindent
\textbf{($\Rightarrow$)}
Assume there exists an invertible linear map $S(\mathbf{x})=B(\mathbf{x})$ with $B\in \GL(n,\mathbb{R})$ such that
\begin{equation}
    S G S^{-1}\ \le\ K. 
\end{equation}
For any translation $t_{\mathbf{v}}(\mathbf{x})=\mathbf{x}+\mathbf{v}$ in $T(G)$, $S\,t_{\mathbf{v}}\,S^{-1} = t_{B(\mathbf{v})}$, hence $B\bigl(T(G)\bigr)\subseteq T(K)$.

Choosing $\mathbb{Z}$-bases of $T(G)$ and $T(K)$, let $M$ be the
matrix of $B$ in these lattice coordinates. Since
$B(T(G))\subseteq T(K)$, we have $M\in M_n(\mathbb{Z})$. As $B$ is
invertible, $\det M\neq 0$, and hence $M\in\GL(n,\mathbb{Q})$. For any $A\in P(G)$, the linear part of $S A S^{-1}$ equals $BAB^{-1}$, and since $S G S^{-1}\le K$ we have $BAB^{-1}\in P(K)$. In lattice coordinates this implies
\begin{equation}
    M\,\rho_G\bigl(P(G)\bigr)\,M^{-1}\ \le\ \rho_K\bigl(P(K)\bigr)
\quad\text{with }M\in \GL(n,\mathbb{Q}), 
\end{equation}
hence $(G)\le (K)$.

\textbf{($\Leftarrow$)}
Assume $(G)\le (K)$. Choose linear isomorphisms $A_G,A_K\in \GL(n,\mathbb{R})$ such that
\begin{equation}
    A_G\bigl(T(G)\bigr)=\mathbb{Z}^n,\quad A_K\bigl(T(K)\bigr)=\mathbb{Z}^n. 
\end{equation}
Conjugating $G$ and $K$ by $A_G$ and $A_K$, we may assume $T(G)=T(K)=\mathbb{Z}^n$ and still have
\begin{equation}
    M\,\rho_G\bigl(P(G)\bigr)\,M^{-1}\ \le\ \rho_K\bigl(P(K)\bigr) 
\end{equation}
for some $M\in \GL(n,\mathbb{Q})$.
Pick $m\in\mathbb{N}^*$ such that $N:=mM\in M_n(\mathbb{Z})$.
Since $mI$ commutes with every matrix,
\begin{equation}
    N\,\rho_G\bigl(P(G)\bigr)\,N^{-1}=M\,\rho_G\bigl(P(G)\bigr)\,M^{-1}\ \le\ \rho_K\bigl(P(K)\bigr). 
\end{equation}

By the standard symmorphic form assumed above, every element of $G$
can be written as $(A,\mathbf{v})$ with $A\in P(G)$ and $\mathbf{v}\in\mathbb{Z}^n$, and one checks
\begin{equation}
    N\,(A,\mathbf{v})\,N^{-1} = (NAN^{-1},\,N(\mathbf{v})). 
\end{equation}
Because $N$ has integer entries, $N(\mathbf{v})\in\mathbb{Z}^n$, and by the previous inclusion $NAN^{-1}\in P(K)$.
Therefore $NGN^{-1}\subset K$ in this normalized setting.
Undoing the conjugations yields an invertible linear map
\begin{equation}
    S:=A_K^{-1}NA_G\in \GL(n,\mathbb{R}) 
\end{equation}
such that $SGS^{-1}\le K$.
\end{proof}

In the previous section, we established that for an affine reflection group, the corresponding subgroup of $\GL(n, \mathbb{Z})$ is the representation of the Weyl group acting on the coroot lattice. Consequently, we provide the condition under which \cref{thm:existence_of_higher_Wa_n=2} holds in $\mathbb{R}^n$. The Weyl-group actions on the root and coroot lattices are conjugate over $\mathbb{Q}$, since the Cartan matrix is symmetrizable over
$\mathbb{Q}$. Hence they determine the same $\mathbb{Q}$-class, and we obtain the following result:

\begin{theorem}
\label{thm:dimension_judge}
Every crystallographic group in $\mathbb{R}^n$ is conjugate to a subgroup of some affine reflection group via an invertible linear transformation if and only if every maximal finite subgroup conjugacy class of $\GL(n, \mathbb{Q})$ contains the action of a Weyl group on the root lattice.
\end{theorem}

\subsection{\texorpdfstring{Existence of Higher Reflectional Symmetry for $n = 2,3$}{Existence of Higher Affine Reflectional Symmetry for n=2,3}}
\label{sec:maximal_finite_subgroups_n=2_3}

\begin{theorem}
\label{thm:affine_reflection_supergroup_n_eq_2_3}
For $n=2,3$, \cref{thm:dimension_judge} holds. Equivalently, every planar group ($n=2$) and every space group ($n=3$) is conjugate to a subgroup of some affine reflection group via an invertible linear transformation.
\end{theorem}

\begin{proof}
We recall the consistency between subgroup conjugacy in $\GL(n,\mathbb Q)$ and $\mathbb Q$-classes arising from $\GL(n,\mathbb Z)$. Let $H\le \GL(n,\mathbb Q)$ be finite. By Sec.~1.10.1 of \citet{lorenzMultiplicativeInvariantTheory2005}, $H$ is $\mathbb Q$-conjugate to a finite subgroup of $\GL(n,\mathbb Z)$. Moreover, $\mathbb Z$-conjugacy implies $\mathbb Q$-conjugacy, so $\mathbb Q$-classes in $\GL(n,\mathbb Z)$ are coarser than $\mathbb Z$-classes and can be obtained by merging $\mathbb Z$-classes, and every maximal finite subgroup of $\GL(n,\mathbb Q)$ is $\mathbb Q$-conjugate to a maximal finite subgroup of $\GL(n,\mathbb Z)$.

For $n=2$, we consider maximal finite subgroups of $\GL(2,\mathbb Z)$. By Table 1.1 in \citet{lorenzMultiplicativeInvariantTheory2005}, there are exactly two maximal finite subgroups up to $\mathbb Z$-conjugacy. For $\mathrm{Aut}(A_2)$, $\mathrm{Aut}(B_2)$, and $\mathrm{Aut}(G_2)$ occur as maximal candidates, and by the identifications
\begin{equation}
    \mathrm{Aut}(B_2)= W(B_2),\quad \mathrm{Aut}(G_2)= W(G_2), 
\end{equation}
it follows that every maximal finite subgroup of $\GL(2,\mathbb{Q})$ is $\GL(2,\mathbb{Q})$-conjugate to either the root-lattice action of $W(B_2)$ or that of $W(G_2)$.

We now consider $n=3$. Let
\begin{equation}
    H\leq\GL(3,\mathbb{Q})
\end{equation}
be maximal finite. Suppose first that the natural rational representation of $H$ is irreducible. By Thm.~II.14 of \citet{pleskenApplicationsRepresentationTheory1991}, the Weyl group $W(B_3)$ is the only irreducible maximal finite subgroup of $\GL(3,\mathbb{Z})$. Therefore every irreducible maximal finite subgroup of $\GL(3,\mathbb{Q})$ is $\mathbb{Q}$-conjugate to the root-lattice action of $W(B_3)$.

Suppose next that the natural rational representation of $H$ is
reducible. Since $H$ is finite and the ground field has
characteristic zero, the representation is completely reducible.
After a rational change of basis, we may therefore write
\begin{equation}
    \mathbb{Q}^{3}=U\oplus V,
    \quad
    \dim U=2,
    \quad
    \dim V=1,
\end{equation}
where both $U$ and $V$ are $H$-invariant. Let $H_U$ and $H_V$
denote the images of $H$ in $\GL(U)$ and $\GL(V)$, respectively.
In a basis adapted to this decomposition,
\begin{equation}
    H\leq H_U\times H_V
    \leq \GL(U)\times\GL(V)
    \leq\GL(3,\mathbb{Q}).
\end{equation}
The group $H_U\times H_V$ is finite. Hence the maximality of $H$
implies
\begin{equation}
    H=H_U\times H_V.
\end{equation}
The same maximality argument shows that $H_U$ and $H_V$ are
maximal finite subgroups of $\GL(U)$ and $\GL(V)$, respectively.
Consequently,
\begin{equation}
    H_V=\{\pm 1\}=W(A_1),
\end{equation}
and, by the two-dimensional case, $H_U$ is
$\GL(2,\mathbb{Q})$-conjugate to either $W(B_2)$ or $W(G_2)$ in
its root-lattice representation.

The first possibility cannot occur. Indeed, under the standard
root-lattice realizations, there is a proper inclusion
\begin{equation}
    W(B_2)\times W(A_1)
    <
    W(B_3),
\end{equation}
where the left-hand side consists of the signed permutation
matrices preserving the decomposition
$\mathbb{Q}^{3}=\mathbb{Q}^{2}\oplus\mathbb{Q}$. Thus,
$W(B_2)\times W(A_1)$ is contained in a strictly larger finite
subgroup of $\GL(3,\mathbb{Q})$, contradicting the maximality of
$H$.

We must therefore have $H_U$ is $\GL(2,\mathbb{Q})$-conjugate to $W(G_2)$. It follows that $H$ is $\GL(3,\mathbb{Q})$-conjugate to
\begin{equation}
    W(G_2)\times W(A_1)
    =
    W(G_2 \times A_1),
\end{equation}
acting on the root lattice
\begin{equation}
    L(G_2 \times A_1)
    =   
    L(G_2)\times L(A_1).
\end{equation}

Thus, every maximal finite subgroup of $\GL(3,\mathbb{Q})$ is
$\mathbb{Q}$-conjugate either to $W(B_3)$ acting on
$Q(B_3)$ or to $W(G_2\times A_1)$ acting on
$Q(G_2\times A_1)$. Together with the $n=2$ case, this
verifies the condition in \cref{thm:dimension_judge} for
$n=2,3$, and the claimed conclusion for planar and space groups
follows.

\end{proof}

\subsection{\texorpdfstring{Non-existence of Higher Reflectional Symmetry for $n > 3, n \neq 7,8$}{Non-existence of Higher Affine Reflectional Symmetry for n>3, n != 7, 8}}
\label{sec:maximal_finite_subgroups_n>3,n!=7,8}

In the following proofs, we use $\mathscr S_k$ to denote the symmetric group on $k$ letters, i.e. the group of all permutations of $\{1,\dots,k\}$.

\begin{theorem}
\label{thm:no_affine_reflection_supergroup_n_neq_7_8}
For $n>3$ and $n\neq 7,8$, there exists a crystallographic group which is not
linearly conjugate to any subgroup of an affine reflection group.
\end{theorem}

\begin{proof}

By Proposition~II.8 of
\citet{pleskenApplicationsRepresentationTheory1991},
$\operatorname{Aut}(A_n)$ is a maximal finite subgroup of
$\GL(n,\mathbb Q)$ for $n\neq7,8$. By \cref{thm:dimension_judge}, we need only verify that its $\mathbb Q$-class is not represented
by a Weyl group. Since
$\operatorname{Aut}(A_n)$ contains $W(A_n)$, its natural
$n$-dimensional representation is irreducible. Therefore, if
it were $\mathbb Q$-conjugate to $W(\Psi)$ for some rank-$n$
root system $\Psi$, then $\Psi$ would have to be irreducible.

Moreover, according to Tab.~1 of Sec.12 in \citet{humphreysintroduction}, we have
\begin{equation}
    \left|\operatorname{Aut}(A_n)\right|
    =
    2(n+1)!,
\end{equation}
and for the classical irreducible root systems,
\begin{equation}
|W(A_n)| = (n+1)!,\quad
    |W(B_n)|=|W(C_n)| = 2^n n!,\quad
    |W(D_n)| = 2^{n-1}n!.
\end{equation}
None of these orders equals $2(n+1)!$ for $n>3$: equality
for $B_n$ or $C_n$ would require
$2^{n-1}=n+1$, while equality for $D_n$ would require
$2^{n-2}=n+1$, and neither equation has a solution for
$n>3$. The remaining exceptional possibilities in the
dimensions under consideration are
\begin{equation}
    |W(F_4)|=1152\neq240
    =
    |\operatorname{Aut}(A_4)|, \quad |W(E_6)|=51840\neq10080
    =
    |\operatorname{Aut}(A_6)|
\end{equation}
Hence the $\mathbb Q$-class of $\operatorname{Aut}(A_n)$
is not represented by any Weyl group.
\end{proof}

\subsection{\texorpdfstring{Non-existence of Higher Reflectional Symmetry for $n = 7,8$}{Non-existence of Higher Affine Reflectional Symmetry for n = 7, 8}}
\label{sec:maximal_finite_subgroups_n=7,8}

For $n\neq 7,8$, according to Prop.~II.8 of \citet{pleskenApplicationsRepresentationTheory1991}, $\mathrm{Aut}(A_n)$ is the maximal finite subgroup of $\GL(n,\mathbb Q)$. Therefore, once the low-dimensional coincidences for $n=2,3$ no longer occur, the statement that every crystallographic group is conjugate to a subgroup of some affine reflection group via an invertible linear transformation fails.

The cases $n=7,8$ are exceptional: according to Prop.~II.8 of \citet{pleskenApplicationsRepresentationTheory1991}, one has
\begin{equation}
    \mathrm{Aut}(A_7)\subset W(E_7),\quad \mathrm{Aut}(A_8)\subset W(E_8), 
\end{equation}
and in these dimensions $W(E_7)$ and $W(E_8)$ are maximal finite subgroups.  Therefore a counterexample for $n=7,8$ requires a more delicate construction.

\begin{lemma}[\citet{pleskenApplicationsRepresentationTheory1991}, Prop. II.6]
\label{lem:plesken_II6_converse}
Let $G_i\le \GL(n_i,\mathbb Q)$ be maximal finite irreducible subgroups for $i=1,\dots,k$ with $k>1$, and let
\begin{equation}
    G:=\mathrm{Diag}(G_1,\dots,G_k)\le \GL(n_1+\cdots+n_k,\mathbb Q). 
\end{equation}
Assume that, for $p=2$, at most one of the $G_i$ has the trivial Brauer character occurring in the restriction of its natural character to the $2$-regular classes. If no two of the $G_i$ are primitively related, then $G$ is maximal finite in $\GL(n_1+\cdots+n_k,\mathbb Q)$.
\end{lemma}

\begin{theorem}
\label{thm:no_affine_reflection_supergroup_n_eq_7_8}
For $n=7,8$, there exists a crystallographic group which is not
linearly conjugate to any subgroup of an affine reflection group.
\end{theorem}

\begin{proof}
We only need to show that there exists a maximal $\mathbb Q$-class in $\GL(n,\mathbb Q)$ which is not represented by any Weyl group.

For $n=8$, according to Thm.~II.14 of \citet{pleskenApplicationsRepresentationTheory1991}, there exists a maximal finite irreducible subgroup of $\GL(8,\mathbb Q)$ which is not a Weyl group, namely $\mathrm{Aut}(4A_2)$. It satisfies
\begin{equation}
    \mathrm{Aut}(4A_2)\cong \mathscr{S}_4\ltimes \mathrm{Aut}(A_2)^4. 
\end{equation}
It is straightforward to verify that $\mathrm{Aut}(4A_2)$ is not $\mathbb Q$-conjugate to any Weyl group of $\mathbb R^8$, similar to the proof of \cref{thm:no_affine_reflection_supergroup_n_neq_7_8}. Hence there exists a maximal $\mathbb Q$-class in $\GL(8,\mathbb Q)$ which is not of Weyl group. Since $\operatorname{Aut}(4A_{2})$ acts irreducibly on
$\mathbb{Q}^{8}$, any Weyl group $\mathbb{Q}$-conjugate to it
would have to arise from an irreducible root system of rank $8$.
However,
\begin{equation}
    \left|\operatorname{Aut}(4A_{2})\right|
    =
    |S_{4}|\left|\operatorname{Aut}(A_{2})\right|^{4}
    =
    24\cdot 12^{4}
    =
    497\,664,
\end{equation}
whereas

\begin{equation}
\begin{aligned}
    |W(A_{8})| &=362\,880, &
    |W(B_{8})| &=|W(C_{8})| =10\,321\,920,\\
    |W(D_{8})| &=5\,160\,960, &
    |W(E_{8})| &=696\,729\,600.
\end{aligned}
\end{equation}

Hence $\operatorname{Aut}(4A_{2})$ is not
$\mathbb{Q}$-conjugate to any Weyl group of rank $8$.

For $n=7$, according to Thm.~II.14 of \citet{pleskenApplicationsRepresentationTheory1991}, all maximal finite irreducible subgroups of $\GL(7,\mathbb Q)$ are $W(B_7)$ and $W(E_7)$, hence are Weyl groups. Thus we construct a maximal finite subgroup via a block-diagonal (reducible) embedding. We claim that
\begin{equation}
    \mathrm{Aut}(A_3 \times A_4)=\mathrm{Diag}\bigl(\mathrm{Aut}(A_3),\mathrm{Aut}(A_4)\bigr)\le \GL(7,\mathbb Q) 
\end{equation}
is maximal finite and is not Weyl group. According to Thm.~II.14 of \citet{pleskenApplicationsRepresentationTheory1991}, $\mathrm{Aut}(A_3)$ and $\mathrm{Aut}(A_4)$ are maximal finite irreducible subgroups of $\GL(3,\mathbb Q)$ and $\GL(4,\mathbb Q)$. Therefore, by \cref{lem:plesken_II6_converse}, it suffices to verify:

(i) $\mathrm{Aut}(A_3)$ and $\mathrm{Aut}(A_4)$ are not primitively related, i.e.\ there does not exist an irreducible
$H\subset \GL(m,\mathbb Q)$ such that
\begin{equation}
    \mathrm{Aut}(A_3)\subset H \wr\mathscr{S}_{k_1},\quad
\mathrm{Aut}(A_4)\subset  H \wr \mathscr{S}_{k_2},
\quad k_1m=3,\ k_2m=4. 
\end{equation}

Since $\gcd(3,4)=1$, hence $m=1$, $k_1=3$ and $k_2=4$, therefore one has $|H| \le 2$ and $|H \wr \mathscr{S}_4| = 2^4 \cdot 4! = 384$ or $|H \wr \mathscr{S}_4| = 1^4 \cdot 4! = 24$. Since $|\mathrm{Aut}(A_4)| = 240$, it only possible that $|H \wr \mathscr{S}_4| = 384$. However $240$ is divisible by $5$, whereas $384$ is not. Therefore, by Lagrange’s theorem, the inclusion $\mathrm{Aut}(A_4)\subset  H \wr \mathscr{S}_4$ cannot hold.

(ii) For $p=2$, at most one of $\mathrm{Aut}(A_3)$ and $\mathrm{Aut}(A_4)$ has the trivial Brauer character occurring in the
restriction of its natural character to the $2$-regular classes. 

It suffices to show that the trivial Brauer character does not occur for $\mathrm{Aut}(A_4)$ in the restriction of its natural character to the $2$-regular classes, hence at most one of the two factors has the trivial constituent. Consider the natural action of $\mathrm{Aut}(A_4)$ on the $A_4$ root lattice
\begin{equation}
    L=\Bigl\{\mathbf{x}=(x_1,x_2,x_3,x_4, x_5)\in\mathbb Z^5\ \Big|\ \sum_{i=1}^5 x_i=0\Bigr\}, 
\end{equation}
which is a rank-$4$ lattice embedded as a hyperplane in $\mathbb R^5$. Recall that
\begin{equation}
    \operatorname{Aut}(A_4)
    \cong
    \{\pm I\}\times \mathscr S_5.
\end{equation}

Explicitly, for $\varepsilon\in\{\pm1\}$ and
$\sigma\in\mathscr S_5$, the action is
\begin{equation}
    (\varepsilon,\sigma)\cdot
    (x_1,\ldots,x_5)
    =
    \varepsilon
    \bigl(
        x_{\sigma^{-1}(1)},\ldots,x_{\sigma^{-1}(5)}
    \bigr).
\end{equation}
Thus $\mathscr S_5$ permutes the coordinates, while the central
element $-I$ acts by global sign change.

Reducing the natural representation modulo $2$, which is equivalent to considering the $\mathbb F_2$-module
\begin{equation}
    V := \Bigl\{(x_1,x_2,x_3,x_4, x_5)\in\mathbb F_2^5\ \Big|\ \sum_{i=1}^5 x_i=0\Bigr\}. 
\end{equation}

Since $-1=1$ in characteristic $2$, the global sign change acts trivially on $V$. Hence the action of $\operatorname{Aut}(A_4)$ on $V$ factors through the coordinate-permutation action of $\mathscr S_5$. 

We claim that $V$ is an irreducible $\mathscr S_5$-module.
Let $0\neq U\subseteq V$ be an $\mathscr S_5$-invariant subspace,
and choose $0\neq v=(a_1,\ldots,a_5)\in U$. The coordinates of $v$ cannot all be equal. Indeed, if $a_1=\cdots=a_5=a$, then the relation of $V$ gives
\begin{equation}
    0=\sum_{i=1}^{5}a_i=5a=a
\end{equation}
in characteristic $2$, and hence $v=0$, a contradiction. Therefore there exist $i\neq j$ such that $a_i\neq a_j$. Let $\tau=(ij)\in\mathscr S_5$ be the transposition interchanging the $i$-th and $j$-th coordinates. Since $U$ is $\mathscr S_5$-invariant,
\begin{equation}
    v-\tau v
    =
    (x_i-x_j)(e_i-e_j)
    \in U.
\end{equation}
Because $a_i-a_j\neq0$, it follows that $e_i-e_j\in U$. Applying coordinate permutations shows that $e_p-e_q\in U$ for all $p\neq q$. The vectors $e_1-e_5,\ e_2-e_5,\ e_3-e_5,\ e_4-e_5$ form a basis of $V$. Consequently, $U=V$, proving that $V$ is irreducible.

Since $\dim V=4$, this irreducible module is nontrivial. Therefore its Brauer character is a nontrivial irreducible Brauer character, and the trivial Brauer character does not occur in the restriction of the natural character of
$\operatorname{Aut}(A_4)$ to the $2$-regular classes. Hence at most one of $\operatorname{Aut}(A_3)$ and $\operatorname{Aut}(A_4)$ has the trivial Brauer character occurring in the restriction of its natural character to the $2$-regular classes, as required.

Consequently, both hypotheses of \cref{lem:plesken_II6_converse} are satisfied, and $\mathrm{Aut}(A_3 \times A_4)$ is maximal finite
in $\GL(7,\mathbb Q)$. Finally, similar to the proof of \cref{thm:no_affine_reflection_supergroup_n_neq_7_8}, it is straightforward to check that $\mathrm{Aut}(A_3\times A_4)$ is not $\mathbb Q$-conjugate to any Weyl group of $\mathbb R^7$. The natural rational representation of
$\operatorname{Aut}(A_{3}\times A_{4})$ has two irreducible
constituents of dimensions $3$ and $4$. Since these dimensions
are distinct, any Weyl group $\mathbb{Q}$-conjugate to it would
have to arise from a root system with irreducible components of
ranks $3$ and $4$, and its rank-$4$ factor would have to be
$\mathbb{Q}$-conjugate to $\operatorname{Aut}(A_{4})$, whose
order is
\begin{equation}
    \left|\operatorname{Aut}(A_{4})\right|=240.
\end{equation}
However, the irreducible rank-$4$ Weyl groups have orders
\begin{equation}
\begin{aligned}
    |W(A_{4})| &=120, &
    |W(D_{4})| &=192,\\
    |W(B_{4})|=|W(C_{4})| &=384, &
    |W(F_{4})| &=1152,
\end{aligned}
\end{equation}
none of which equals $240$; hence
$\operatorname{Aut}(A_{3}\times A_{4})$ is not
$\mathbb{Q}$-conjugate to any Weyl group of rank $7$.

Hence there exists a maximal $\mathbb Q$-class in $\GL(7,\mathbb Q)$ which is not of Weyl group.
\end{proof}

\clearpage
\section{Hironaka Decomposition}
\label{sec:hironaka_decomposition_rn}

\subsection{Crystallographic Groups and Multiplicative Invariant Theory}
\label{app:mult_invariants}

Let $G$ be a crystallographic group, and identify its translation
subgroup $T(G)$ with its lattice $L\subset\mathbb{R}^n$ of
translation vectors. Its reciprocal lattice is
\begin{equation}
L^\ast
:=
\left\{
\lambda\in\mathbb{R}^n:
\langle\lambda,\ell\rangle\in\mathbb{Z}
\text{ for every }\ell\in L
\right\}.
\end{equation}
For $\lambda\in L^\ast$, define the character
\begin{equation}
X^\lambda(\mathbf{x})
:=
\exp\bigl(2\pi\mathrm{i}\langle\lambda,\mathbf{x}\rangle\bigr).
\end{equation}
The associated complex lattice algebra is
\begin{equation}
\mathbb{C}[L^\ast]
:=
\bigoplus_{\lambda\in L^\ast}\mathbb{C}X^\lambda,
\quad
X^\lambda X^\mu=X^{\lambda+\mu}.
\end{equation}

Define the conjugate-linear involution $\tau$ on
$\mathbb{C}[L^\ast]$ by
\begin{equation}
\tau\left(
\sum_{\lambda\in L^\ast}a_\lambda X^\lambda
\right)
=
\sum_{\lambda\in L^\ast}
\overline{a_\lambda}X^{-\lambda}.
\end{equation}
Since $\tau(f)(\mathbf{x})=\overline{f(\mathbf{x})}$, the real
algebra of combinations of trigonometric functions is
\begin{equation}
P_{\mathrm{torus}}
:=
\mathbb{C}[L^\ast]^\tau.
\end{equation}

Equivalently, choose a $\mathbb{Z}$-basis
$\alpha_1,\ldots,\alpha_n$ of $L^\ast$, set
\begin{equation}
e_j:=X^{\alpha_j},
\quad
c_j:=\frac{e_j+e_j^{-1}}{2},
\quad
s_j:=\frac{e_j-e_j^{-1}}{2\mathrm{i}},
\end{equation}
and obtain
\begin{equation}
P_{\mathrm{torus}}
=
\mathbb{R}[c_1,s_1,\ldots,c_n,s_n]
\cong
\frac{
\mathbb{R}[u_1,v_1,\ldots,u_n,v_n]
}{
\left\langle
\{u_j^2+v_j^2-1|\ 1\leq j\leq n\}
\right\rangle
}.
\end{equation}
Its complexification is
\begin{equation}
\mathbb{C}\otimes_{\mathbb{R}}P_{\mathrm{torus}}
\cong
\mathbb{C}[L^\ast]
\cong
\mathbb{C}[e_1^{\pm1},\ldots,e_n^{\pm1}].
\end{equation}

\begin{lemma}
\label{lem:symmorphic_torus_invariants}
Let $G$ be a symmorphic crystallographic group. Identify $T(G)$ with its
lattice $L$ of translation vectors, and let $L^\ast=T(G)^\ast$ be the reciprocal lattice. Then
\begin{equation}
P_{\mathrm{torus}}^G
=
\left(\mathbb{C}[L^\ast]^{P(G)}\right)^\tau.
\end{equation}
Moreover, the natural complexification map induces an isomorphism
\begin{equation}
\mathbb{C}\otimes_{\mathbb{R}}P_{\mathrm{torus}}^G
\cong
\mathbb{C}[L^\ast]^{P(G)}.
\end{equation}
\end{lemma}

\begin{proof}
For $\ell\in L$ and $\lambda\in L^\ast$, translation by $\ell$
acts on $X^\lambda$ by
\begin{equation}
\begin{aligned}
(t_\ell\cdot X^\lambda)(\mathbf{x})
&=
X^\lambda(\mathbf{x}-\ell)\\
&=
\exp\bigl(-2\pi\mathrm{i}
\langle\lambda,\ell\rangle\bigr)
X^\lambda(\mathbf{x})\\
&=
X^\lambda(\mathbf{x}),
\end{aligned}
\end{equation}
because $\langle\lambda,\ell\rangle\in\mathbb{Z}$. Thus $T(G)$ acts trivially on the torus algebra. Since $G$ is
symmorphic, its effective action on $P_{\mathrm{torus}}$ is
precisely the induced action of $P(G)$ on $L^\ast$.

The action of $P(G)$ commutes with $\tau$. Indeed, for
$p\in P(G)$,
\begin{equation}
\tau\bigl(p\cdot X^\lambda\bigr) = \tau\bigl(X^{p\lambda}\bigr) = X^{-p\lambda} = p\cdot X^{-\lambda} = p\cdot\tau(X^\lambda).
\end{equation}
Consequently,
\begin{equation}
P_{\mathrm{torus}}^G
=
\left(\mathbb{C}[L^\ast]^\tau\right)^{P(G)}
=
\left(\mathbb{C}[L^\ast]^{P(G)}\right)^\tau.
\end{equation}

Set $B:=\mathbb{C}[L^\ast]^{P(G)}$. For every $b\in B$, define
\begin{equation}
b_1:=\frac{b+\tau(b)}{2},
\quad
b_2:=\frac{b-\tau(b)}{2\mathrm{i}}.
\end{equation}

Then $b=b_1+\mathrm{i}b_2, b_1,b_2\in B^\tau$. Therefore, $B=B^\tau\oplus\mathrm{i}B^\tau$ as real vector spaces. It follows that the natural map
\begin{equation}
\mathbb{C}\otimes_{\mathbb{R}}B^\tau
\longrightarrow B
\end{equation}
is an isomorphism of complex algebras. Since $B^\tau=P_{\mathrm{torus}}^G$, the desired complexification isomorphism follows.
\end{proof}

This lemma shows that for a root system $\Phi$, the action of the affine reflection group $W_{a}$ on the reciprocal-space algebra
$\mathbb{C} \otimes_{\mathbb{R}} P_{\mathrm{torus}}$
is equivalent to the action of $W_{a}$ on the weight lattice $\Lambda(\Phi)$.  The following lemma shows the polynomial structure of $\mathbb{C} \otimes_{\mathbb{R}} P_{\mathrm{torus}}^{W_a}$ as a complex algebra.

\begin{proposition}[\citet{lorenzMultiplicativeInvariantTheory2005}, Thm.~3.6.1]
\label{prop:weyl_weight_lattice_invariants}
Let $\Phi$ be a reduced crystallographic root system of rank $n$,
with Weyl group $W=W(\Phi)$, weight lattice
$\Lambda=\Lambda(\Phi)$, and fundamental weights
$\omega_1,\ldots,\omega_n$. Define
\begin{equation}
\vartheta_i
:=
\operatorname{orb}(\omega_i)
=
\sum_{\lambda\in W\omega_i}X^\lambda,
\quad
i=1,\ldots,n.
\end{equation}
Then $\vartheta_1,\ldots,\vartheta_n$ are algebraically
independent and
\begin{equation}
\mathbb{Z}[\Lambda]^W
=
\mathbb{Z}[\vartheta_1,\ldots,\vartheta_n].
\end{equation}
Consequently,
\begin{equation}
\mathbb{C}[\Lambda]^W
=
\mathbb{C}[\vartheta_1,\ldots,\vartheta_n].
\end{equation}
\end{proposition}

\textit{Remark.}
This theorem constitutes a generalization of the Chevalley–Shephard–Todd theorem from classical invariant theory to multiplicative invariant theory. The sufficiency of the reflection group condition was established by \citet{bourbakiLieGroupsLie2002}, while the necessity was proven by \citet{farkasReflectionGroupsMultiplicative1986}. \citet{lorenzMultiplicativeInvariantTheory2005} discusses results for general $R[L]^G$, where $R$ is a regular commutative ring which covers the cases of $\mathbb{Z}$ and $\mathbb{C}$.

\begin{lemma}
\label{lem:real_weyl_weight_lattice_invariants}
Under the assumptions of
\cref{prop:weyl_weight_lattice_invariants}, define the
conjugate-linear involution
\begin{equation}
\tau(aX^\lambda)
=
\overline{a}X^{-\lambda}.
\end{equation}
Then $\left(\mathbb{C}[\Lambda]^W\right)^\tau$ is a polynomial algebra over $\mathbb{R}$ in $n$ variables.
\end{lemma}

\begin{proof}
Let $w_0$ be the longest element of $W$. There is an involution
$i\mapsto i^\ast$ of the set of fundamental weights determined by
\begin{equation}
-w_0\omega_i=\omega_{i^\ast}.
\end{equation}
For the orbit-sum generators in
\cref{prop:weyl_weight_lattice_invariants}, we have
\begin{equation}
\tau(\vartheta_i)
=
\sum_{\lambda\in W\omega_i}X^{-\lambda}
=
\operatorname{orb}(-\omega_i)
=
\operatorname{orb}(-w_0\omega_i)
=
\vartheta_{i^\ast}.
\end{equation}
The third equality follows because $-\omega_i$ and $-w_0\omega_i$ belong to the same $W$-orbit. Therefore, the complex polynomial generators are either fixed by $\tau$ or occur in pairs exchanged by $\tau$. The complex generators can be used to construct algebraically independent real generators, showing that the fixed algebra is a polynomial algebra over $\mathbb{R}$.

Let $I_0:=\{i:i=i^\ast\}$, and let $I_1$ contain one representative from each two-element orbit $\{i,i^\ast\}$ that $i\neq i^\ast$. For $i\in I_1$, define
\begin{equation}
u_i
:=
\frac{\vartheta_i+\vartheta_{i^\ast}}{2},
\quad
v_i
:=
\frac{\vartheta_i-\vartheta_{i^\ast}}{2\mathrm{i}}.
\end{equation}
Since $\tau$ is conjugate-linear and exchanges
$\vartheta_i$ and $\vartheta_{i^\ast}$, we obtain $u_i$ and $v_i$ are invariant to $\tau$. For $i\in I_0$, we have $\vartheta_i$ are also invariant to $\tau$. The change of variables is invertible over $\mathbb{C}$ by $\vartheta_i=u_i+\mathrm{i}v_i$ and $\vartheta_{i^\ast}=u_i-\mathrm{i}v_i$. Therefore,
\begin{equation}
\mathbb{C}[\Lambda]^W
=
\mathbb{C}\left[
\{\vartheta_i\}_{i\in I_0},
\{u_i,v_i\}_{i\in I_1}
\right].
\end{equation}

Since these generators are algebraically independent, every
element of $\mathbb{C}[\Lambda]^W$ has a unique polynomial
expression in them. The involution $\tau$ fixes the generators
and conjugates only the coefficients. Hence an element is fixed
by $\tau$ if and only if all coefficients in this expression are
real. It follows that
\begin{equation}
\left(\mathbb{C}[\Lambda]^W\right)^\tau
=
\mathbb{R}\left[
\{\vartheta_i\}_{i\in I_0},
\{u_i,v_i\}_{i\in I_1}
\right].
\end{equation}
Finally, $|I_0|+2|I_1|=n$, so this is a polynomial algebra over $\mathbb{R}$ in $n$ variables.
\end{proof}

\begin{corollary}
\label{cor:torus_invariants_poly_iff_reflection}
Let $G$ be a symmorphic crystallographic group, and let $L^\ast=T(G)^\ast$ be its reciprocal lattice. Then
$P_{\mathrm{torus}}^G$ is a polynomial algebra over
$\mathbb{R}$ if and only if the action of $P(G)$ on $L^\ast$
is equivalent to the action of the Weyl group $W(\Phi)$ on the
weight lattice $\Lambda(\Phi)$ of some reduced root system
$\Phi$.
\end{corollary}

\begin{proof}
Suppose first that the action of $P(G)$ on $L^\ast$ is equivalent
to the action of $W(\Phi)$ on $\Lambda(\Phi)$. Thus, there exist
a group isomorphism $\eta:P(G)\longrightarrow W(\Phi)$ and a lattice isomorphism $\varphi:L^\ast\longrightarrow\Lambda(\Phi)$ such that
\begin{equation}
\varphi(p\lambda)
=
\eta(p)\varphi(\lambda)
\end{equation}
for every $p\in P(G)$ and $\lambda\in L^\ast$. The induced algebra isomorphism
\begin{equation}
\varphi_\#:
\mathbb{C}[L^\ast]
\longrightarrow
\mathbb{C}[\Lambda(\Phi)],
\quad
X^\lambda\longmapsto X^{\varphi(\lambda)},
\end{equation}
intertwines the two group actions. It also commutes with the real
involution because $\varphi(-\lambda)=-\varphi(\lambda)$. Consequently,
\begin{equation}
\left(\mathbb{C}[L^\ast]^{P(G)}\right)^\tau
\cong
\left(
\mathbb{C}[\Lambda(\Phi)]^{W(\Phi)}
\right)^\tau.
\end{equation}
By
\cref{lem:symmorphic_torus_invariants,%
lem:real_weyl_weight_lattice_invariants}, the left-hand side is
$P_{\mathrm{torus}}^G$, while the right-hand side is a
polynomial algebra over $\mathbb{R}$ in $n$ variables.
Therefore,
\begin{equation}
P_{\mathrm{torus}}^G
\cong
\mathbb{R}[\theta_1,\ldots,\theta_n].
\end{equation}

Conversely, suppose that $P_{\mathrm{torus}}^G$ is a polynomial
algebra over $\mathbb{R}$. By
\cref{lem:symmorphic_torus_invariants},
\begin{equation}
\mathbb{C}[L^\ast]^{P(G)}
\cong
\mathbb{C}\otimes_{\mathbb{R}}P_{\mathrm{torus}}^G.
\end{equation}
Hence $\mathbb{C}[L^\ast]^{P(G)}$ is a polynomial algebra over
$\mathbb{C}$. By Cor.~7.1.2 of \citet{lorenzMultiplicativeInvariantTheory2005},
the action of $P(G)$ on $L^\ast$ is equivalent to the action of
a Weyl group on the weight lattice of a reduced root system.
\end{proof}

\subsection{Cohen-Macaulay Property and Hironaka Decomposition}
\label{sec:cm_promperty}

\begin{proposition}[\citet{lorenzMultiplicativeInvariantTheory2005}, Thm. 8.4.2]
\label{prop:cm_and_free}
Let $A$ be a commutative affine $\mathbb{R}$-domain, and let $P$ be a polynomial subalgebra such that $A$ is module-finite over
$P$. Then $A$ is Cohen--Macaulay if and only if $A$ is a free
$P$-module.
\end{proposition}

\begin{lemma}
\label{lem:torus_cm}
For every crystallographic group $G$, the invariant ring $P_{\mathrm{torus}}^G$ satisfies the Cohen-Macaulay property.
\end{lemma}

\begin{proof}
Set $S:=P_{\mathrm{torus}}$. Recall that
\begin{equation}
S
\cong
\frac{
\mathbb{R}[c_1,s_1,\ldots,c_n,s_n]
}{
\left\langle
\{c_j^2+s_j^2-1|1\leq j\leq n\}
\right\rangle
}.
\end{equation}
Since each defining relation is monic in $s_j$, repeated
polynomial division gives the direct-sum decomposition
\begin{equation}
S
=
\bigoplus_{\boldsymbol{\epsilon}\in\{0,1\}^n}
\mathbb{R}[c_1,\ldots,c_n]\,
s_1^{\epsilon_1}\cdots s_n^{\epsilon_n}.
\label{eq:torus_free_cosine_module}
\end{equation}
In particular, $c_1,\ldots,c_n$ are algebraically independent,
and $S$ is a finite free module of rank $2^n$ over the
polynomial subalgebra
$\mathbb{R}[c_1,\ldots,c_n]$. Hence, by
\cref{prop:cm_and_free}, $S$ is Cohen--Macaulay.

The proof strategy employed below, utilizing the Reynolds averaging operator, is a standard technique in classical invariant theory, see \citet{stanleyInvariantsFiniteGroups1979}. Let
$\overline{G}:=G/T(G)$, which is finite. Since $T(G)$ acts
trivially on the torus algebra, one has
$P_{\mathrm{torus}}^G=S^{\overline{G}}$. 
For every $f\in S$, consider the orbit polynomial
\begin{equation}
Q_f(t)
:=
\prod_{\gamma\in\overline{G}}
\bigl(t-\gamma(f)\bigr).
\end{equation}
The group $\overline{G}$ permutes the factors, so
$Q_f(t)\in S^{\overline{G}}[t]$, and $Q_f(f)=0$. Therefore
$S$ is integral over $S^{\overline{G}}$.

Moreover, the averaging map
\begin{equation}
\rho_{\overline{G}}(f)
:=
\frac{1}{|\overline{G}|}
\sum_{\gamma\in\overline{G}}\gamma(f)
\end{equation}
is an $S^{\overline{G}}$-linear Reynolds operator from $S$
onto $S^{\overline{G}}$. Since $S$ is Cohen--Macaulay,
the Cohen--Macaulay descent theorem for an integral extension
admitting a Reynolds operator
implies that $S^{\overline{G}}$ is Cohen--Macaulay.
Consequently,
\begin{equation}
P_{\mathrm{torus}}^G
=
S^{\overline{G}}
\end{equation}
is Cohen--Macaulay.
\end{proof}

\begin{lemma}[Noether Normalization, \citet{mumfordRedBookVarieties2004}, Sec. 1.1]
\label{lem:noether_normalization}
Let $R$ be a finitely generated commutative integral domain over a field $k$.
Assume that the transcendence degree of $R$ over $k$ is $n$, i.e., the maximal number of algebraically independent elements in $R$ over $k$ is $n$.
Then there exist $n$ algebraically independent elements $\theta_1,\dots,\theta_n\in R$ such that $R$ is a finitely generated module over the subring $k[\theta_1,\dots,\theta_n]$.
\end{lemma}

\begin{lemma}
\label{lem:trdeg_torus_invariant}
For every crystallographic group $G$ acting on
$\mathbb{R}^n$, the invariant ring
$P_{\mathrm{torus}}^G$ is a finitely generated $\mathbb R$-domain of transcendence degree $n$ over $\mathbb R$.
\end{lemma}

\begin{proof}
Set
$S:=P_{\mathrm{torus}}$ and
$A:=P_{\mathrm{torus}}^G$.
By \cref{eq:torus_free_cosine_module}, $S$ is a finite free
module over $\mathbb{R}[c_1,\ldots,c_n]$. Hence
$c_1,\ldots,c_n$ are algebraically independent, and the
fraction field
$\operatorname{Frac}(S)$ is algebraic over
$\mathbb{R}(c_1,\ldots,c_n)$. Indeed, each $s_j$ satisfies $s_j^2=1-c_j^2$. Therefore, the transcendence degree $\mathbb{R}$ satisfies
\begin{equation}
\operatorname{trdeg}_{\mathbb{R}}
\operatorname{Frac}(S)
=
n.
\label{eq:trdeg_full_torus}
\end{equation}

As shown in the proof of \cref{lem:torus_cm}, the
finite group $\overline{G}=G/T(G)$ acts on $S$, and every
$f\in S$ satisfies the monic orbit polynomial
\begin{equation}
\prod_{\gamma\in\overline{G}}
\bigl(t-\gamma(f)\bigr)
\in A[t].
\end{equation}
Thus $S$ is integral over $A$. It follows that the field
extension $\operatorname{Frac}(A)\subseteq \operatorname{Frac}(S)$ is algebraic. Algebraic field extensions do not change
transcendence degree, so by
\cref{eq:trdeg_full_torus},
\begin{equation}
\operatorname{trdeg}_{\mathbb{R}}A
=
\operatorname{trdeg}_{\mathbb{R}}S
=
n.
\end{equation}

Finally, $S$ is generated over $A$ by
$c_1,s_1,\ldots,c_n,s_n$, each of which is integral over
$A$. Hence $S$ is a finite $A$-module. Since $S$ is a finitely generated $\mathbb{R}$-algebra, Prop.~7.8 in \citet{atiyahIntroductionCommutativeAlgebra1994} implies that $A$ is also a finitely generated $\mathbb{R}$-algebra.
\end{proof}

\begin{theorem}[Hironaka Decomposition]
\label{thm:hironaka_decomposition_poly}
For $P_{\mathrm{torus}}^G$, there always exist $n$ primary invariants $\theta_1, \dots, \theta_n$ and a set of secondary invariants $g_1, \dots, g_r$, such that any invariant polynomial $f$ can be uniquely decomposed as
\begin{equation}
    f(\mathbf{x}) = \sum_{i=1}^{r} p_i(\theta_1, \dots, \theta_n)\, g_{i}(\mathbf{x}), 
\end{equation}
where each $p_i$ is a polynomial in $\mathbb{R}[\theta_1,\dots,\theta_n]$.
In particular, for affine reflection groups, we may take $r = 1$ and $g_{1}(\mathbf{x}) = 1$.
\end{theorem}

\begin{proof}
For affine reflection groups, the polynomial structure of the invariant ring is given by \cref{cor:torus_invariants_poly_iff_reflection}, which yields $r=1$ and $g_1(\mathbf{x})=1$. For other case, set $A:=P_{\mathrm{torus}}^G$. By \cref{lem:trdeg_torus_invariant}, $A$ is a finitely generated
$\mathbb{R}$-domain satisfying $\operatorname{trdeg}_{\mathbb{R}}A=n$. \cref{lem:noether_normalization} therefore gives algebraically independent
elements $\theta_1,\ldots,\theta_n\in A$ such that $A$ is a
finite module over
\begin{equation}
P:=\mathbb{R}[\theta_1,\ldots,\theta_n].
\end{equation}

By \cref{lem:torus_cm}, $A$ is Cohen--Macaulay.
Hence, by \cref{prop:cm_and_free}, $A$ is a finite free
$P$-module. Choose a $P$-basis $g_1,\ldots,g_r$. Then every
$f\in A$ has a unique decomposition
\begin{equation}
f
=
\sum_{i=1}^r
p_i(\theta_1,\ldots,\theta_n)g_i,
\end{equation}
where $p_i\in\mathbb{R}[t_1,\ldots,t_n]$.

\end{proof}

\subsection{Choice of Primary Invariants}
\label{sec:choice_of_primary_inv}

\begin{proposition}
Let $G\leq K$ be crystallographic groups. Assume that
$P_{\mathrm{torus}}^K$ is a polynomial algebra over $\mathbb R$.
Then $P_{\mathrm{torus}}^G$ is a finitely generated free
$P_{\mathrm{torus}}^K$-module, and the rank of this module equals
$[K:G]$.
\end{proposition}

\begin{proof}

Since $G\leq K$, their translation lattices satisfy
$T(G)\leq T(K)$, with finite index. Set
\begin{equation}
    m:=[T(K):T(G)],
    \quad
    T:=mT(K).
\end{equation}
Since the finite group $T(K)/T(G)$ has order $m$, one has
$T=mT(K)\subseteq T(G)$. Moreover, every element of $K$
preserves $T(K)$ and therefore also preserves $mT(K)$.
Hence $T$ is normal in $K$, and thus also in $G$. In
particular, $T$ is a common finite-index translation
subgroup of $G$ and $K$. 

Let $R$ be the torus algebra associated with $T$, and set $\overline{G}:=G/T$ and $\overline{K}:=K/T$. Then $\overline{G}$ and $\overline{K}$ are finite groups acting faithfully on $R$, with
$\overline{G}\leq\overline{K}$, and $R^G=R^{\overline{G}}$ and $R^K=R^{\overline{K}}$. These invariant rings coincide with
$P_{\mathrm{torus}}^G$ and $P_{\mathrm{torus}}^K$,
respectively. Since $\overline{K}$ is finite, $R$ is module-finite over
$R^{\overline{K}}$. Therefore,
$R^{\overline{G}}\subseteq R$ is also a finite
$R^{\overline{K}}$-module, because
$R^{\overline{K}}$ is Noetherian. By
\cref{lem:torus_cm}, $R^{\overline{G}}$ is
Cohen--Macaulay, while $R^{\overline{K}}$ is polynomial.
Hence \cref{prop:cm_and_free} shows that
$R^{\overline{G}}$ is a finite free
$R^{\overline{K}}$-module.

Let $F:=\operatorname{Frac}(R)$. For a finite group acting
faithfully on a domain, finite-group Galois theory gives
\begin{equation}
    [F:F^{\overline{G}}]=|\overline{G}|,
    \quad
    [F:F^{\overline{K}}]=|\overline{K}|.
\end{equation}
Moreover,
$\operatorname{Frac}(R^{\overline{G}})=F^{\overline{G}}$
and
$\operatorname{Frac}(R^{\overline{K}})=F^{\overline{K}}$.
Thus,
\begin{equation}
\begin{aligned}
\left[
\operatorname{Frac}(R^{\overline{G}}):
\operatorname{Frac}(R^{\overline{K}})
\right]
=
\frac{|\overline{K}|}{|\overline{G}|}
=
[\overline{K}:\overline{G}]
=
[K:G].
\end{aligned}
\end{equation}

A basis of $R^{\overline{G}}$ as an
$R^{\overline{K}}$-module induces a spanning set of
$\operatorname{Frac}(R^{\overline{G}})$ as a vector space
over $\operatorname{Frac}(R^{\overline{K}})$. Since the
dimension of this vector space equals
\begin{equation}
\left[
\operatorname{Frac}(R^{\overline{G}}):
\operatorname{Frac}(R^{\overline{K}})
\right]
=
[K:G],
\end{equation}
any generating set has cardinality at least $[K:G]$.
Since the module is free, its basis elements remain
$\operatorname{Frac}(R^{\overline{K}})$-linearly
independent and hence form a vector-space basis.
Therefore, the number of basis elements is exactly
$[K:G]$.

\end{proof}

If $G$ admits an affine reflection supergroup $W_a$, then by \cref{cor:torus_invariants_poly_iff_reflection} the invariant ring of $W_a$ has a polynomial structure. Consequently we obtain that the primary invariants in \cref{thm:hironaka_decomposition_poly} can be chosen as the generator of the polynomial ring.

\begin{corollary}
\label{cor:hironaka_decomposition_rank}
Let $G$ be a n-dimensional crystallographic group. If $G$ admits an affine reflection group $W_a$ as a supergroup, then one can choose $\theta_1,\ldots,\theta_n$ to be the primary invariants of $P_{\mathrm{torus}}^{W_a}$. In this case, the number of secondary invariants $r$ is exactly the index $[W_a:G]$.
\end{corollary}

\subsection{Approximation of Continuous Functions}
\label{sec:app_c_function}

\begin{theorem}
\label{thm:hironaka_func_with_higher_sym}
Let $G$ be an $n$-dimensional
crystallographic group. If $G$ admits an affine reflection group $W_{a}$ as a supergroup, there exist fixed
$G$-invariant functions $g_1,\dots,g_r$, where
$r=[W_a:G]$, such that for every
$f\in C_G(\mathbb R^n)$ and every $\epsilon>0$, there exist
functions $h_1,\dots,h_r\in C_{W_a}(\mathbb R^n)$ satisfying
\begin{equation}
    \int_{\Omega}
    \left|
        f(\mathbf{x})-\sum_{i=1}^{r}h_i(\mathbf{x})g_i(\mathbf{x})
    \right|^2
    \dd x
    <\epsilon,
\end{equation}
where $\Omega$ is a fundamental cell.
\end{theorem}

\begin{proof}
Let $\Gamma_G$ denote the finite effective image of $G$ acting
on the torus associated with the common translation lattice, and
let $\Omega$ be a fundamental cell. Since trigonometric
polynomials are dense in $L^2(\Omega)$, there exists
$p\in P_{\mathrm{torus}}$ such that
\begin{equation}
\|f-p\|_{L^2(\Omega)}<\sqrt{\epsilon}.
\end{equation}

Apply the Reynolds operator associated with $\Gamma_G$ and set
\begin{equation}
p_G
:=
\frac{1}{|\Gamma_G|}
\sum_{\gamma\in\Gamma_G}\gamma(p).
\end{equation}
Then $p_G\in P_{\mathrm{torus}}^G$. Since $f$ is
$G$-invariant and the action of $\Gamma_G$ preserves the
measure on the torus, the Reynolds operator is contractive in
$L^2(\Omega)$. Hence
\begin{equation}
\begin{aligned}
\|f-p_G\|_{L^2(\Omega)}
&=
\left\|
\frac{1}{|\Gamma_G|}
\sum_{\gamma\in\Gamma_G}
\gamma(f-p)
\right\|_{L^2(\Omega)}\\
&\leq
\frac{1}{|\Gamma_G|}
\sum_{\gamma\in\Gamma_G}
\|\gamma(f-p)\|_{L^2(\Omega)}\\
&=
\|f-p\|_{L^2(\Omega)}
<
\sqrt{\epsilon}.
\end{aligned}
\end{equation}

By \cref{cor:hironaka_decomposition_rank}, the invariant
polynomial $p_G$ admits an exact decomposition
\begin{equation}
p_G
=
\sum_{i=1}^r h_i g_i,
\quad
h_i\in P_{\mathrm{torus}}^{W_a}
\subseteq C^{W_a}(\mathbb{R}^n),\quad r = [W_a:G].
\end{equation}
Consequently,
\begin{equation}
\int_{\Omega}
\left|
f(\mathbf{x})
-
\sum_{i=1}^r h_i(\mathbf{x})g_i(\mathbf{x})
\right|^2
\,d\mathbf{x}
=
\|f-p_G\|_{L^2(\Omega)}^2
<
\epsilon.
\end{equation}
\end{proof}

\begin{corollary}
\label{cor:hironaka_approximation_n_eq_2_3}
Let $n\in\{2,3\}$ and let $G$ be an $n$-dimensional
crystallographic group. After replacing $G$ by a linearly
conjugate realization, there exists an affine reflection group
$W_a$ such that $G\leq W_a$. Consequently, there exist fixed
$G$-invariant functions $g_1,\dots,g_r$, where
$r=[W_a:G]$, such that for every
$f\in C_G(\mathbb R^n)$ and every $\epsilon>0$, there exist
functions $h_1,\dots,h_r\in C_{W_a}(\mathbb R^n)$ satisfying
\begin{equation}
    \int_{\Omega}
    \left|
        f(\mathbf{x})-\sum_{i=1}^{r}h_i(\mathbf{x})g_i(\mathbf{x})
    \right|^2
    \dd x
    <\epsilon,
\end{equation}
where $\Omega$ is a fundamental cell.
\end{corollary}

\begin{proof}
The existence of $W_a$ follows from
\cref{thm:affine_reflection_supergroup_n_eq_2_3}, and the
approximation statement then follows directly from
\cref{thm:hironaka_func_with_higher_sym}.
\end{proof}

\clearpage
\section{Secondary Invariant Construction}
\label{sec:sec_inv_r2}

\subsection{Secondary Invariant for Non-Symmorphic Crystallographic Groups}
\label{sec:sec_inv_non-symmorphic}

The results for symmorphic crystallographic groups are presented in
\citet{kimRingInvariantReal2001}.
Our primary discussion focuses on the treatment of non-symmorphic crystallographic groups.
We adopt the approach from \citet{kimRingInvariantReal2001}.
In the subsequent section, we will introduce a method based on Fourier expansions,
for which a calculation example is provided.

For the case of $n = 2$, the non-symmorphic crystallographic groups are:
$pg$ (a subgroup of $p2mm$ in the $pm$ arithmetic class);
$p2mg$ and $p2gg$ (both subgroups of $p2mm$ in the $p2mm$ arithmetic class);
and $p4gm$ (a subgroup of $p4mm$ in the $p4mm$ arithmetic class).
In our discussion, it is necessary to halve the period of the translation subgroup along the direction corresponding to the glide plane.

The calculation of primary and secondary invariants for a non-symmorphic
crystallographic group $G$ uses a chain
$H\triangleleft G\subset K$, where $K$ is an affine reflection supergroup
and $H$ is a symmorphic subgroup satisfying $[G:H]=2$.
In the first step, we set $A:=P_{\mathrm{torus}}^{K}$, whose generators
serve as the primary invariants, and express
$P_{\mathrm{torus}}^{H}$ as a free $A$-module.
In the second step, we choose $\tau\in G\setminus H$ and compute its action
on an $A$-basis of secondary invariants. Since $\tau\in K$, it fixes $A$
pointwise, while $\tau^{2}\in H$ implies that its induced action on
$P_{\mathrm{torus}}^{H}$ is an involution. Therefore,
\begin{equation}
    P_{\mathrm{torus}}^{G}
    =
    \bigl(P_{\mathrm{torus}}^{H}\bigr)^{\tau},
\end{equation}
so the secondary invariants for $G$ are obtained from the $1$ eigenspace
of this involution. In the examples below, the chosen secondary basis
diagonalizes the action, and it therefore suffices to retain the $[K:G]$
sign-invariant basis elements.
Here, the primes in $p2mm'$ and $p4mm'$ indicate the refined translation
lattices specified in the corresponding cases.

\paragraph{$pg$.}
Choose the affine reflection supergroup $K=p2mm'$, whose translation
lattice is obtained by adjoining the half-translation
$(0,\tfrac12)$. Thus
\begin{equation}
    P_{\mathrm{torus}}^{K}
    =
    \mathbb{R}[c_{1},c(2y)]
    =
    \mathbb{R}[c_{1},c_{2}^{2}].
\end{equation}
The index-$2$ symmorphic subgroup is $H=p1$, and
\begin{equation}
    P_{\mathrm{torus}}^{H}
    =
    \mathbb{R}[c_{1},c_{2}](1,s_{1})(1,s_{2})
    =
    \mathbb{R}[c_{1},c_{2}^{2}](1,c_{2})(1,s_{1})(1,s_{2}).
\end{equation}
Let $g\in G\setminus H$ be the glide reflection
$g(x,y)=(-x,y+\tfrac12)$, acting by
\begin{equation}
    (c_{1},s_{1},c_{2},s_{2})
    \mapsto
    (c_{1},-s_{1},-c_{2},-s_{2}).
\end{equation}
Thus the three factors $c_{2}$, $s_{1}$, and $s_{2}$ are sign-changing,
and the fixed secondary basis consists of their products containing an
even number of these factors. Hence $[K:G]=4$ and
\begin{equation}
    P_{\mathrm{torus}}^{G}
    =
    P_{\mathrm{torus}}^{pg}
    =
    \mathbb{R}[c_{1},c_{2}^{2}]
    \bigl(1,\ s_{1}c_{2},\ s_{1}s_{2},\ c_{2}s_{2}\bigr).
\end{equation}

\paragraph{$p2mg$.}
Choose the affine reflection supergroup $K=p2mm'$, whose translation
lattice is obtained by adjoining the half-translation
$(\tfrac12,0)$. Thus
\begin{equation}
    P_{\mathrm{torus}}^{K}
    =
    \mathbb{R}[c(2x),c_{2}]
    =
    \mathbb{R}[c_{1}^{2},c_{2}].
\end{equation}
The index-$2$ symmorphic subgroup is $H=p2$, and
\begin{equation}
    P_{\mathrm{torus}}^{H}
    =
    \mathbb{R}[c_{1},c_{2}](1,s_{1}s_{2})
    =
    \mathbb{R}[c_{1}^{2},c_{2}](1,c_{1})(1,s_{1}s_{2}).
\end{equation}
Let $m\in G\setminus H$ be the reflection across the line
$x=\tfrac14$,
\begin{equation}
    m(x,y)=\left(\tfrac12-x,y\right),
\end{equation}
which acts by
\begin{equation}
    (c_{1},s_{1},c_{2},s_{2})
    \mapsto
    (-c_{1},s_{1},c_{2},s_{2}).
\end{equation}
In the displayed secondary basis, $c_{1}$ is sign-changing, whereas
$s_{1}s_{2}$ is fixed. Therefore, the fixed basis elements are
$1$ and $s_{1}s_{2}$. Hence $[K:G]=2$ and
\begin{equation}
    P_{\mathrm{torus}}^{G}
    =
    P_{\mathrm{torus}}^{p2mg}
    =
    \mathbb{R}[c_{1}^{2},c_{2}](1,s_{1}s_{2}).
\end{equation}

\paragraph{$p2gg$.}
Choose the affine reflection supergroup $K=p2mm'$, whose translation
lattice is obtained by adjoining the half-translations
$(\tfrac12,0)$ and $(0,\tfrac12)$. Thus
\begin{equation}
    P_{\mathrm{torus}}^{K}
    =
    \mathbb{R}[c(2x),c(2y)]
    =
    \mathbb{R}[c_{1}^{2},c_{2}^{2}].
\end{equation}
The index-$2$ symmorphic subgroup is $H=p2$, and
\begin{equation}
    P_{\mathrm{torus}}^{H}
    =
    \mathbb{R}[c_{1},c_{2}](1,s_{1}s_{2})
    =
    \mathbb{R}[c_{1}^{2},c_{2}^{2}]
    (1,c_{1})(1,c_{2})(1,s_{1}s_{2}).
\end{equation}
Let $g\in G\setminus H$ be the glide reflection
\begin{equation}
    g(x,y)=\left(\tfrac12-x,y+\tfrac12\right),
\end{equation}
which acts by
\begin{equation}
    (c_{1},s_{1},c_{2},s_{2})
    \mapsto
    (-c_{1},s_{1},-c_{2},-s_{2}).
\end{equation}
Thus $c_{1}$, $c_{2}$, and $s_{1}s_{2}$ are all sign-changing.
The fixed secondary basis therefore consists of the products containing
an even number of these three factors. Hence $[K:G]=4$ and
\begin{equation}
    P_{\mathrm{torus}}^{G}
    =
    P_{\mathrm{torus}}^{p2gg}
    =
    \mathbb{R}[c_{1}^{2},c_{2}^{2}]
    \bigl(
        1,\ c_{1}c_{2},\
        c_{1}s_{1}s_{2},\
        c_{2}s_{1}s_{2}
    \bigr).
\end{equation}

\paragraph{$p4gm$.}
Choose the affine reflection supergroup $K=p4mm'$, whose translation
lattice is obtained by adjoining the half-translations
$(\tfrac12,0)$ and $(0,\tfrac12)$. Thus
\begin{equation}
    P_{\mathrm{torus}}^{K}
    =
    \mathbb{R}[c(2x)+c(2y),c(2x)c(2y)]
    =
    \mathbb{R}[c_{1}^{2}+c_{2}^{2},c_{1}^{2}c_{2}^{2}].
\end{equation}
The index-$2$ symmorphic subgroup is $H=p4$, and
\begin{align}
    P_{\mathrm{torus}}^{H}
    &=
    \mathbb{R}[c_{1}+c_{2},c_{1}c_{2}]
    \bigl(1,(c_{1}-c_{2})s_{1}s_{2}\bigr) \\
    &=
    \mathbb{R}[c_{1}^{2}+c_{2}^{2},c_{1}^{2}c_{2}^{2}]
    (1,c_{1}+c_{2})(1,c_{1}c_{2})
    \bigl(1,(c_{1}-c_{2})s_{1}s_{2}\bigr).
\end{align}
Let $m\in G\setminus H$ be the reflection across the line
$y=-x+\tfrac12$,
\begin{equation}
    m(x,y)
    =
    \left(\tfrac12-y,\tfrac12-x\right),
\end{equation}
which acts by
\begin{equation}
    (c_{1},s_{1},c_{2},s_{2})
    \mapsto
    (-c_{2},s_{2},-c_{1},s_{1}).
\end{equation}
Writing
\begin{equation}
    u=c_{1}+c_{2},
    \quad
    v=c_{1}c_{2},
    \quad
    w=(c_{1}-c_{2})s_{1}s_{2},
\end{equation}

Among the secondary basis, the fixed elements are
$1$, $v$, $w$, and $vw$. Hence $[K:G]=4$ and
\begin{align}
    P_{\mathrm{torus}}^{G}
    &=
    P_{\mathrm{torus}}^{p4gm}
    =
    \mathbb{R}[c_{1}^{2}+c_{2}^{2},c_{1}^{2}c_{2}^{2}]
    (1,c_{1}c_{2})
    \bigl(1,(c_{1}-c_{2})s_{1}s_{2}\bigr) \\
    &=
    \mathbb{R}[c_{1}^{2}+c_{2}^{2},c_{1}^{2}c_{2}^{2}]
    \bigl(
        1,\ c_{1}c_{2},\
        (c_{1}-c_{2})s_{1}s_{2},\
        (c_{1}-c_{2})c_{1}c_{2}s_{1}s_{2}
    \bigr).
\end{align}

\subsection{Adjusted Coordinate Conventions}
\label{sec:adj_conventions}

For the centered rectangular groups $cm$ and $c2mm$, the primitive
rhombic basis used in \citet{kimRingInvariantReal2001} differs from the
conventional rectangular basis used in the ITA. Let
$(\mathbf e_{1},\mathbf e_{2})$ denote the conventional rectangular basis,
and let
$(\widetilde{\mathbf e}_{1},\widetilde{\mathbf e}_{2})$ denote the
corresponding primitive rhombic basis, defined by
\begin{equation}
    \widetilde{\mathbf e}_{1}
    =
    \tfrac12(\mathbf e_{2}-\mathbf e_{1}),
    \quad
    \widetilde{\mathbf e}_{2}
    =
    \tfrac12(\mathbf e_{2}+\mathbf e_{1}).
\end{equation}
Equivalently,
\begin{equation}
    \begin{bmatrix}
        \widetilde{\mathbf e}_{1} &
        \widetilde{\mathbf e}_{2}
    \end{bmatrix}
    =
    \begin{bmatrix}
        \mathbf e_{1} &
        \mathbf e_{2}
    \end{bmatrix}
    C,
    \quad
    C
    =
    \begin{bmatrix}
        -\tfrac12 & \tfrac12 \\
        \tfrac12 & \tfrac12
    \end{bmatrix}.
\end{equation}
If a point $\mathbf r$ has coordinates $(x,y)$ in the conventional basis
and $(\widetilde{x},\widetilde{y})$ in the primitive basis, then
\begin{equation}
    \mathbf r
    =
    \begin{bmatrix}
        \mathbf e_{1} &
        \mathbf e_{2}
    \end{bmatrix}
    \begin{bmatrix}
        x\\
        y
    \end{bmatrix}
    =
    \begin{bmatrix}
        \widetilde{\mathbf e}_{1} &
        \widetilde{\mathbf e}_{2}
    \end{bmatrix}
    \begin{bmatrix}
        \widetilde{x}\\
        \widetilde{y}
    \end{bmatrix}.
\end{equation}
Therefore, the coordinate transformation is
\begin{equation}
    \begin{bmatrix}
        \widetilde{x}\\
        \widetilde{y}
    \end{bmatrix}
    =
    C^{-1}
    \begin{bmatrix}
        x\\
        y
    \end{bmatrix}
    =
    \begin{bmatrix}
        -1 & 1\\
        1 & 1
    \end{bmatrix}
    \begin{bmatrix}
        x\\
        y
    \end{bmatrix}
    =
    \begin{bmatrix}
        y-x\\
        x+y
    \end{bmatrix}.
\end{equation}
Rather than explicitly substituting this coordinate transformation into
the primitive-cell formulas, we derive the corresponding ITA-coordinate
expressions directly by expanding the invariant modules inside a common
affine reflection supergroup.

\paragraph{$cm$.}
For bookkeeping, write $pm_{+}$ and $cm_{+}$ for the orientations whose
mirror or glide axes are parallel to the $y$-axis, and write $pm_{-}$ and
$cm_{-}$ for those whose axes are parallel to the $x$-axis.

Let $K=p2mm'$ be the affine reflection supergroup whose translation
lattice contains the half-translations
$(\tfrac12,0)$ and $(0,\tfrac12)$. Its invariant ring is
\begin{equation}
    A
    :=
    P_{\mathrm{torus}}^{K}
    =
    \mathbb{R}[c(2x),c(2y)]
    =
    \mathbb{R}[c_{1}^{2},c_{2}^{2}].
\end{equation}

Consider first $G=cm_{+}$. Its index-$2$ symmorphic subgroup is
$H=pm_{+}$. Since the mirror of $pm_{+}$ is parallel to the $y$-axis,
it acts by $(x,y)\mapsto(-x,y)$. Hence
\begin{equation}
    P_{\mathrm{torus}}^{pm_{+}}
    =
    \mathbb{R}[c_{1},c_{2}](1,s_{2})
    =
    A(1,c_{1})(1,c_{2})(1,s_{2}).
\end{equation}
The group $cm_{+}$ is obtained from $pm_{+}$ by adjoining the glide
reflection with axis $x=\tfrac14$,
\begin{equation}
    g(x,y)
    =
    \left(\tfrac12-x,y+\tfrac12\right),
\end{equation}
which acts on the trigonometric generators by
\begin{equation}
    (c_{1},s_{1},c_{2},s_{2})
    \mapsto
    (-c_{1},s_{1},-c_{2},-s_{2}).
\end{equation}
Thus $c_{1}$, $c_{2}$, and $s_{2}$ are sign-changing. The fixed
secondary basis therefore consists of the products containing an even
number of these three factors. Hence $[K:cm_{+}]=4$ and
\begin{equation}
    P_{\mathrm{torus}}^{cm_{+}}
    =
    A
    \bigl(
        1,\ c_{1}c_{2},\ c_{1}s_{2},\ c_{2}s_{2}
    \bigr).
\end{equation}
Equivalently,
\begin{equation}
    P_{\mathrm{torus}}^{cm_{+}}
    =
    \mathbb{R}[c_{1}^{2},c_{2}^{2}]
    \bigl(
        1,\ c_{1}c_{2},\ c_{1}s_{2},\ c_{2}s_{2}
    \bigr).
\end{equation}

Swapping $x$ and $y$ gives the other orientation,
\begin{equation}
    P_{\mathrm{torus}}^{cm_{-}}
    =
    \mathbb{R}[c_{1}^{2},c_{2}^{2}]
    \bigl(
        1,\ c_{1}c_{2},\ c_{2}s_{1},\ c_{1}s_{1}
    \bigr).
\end{equation}

\paragraph{$c2mm$.}
After fixing a common translation lattice and origin, the group $c2mm$
is generated by $cm_{+}$ and $cm_{-}$. Therefore,
\begin{equation}
    P_{\mathrm{torus}}^{c2mm}
    =
    P_{\mathrm{torus}}^{cm_{+}}
    \cap
    P_{\mathrm{torus}}^{cm_{-}}.
\end{equation}
The intersection is taken inside the common free $A$-module
\begin{equation}
    P_{\mathrm{torus}}
    =
    A(1,c_{1})(1,s_{1})(1,c_{2})(1,s_{2}).
\end{equation}
The displayed secondary bases for
$P_{\mathrm{torus}}^{cm_{+}}$ and
$P_{\mathrm{torus}}^{cm_{-}}$ are subsets of this common $A$-basis.
Their only common basis elements are $1$ and $c_{1}c_{2}$. It follows
that
\begin{equation}
    P_{\mathrm{torus}}^{c2mm}
    =
    A(1,c_{1}c_{2})
    =
    \mathbb{R}[c_{1}^{2},c_{2}^{2}]
    \bigl(1,c_{1}c_{2}\bigr).
\end{equation}
Consequently, $P_{\mathrm{torus}}^{c2mm}$ has rank $2$ over
$P_{\mathrm{torus}}^{K}$, and hence
\begin{equation}
    [K:c2mm]
    =
    [p2mm':c2mm]
    =
    2.
\end{equation}

\clearpage
\section{Symmetric Connectivity Criterion}
\label{sec:sym_conn_criterion}

In what follows, we fix the translation subgroup of the planar group $p1$ by the two primitive vectors
\begin{equation}
    \mathbf e_1=(1,0),\quad \mathbf e_2=(0,1), 
\end{equation}
and all occurrences of connected mean path-connected. Let
\begin{equation}
    \pi:\mathbb R^2\to \mathbb T^2:=\mathbb R^2/\mathbb Z^2 
\end{equation}
be the quotient (covering) map. For a path-connected subset $C\subset\mathbb T^2$, the full preimage $\pi^{-1}(C)\subset\mathbb R^2$ may have
several path-connected components. For any such component $\widetilde C$, define its stabilizer
\begin{equation}
    H(\widetilde C)=\{g\in\mathbb Z^2\mid \widetilde C+g=\widetilde C\}. 
\end{equation}

\begin{lemma}\label{lem:stabilizer_well_defined}
Let $C\subset\mathbb T^2$ be nonempty and path-connected. If $\widetilde C_1,\widetilde C_2$ are any two path-connected components
of $\pi^{-1}(C)$, then there exists $k\in\mathbb Z^2$ such that $\widetilde C_2=\widetilde C_1+k$, and moreover
\begin{equation}
    H(\widetilde C_1)=H(\widetilde C_2). 
\end{equation}
Hence $H(\widetilde C)$ depends only on $C$, and we may write it as $H(C)$.
\end{lemma}

\begin{proof}
We first show that every path-connected component
$\widetilde C$ of $\pi^{-1}(C)$ projects onto $C$.
Fix $x\in\widetilde C$. For any $\bar y\in C$, choose a path
in $C$ from $\pi(x)$ to $\bar y$. Its lift starting at $x$
lies in $\pi^{-1}(C)$ and, since it begins in
$\widetilde C$, remains in $\widetilde C$. Its endpoint
projects to $\bar y$. Therefore,
$$
\pi(\widetilde C)=C.
$$

Pick $x_1\in\widetilde C_1$ and set $\bar x:=\pi(x_1)\in C$. Since $\pi(\widetilde C_2)=C$, there exists
$x_2\in\widetilde C_2$ with $\pi(x_2)=\bar x$. Then $x_2-x_1\in\mathbb Z^2$; write $k:=x_2-x_1$ so that
$x_2=x_1+k$. The translation $T_k(x)=x+k$ satisfies $\pi\circ T_k=\pi$, hence $T_k(\pi^{-1}(C))=\pi^{-1}(C)$.
Therefore $T_k(\widetilde C_1)$ is path-connected, contained in $\pi^{-1}(C)$, and contains $x_2$. Since $T_k$ restricts to a homeomorphism of
$\pi^{-1}(C)$, it maps path-connected components to
path-connected components. Hence $T_k(\widetilde C_1)$ is a
path-connected component of $\pi^{-1}(C)$. Since it contains
$x_2$, it must equal $\widetilde C_2$. Therefore,
$$
\widetilde C_2
=
T_k(\widetilde C_1)
=
\widetilde C_1+k.
$$
\end{proof}

\begin{lemma}
\label{lem:two_p1_invariant_components_intersect}
Let $E,F\subset\mathbb R^2$ be nonempty, $\mathbb Z^2$-invariant subsets, i.e.
\begin{equation}
    E+(m,n)=E,\quad F+(m,n)=F,\quad \forall (m,n)\in\mathbb Z^2. 
\end{equation}
If both $E$ and $F$ are path-connected, then $E\cap F\neq\varnothing$.
\end{lemma}

\begin{proof}
Let $\pi:\mathbb R^2\to\mathbb T^2$ be the quotient map. We first show $\pi(E)\cap\pi(F)\neq\varnothing$.

Pick $p\in E$. Since $p+\mathbf e_1\in E$ and $E$ is path-connected, there exists a path
\begin{equation}
    \gamma_x:[0,1]\to E,\quad \gamma_x(0)=p,\ \gamma_x(1)=p+\mathbf e_1. 
\end{equation}
Set $\alpha:=\pi\circ\gamma_x$, which is a loop in $\mathbb T^2$ based at $\pi(p)$.
Likewise, pick $q\in F$. Since $q+\mathbf e_2\in F$ and $F$ is path-connected, there exists a path
\begin{equation}
    \gamma_y:[0,1]\to F,\quad \gamma_y(0)=q,\ \gamma_y(1)=q+\mathbf e_2, 
\end{equation}
and set $\beta:=\pi\circ\gamma_y$, a loop in $\mathbb T^2$ based at $\pi(q)$.

Lift $\alpha$ to $\widetilde\alpha$ with $\widetilde\alpha(0)=p$. By uniqueness of path lifting,
$\widetilde\alpha=\gamma_x$, hence $\widetilde\alpha(1)-\widetilde\alpha(0)=\mathbf e_1$. Similarly, the lift of
$\beta$ starting at $q$ satisfies $\widetilde\beta(1)-\widetilde\beta(0)=\mathbf e_2$. Thus, under the standard
identification $\pi_1(\mathbb T^2)\cong\mathbb Z^2$, the loops $\alpha,\beta$ represent the classes $(1,0)$ and $(0,1)$.

By intersection theory on surfaces (e.g. the mod-$2$ intersection number; see Sec. 2.4 of \citet{guilleminDifferentialTopology1974}), the mod-$2$ intersection number of two loops depends only on their homology classes in
$H_1(\mathbb T^2;\mathbb Z_2)$, and the two coordinate generators $(1,0)$ and $(0,1)$ have mod-$2$ intersection equal to $1$. Hence $\alpha$ and $\beta$ cannot be disjoint, so $\alpha([0,1])\cap\beta([0,1])\neq\varnothing$. Consequently,
\begin{equation}
    \pi(E)\cap\pi(F)\neq\varnothing. 
\end{equation}

Now take $\bar z\in\pi(E)\cap\pi(F)$. Choose $e\in E$ and $f\in F$ with $\pi(e)=\pi(f)=\bar z$.
Then $e-f\in\mathbb Z^2$. Let $t:=e-f\in\mathbb Z^2$, so $f+t=e$. Since $F$ is $\mathbb Z^2$-invariant, $f+t\in F$,
hence $e\in E\cap F$. Therefore $E\cap F\neq\varnothing$.
\end{proof}

\begin{theorem}
\label{thm:2x2_gamma_connectivity}
Let $S\subset\mathbb R^2$ be a nonempty,
$\mathbb Z^2$-invariant set. Fix
\begin{equation}
    A=[0,2)\times[0,2),
    \quad
    B=[0,1)\times[0,1),
    \quad
    \Gamma
    =
    \bigl(\{0\}\times[0,1)\bigr)
    \cup
    \bigl([0,1)\times\{0\}\bigr)
    \subset B.
\end{equation}
Assume that
$S\cap A$ has finitely many path-connected components and
that every path-connected component of $S\cap A$ intersects
$\Gamma$. Then $S$ is path-connected.
\end{theorem}

\begin{proof}

Let $C:=\pi(S)\subset\mathbb T^2$ and write its decomposition into path-connected components as
\begin{equation}
    C=\bigsqcup_{j\in J}C_j.
\end{equation}
Since $S$ is $\mathbb Z^2$-invariant, one has $S=\pi^{-1}(C)$. Indeed, the inclusion $S\subset\pi^{-1}(C)$ is immediate.
Conversely, if $x\in\pi^{-1}(C)$, then
$\pi(x)=\pi(y)$ for some $y\in S$. Hence
$x-y\in\mathbb Z^2$, and the $\mathbb Z^2$-invariance of
$S$ implies $x\in S$.

Fix $j\in J$. Let $\mathcal L_j$ denote the collection of
path-connected components of $\pi^{-1}(C_j)$ that intersect
$B$:
\begin{equation}
    \mathcal L_j
    :=
    \left\{
        \widetilde C
        \;\middle|\;
        \widetilde C
        \text{ is a path-connected component of }
        \pi^{-1}(C_j),
        \quad
        \widetilde C\cap B\neq\varnothing
    \right\}.
\end{equation}

The set $\mathcal L_j$ is nonempty. Indeed, choose any
path-connected component $\widetilde C$ of
$\pi^{-1}(C_j)$ and a point $x\in\widetilde C$. There exists
$k\in\mathbb Z^2$ such that $x-k\in B$. By
\cref{lem:stabilizer_well_defined},
$\widetilde C-k$ is another path-connected component of
$\pi^{-1}(C_j)$, and it intersects $B$.

We next show that $\mathcal L_j$ is finite. First observe
that every path-connected component of $\pi^{-1}(C_j)$ is
also a path-connected component of $S$. Indeed, any path in
$S$ starting in $\pi^{-1}(C_j)$ projects to a path in $C$
starting in $C_j$, and therefore in $C_j$. For each $\widetilde C\in\mathcal L_j$, choose $x_{\widetilde C}\in\widetilde C\cap B$ and let $D_{\widetilde C}$ be the path-connected component
of $S\cap A$ containing $x_{\widetilde C}$. Since
$\widetilde C$ is a path-connected component of $S$, one has $D_{\widetilde C}\subset\widetilde C$. Consequently, distinct elements of $\mathcal L_j$ determine
distinct path-connected components of $S\cap A$. Since
$S\cap A$ has finitely many path-connected components,
$\mathcal L_j$ is finite.

For $i\in\{1,2\}$, define
\begin{equation}
    \tau_i:\mathcal L_j\longrightarrow\mathcal L_j,
    \quad
    \tau_i(\widetilde C)
    =
    \widetilde C+\mathbf e_i.
\end{equation}
We first verify that $\tau_i$ is well-defined. Let
$\widetilde C\in\mathcal L_j$, and choose $x\in\widetilde C\cap B$. By \cref{lem:stabilizer_well_defined},
$\widetilde C+\mathbf e_i$ is a path-connected component of
$\pi^{-1}(C_j)$. Moreover,
\begin{equation}
    x+\mathbf e_i
    \in
    (\widetilde C+\mathbf e_i)
    \cap
    (B+\mathbf e_i)
    \subset
    S\cap A,
\end{equation}
because $B+\mathbf e_i\subset A$. Let $D_i$ be the path-connected component of $S\cap A$
containing $x+\mathbf e_i$. Since
$\widetilde C+\mathbf e_i$ is a path-connected component of
$S$, one has $D_i\subset\widetilde C+\mathbf e_i$. By the hypothesis of the theorem, $D_i\cap\Gamma\neq\varnothing$. Since $\Gamma\subset B$, it follows that $(\widetilde C+\mathbf e_i)\cap B\neq\varnothing$. Therefore, $\widetilde C+\mathbf e_i\in\mathcal L_j$, and $\tau_i$ is well-defined. Each $\tau_i$ is injective, because
\begin{equation}
    \widetilde C_1+\mathbf e_i
    =
    \widetilde C_2+\mathbf e_i
    \quad\Longrightarrow\quad
    \widetilde C_1=\widetilde C_2.
\end{equation}
Since $\mathcal L_j$ is finite, $\tau_i$ is a permutation.

Fix $\widetilde C\in\mathcal L_j$. Since $\tau_1$ and
$\tau_2$ are permutations, there exist integers
$r_1,r_2\geq1$ such that
\begin{equation}
    \widetilde C+r_1\mathbf e_1=\widetilde C,
    \quad
    \widetilde C+r_2\mathbf e_2=\widetilde C.
\end{equation}
Thus $\widetilde C$ is invariant under the full-rank
sublattice
\begin{equation}
    \Lambda
    :=
    r_1\mathbb Z\mathbf e_1
    \oplus
    r_2\mathbb Z\mathbf e_2.
\end{equation}

We now show that $\widetilde C$ is invariant under the
primitive translations. Consider $\widetilde C$ and $\widetilde C+\mathbf e_1$. Both sets are nonempty, path-connected, and
$\Lambda$-invariant. Define the invertible linear map
\begin{equation}
    \Phi:\mathbb R^2\longrightarrow\mathbb R^2,
    \quad
    \Phi(u,v)
    =
    r_1u\,\mathbf e_1+r_2v\,\mathbf e_2.
\end{equation}
Then $\Phi^{-1}(\widetilde C)$ and $\Phi^{-1}(\widetilde C+\mathbf e_1)$ are nonempty, path-connected, and $\mathbb Z^2$-invariant.
\cref{lem:two_p1_invariant_components_intersect}
therefore gives
\begin{equation}
    \Phi^{-1}(\widetilde C)
    \cap
    \Phi^{-1}(\widetilde C+\mathbf e_1)
    \neq\varnothing.
\end{equation}
Applying $\Phi$, we obtain $\widetilde C \cap (\widetilde C+\mathbf e_1)\neq\varnothing$. By Lemma~\ref{lem:stabilizer_well_defined}, $\widetilde C$ and $\widetilde C+\mathbf e_1$ are path-connected components of the same set
$\pi^{-1}(C_j)$. Since two path-connected components are
either equal or disjoint, it follows that $\widetilde C+\mathbf e_1=\widetilde C$. Hence $\mathbf e_1\in H(\widetilde C)$. The same argument, with $\mathbf e_2$ in place of
$\mathbf e_1$, gives $\mathbf e_2\in H(\widetilde C)$. Therefore, $H(\widetilde C)=\mathbb Z^2$.

Let $\widetilde C'$ be any other path-connected component of
$\pi^{-1}(C_j)$. By
Lemma~\ref{lem:stabilizer_well_defined}, there exists
$k\in\mathbb Z^2$ such that $\widetilde C'=\widetilde C+k$. Since $H(\widetilde C)=\mathbb Z^2$, one has $\widetilde C+k=\widetilde C$. Thus $\widetilde C'=\widetilde C$, proving that
$\pi^{-1}(C_j)$ has exactly one path-connected component.
Therefore $\pi^{-1}(C_j)$ is path-connected and $\mathbb Z^2$-invariant.

It remains to show that $C$ has only one path-connected
component. Suppose, for contradiction, that there exist
distinct indices $j_1,j_2\in J$. Then $\pi^{-1}(C_{j_1})$ and $\pi^{-1}(C_{j_2})$ are nonempty, path-connected, and $\mathbb Z^2$-invariant.
By \cref{lem:two_p1_invariant_components_intersect}, $\pi^{-1}(C_{j_1})\cap \pi^{-1}(C_{j_2})\neq\varnothing$. However, $\pi^{-1}(C_{j_1}\cap C_{j_2})=\varnothing$ because $C_{j_1}$ and $C_{j_2}$ are distinct
path-connected components of $C$. This is a contradiction. Therefore, $C$ has exactly one path-connected component.
Since its full preimage was shown above to be path-connected, and since $S=\pi^{-1}(C)$, we conclude that $S$ is path-connected.

\end{proof}

\clearpage
\section{Experimental Details}
\label{sec:exp_detail}

For continuous 2D representations, we adopt a hash-grid encoding with bilinear interpolation, inspired by InstantNGP~\citep{mullerInstantNeuralGraphics2022}. Unlike original InstantNGP, which feeds multi-level interpolated embeddings into an MLP, we remove the MLP for more stable latent-space SDS optimization and directly average embeddings across levels. We apply random rotations and translations to input coordinates to reduce grid artifacts. We set the number of levels and hash-table capacity to $L=16$ and $T=2^{19}$, respectively. Since the latent space typically contains denser information, we adopt a relatively narrow and high-resolution range with $N_{\min}=128$ and $N_{\max}=256$ with latent resolution of $128 \times 128$ for \cref{sec:pattern_results}, \cref{sec:paper_cutting_results} and \cref{sec:topology_results}. While for the pixel-space baselines in \cref{sec:topology_results}, we set $N_{\min}=8$ and $N_{\max}=128$ with image resolution of $256 \times 256$. For the pixel-space optimization in \cref{sec:metamaterial_results} we set $N_{\min}=2$ and $N_{\max}=32$ with image resolution of $64 \times 64$.
 
\subsection{Pattern Design}
\label{sec:pattern_design_detail}

In \cref{sec:pattern_results}, the lattice parameters for all 17 planar groups are specified by the symmetry-operation markers shown in \cref{fig:sd_label} and are used for both visualization and comparison in the text-conditioned generation experiments. For comparisons with other symmetrization methods, we set the lattice-vector lengths $a$ and $b$ to half the image side length, and $\gamma=2\pi/3$.

\begin{figure*}[!htbp]
\centering
\includegraphics[width=0.8\linewidth]{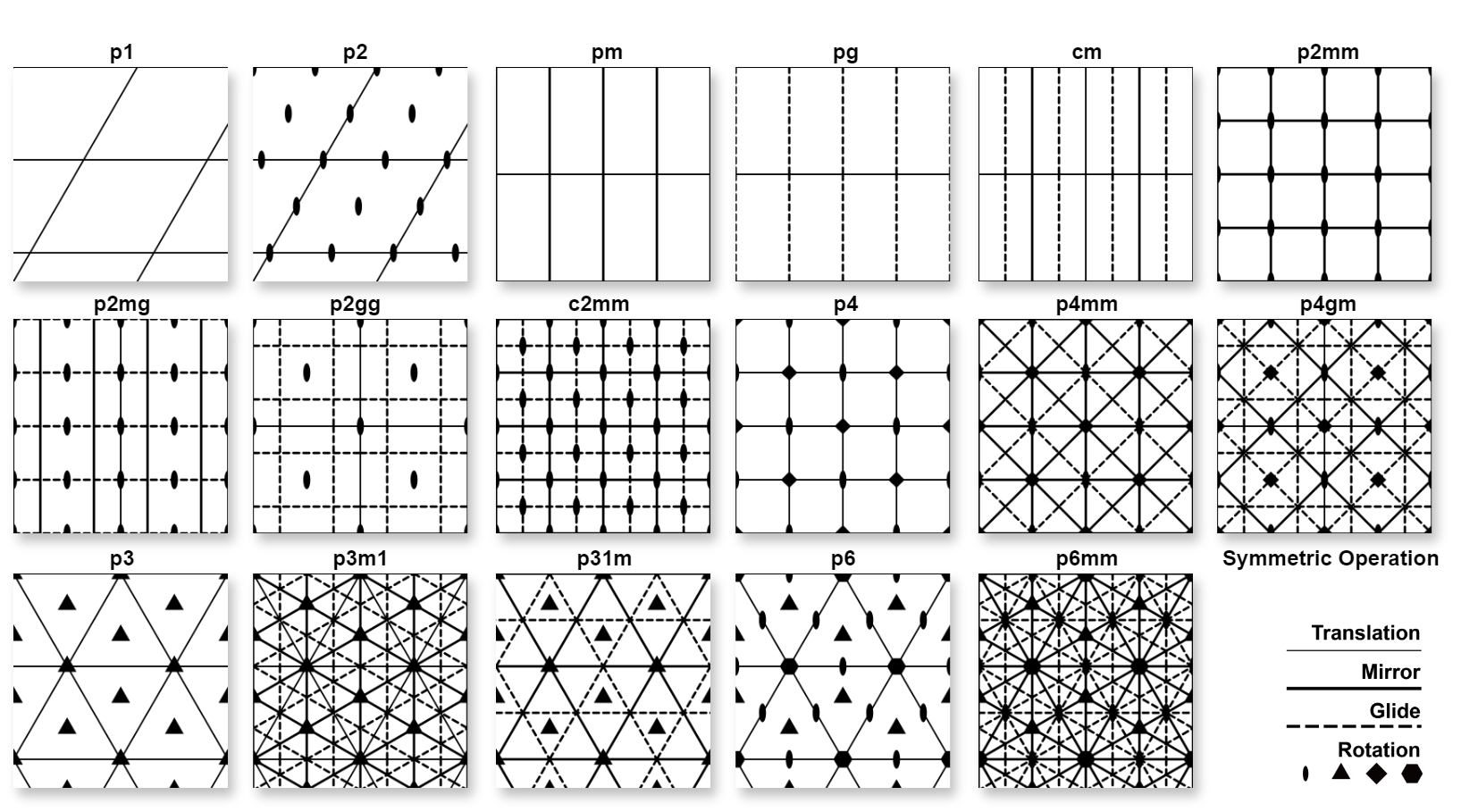}
\caption{Reference images illustrating the symmetry operations of the 17 planar groups with markers.}
\label{fig:sd_label}
\end{figure*}

\subsubsection{Visualization}
\label{sec:vis_detail}

For the generative tasks, during SDS optimization, we use the positive prompt: \textit{stained-glass mosaic fragments, simple polygon shards with thick lead outlines}, and the negative prompt: \textit{lowres, bad anatomy, error, extra digit, fewer digits, worst quality, watermark}. We perform optimization in the latent space. The optimization process is conducted for a total of $200$ steps using the AdamW optimizer with a learning rate of $0.01$. We employ a dynamic guidance strategy, and the CFG scale is linearly annealed from an initial value of $100$ to a final value of $7.5$. We also employ a linear annealing strategy for the timestep, decreasing from $0.4T$ to $0$ (where $T$ represents the total diffusion timesteps). Following the SDS convergence, we apply a refinement stage to enhance image quality and correct potential artifacts. We perturb the optimized latent code by injecting noise corresponding to $t=0.4T$ and subsequently denoise it back to $t=0$ using the standard diffusion sampling process with $100$ inference steps. The resulting images are shown in \cref{fig:sd_1_to_17}.

\subsubsection{Comparison with Text-conditioned Generation}
\label{sec:mllm_setting_detail}

For our method, we adopt the SDS-based optimization and refinement pipeline described in \cref{sec:vis_detail}, with the following modifications. We run SDS optimization for $800$ steps with a learning rate of $0.03$, anneal the diffusion timestep from $0.5T$ to $0$, and inject noise at $t=0.5T$ before the final refinement denoising. The negative prompt remains unchanged, while the positive prompts are specified later. The final generation results are presented in \cref{fig:sd_17_to_17}.

\begin{figure*}[!htbp]
    \centering
    \includegraphics[width=0.85\linewidth]{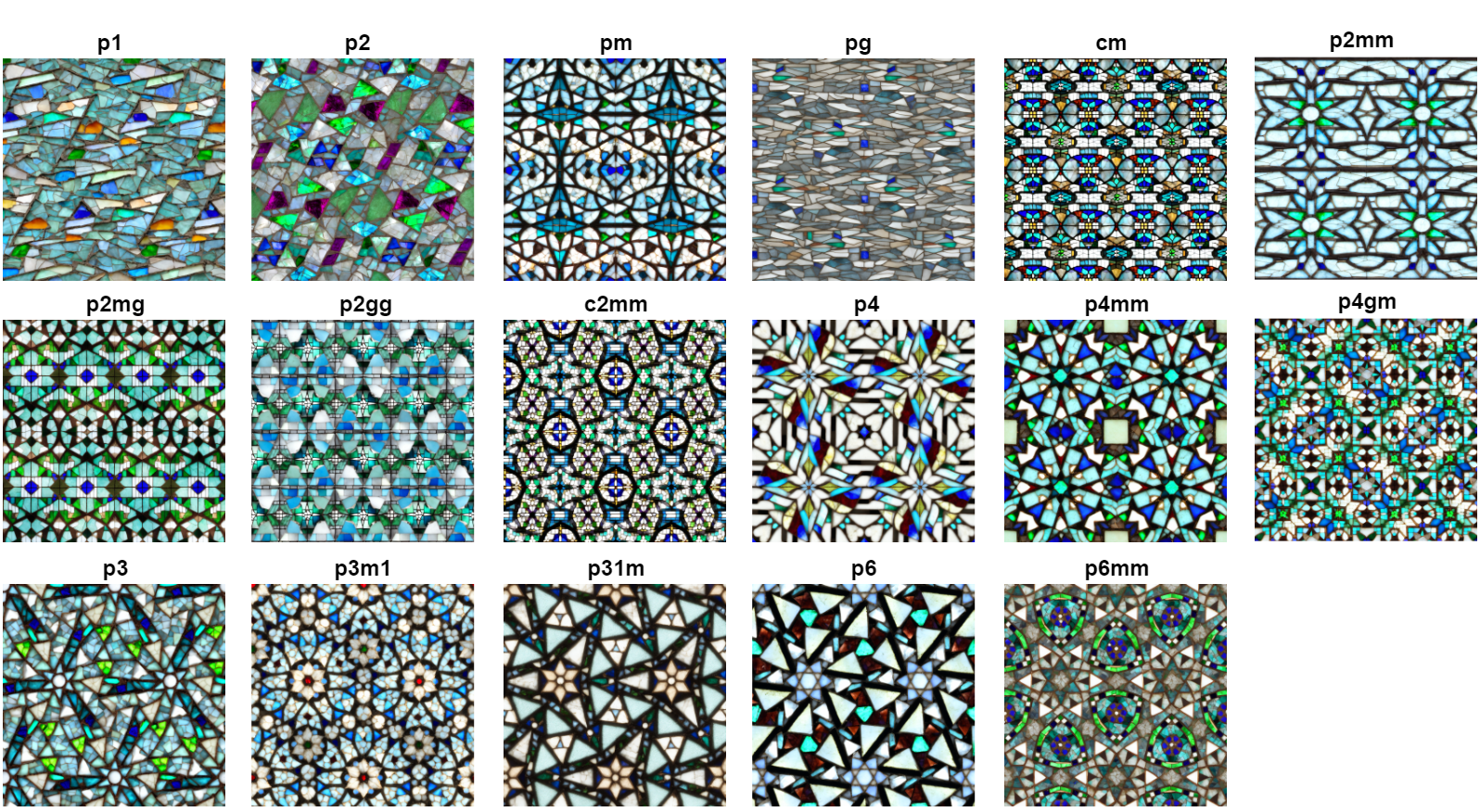}
    \caption{Generated patterns for visualization across the 17 plane symmetry groups.}
    \label{fig:sd_1_to_17}
\end{figure*}

\begin{figure*}[!htbp]
    \centering
    \includegraphics[width=0.85\linewidth]{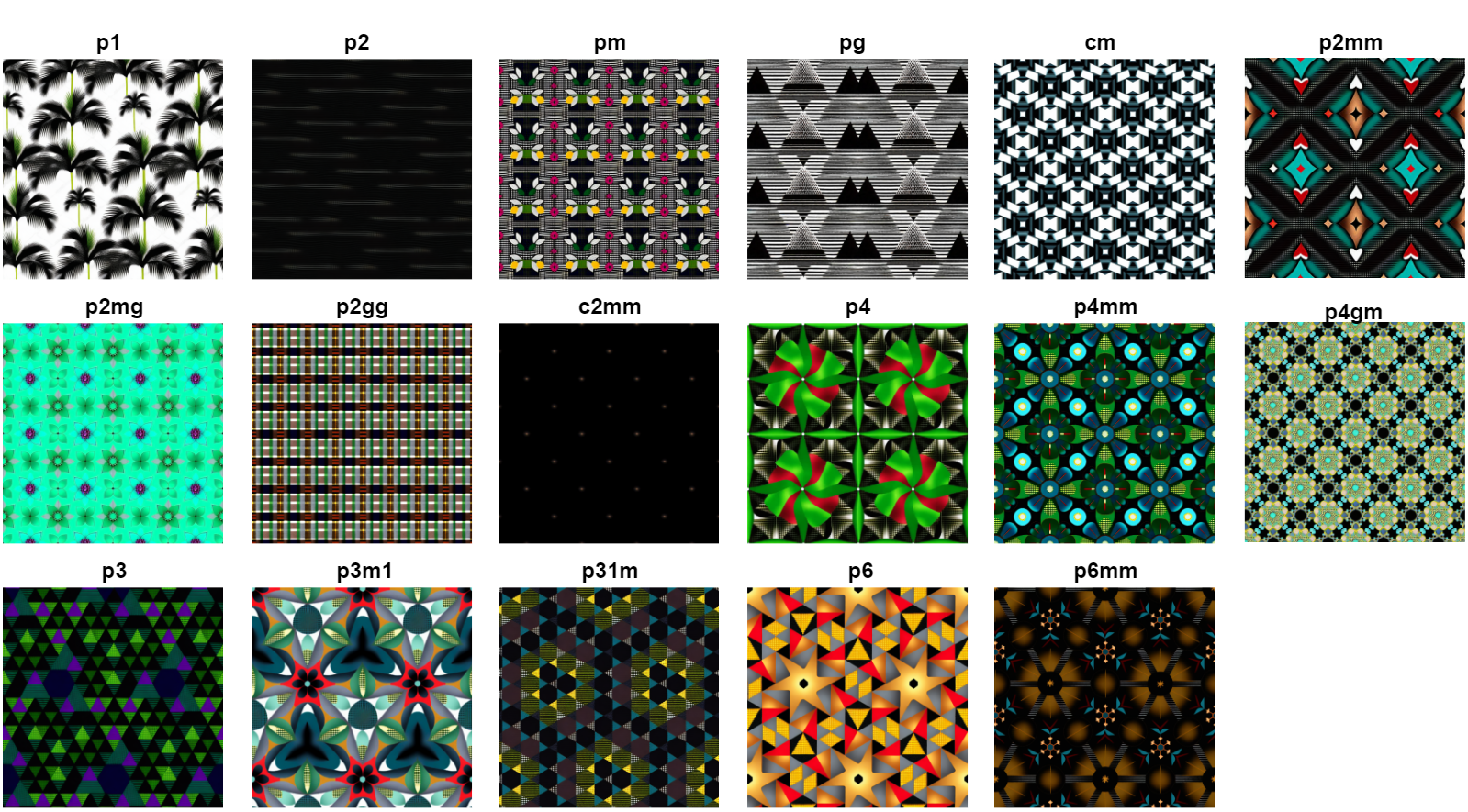}
    \caption{Generated patterns for comparison across the 17 plane symmetry groups.}
    \label{fig:sd_17_to_17}
\end{figure*}

\begin{figure*}[!htbp]
    \centering
    \includegraphics[width=0.9\linewidth]{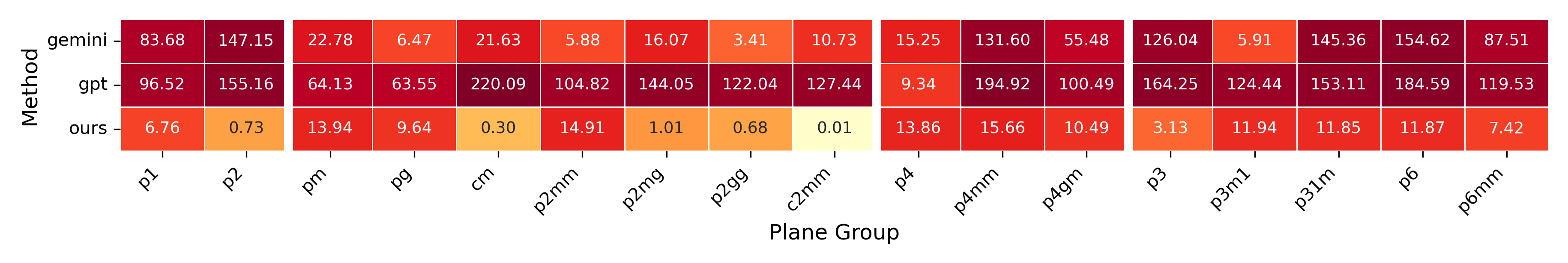}
    \caption[MSE Comparison across 17 Plane Groups for Different Methods.]{\textbf{MSE comparison across 17 plane groups for different methods.} The plane groups are grouped by lattice type. Entries in each cell report $\mathrm{MSE}\times 10^3$, while cell colors are determined by $\log_{10}(\mathrm{MSE})$. Our method achieves lower errors across most groups, demonstrating superior symmetry preservation.}
    \label{fig:mse_detail}
\end{figure*}

\textbf{Prompts Collection.}
As mentioned in Sec.~\ref{sec:experiment}, we curated a set of 17 text prompts, each corresponding to one of the 17 symmetry groups. These prompts serve as the basis for our comparative evaluation. The original concepts were derived from the examples listed in Tab.~9 of \citet{shubnikovSymmetryScienceArt1974}. We utilized Gemini to label these examples and simplify the descriptions into concise text prompts suitable for text-to-image generation. The complete mapping between symmetry groups and their corresponding prompts is detailed below

\begin{tcolorbox}[breakable, colback=TiffanyBlue!5!white, colframe=TiffanyBlue!75!black, width=\textwidth, title={Symmetry Groups with Corresponding Prompts}]
\begin{description}
\setlength{\itemsep}{4pt}
\item[p1:] \textit{palm fronds, seamless pattern, repeating, high contrast}
\item[p2:] \textit{Horizontal striped pattern, vertical lines and running waves, seamless, repeating, high contrast}
\item[pm:] \textit{floral line pattern, seamless, repeating, diagonal}
\item[pg:] \textit{geometric triangle pattern, seamless, repeating, high contrast}
\item[cm:] \textit{square forming a continuous meander maze, right-angled lines, seamless, repeating, geometric}
\item[p2mm:] \textit{heart-shaped frames with arrowheads, seamless, repeating, high contrast}
\item[p2mg:] \textit{lotus flowers, seamless, repeating, geometric}
\item[p2gg:] \textit{rectangular blocks, checkerboard, seamless, repeating, geometric}
\item[c2mm:] \textit{ornamental pattern, diamond grid, seamless, repeating}
\item[p4:] \textit{pinwheel pattern, seamless, repeating, tiled}
\item[p4mm:] \textit{radial circle pattern, seamless, repeating, geometric}
\item[p4gm:] \textit{arabesque, seamless, repeating, symmetric}
\item[p3:] \textit{triangular geometric pattern, seamless, repeating}
\item[p3m1:] \textit{petal tiling pattern, seamless, repeating, curved}
\item[p31m:] \textit{honeycomb geometric pattern, seamless, repeating}
\item[p6:] \textit{star motif pattern, triangle tiling, seamless, repeating}
\item[p6mm:] \textit{snowflake pattern, seamless, repeating, geometric}
\end{description}
\end{tcolorbox}

\textbf{Prompt Templates of MLLMs.}
The prompt templates used for baseline models, including GPT-5.2, Gemini 3 Pro, and SD 2.1, are defined as follows. For direct generation, the models are instructed to generate images directly from the text description without explicit symmetry constraints, using the template: \textit{Generate a square image based on the prompt: [pos\_prompt].} For conditional generation, we provide an auxiliary visual reference to guide MLLMs toward a specific plane symmetry group, using the template: \textit{Based on the input prompt and the [Group Name] plane symmetry group in the reference image, draw a square picture. DO NOT draw markers or lines. Prompt: [pos\_prompt].}

\textbf{Details of Post-Symmetrization.}
In post-symmetrization, we strictly enforce symmetry on images. We project the generated non-perfect images into our symmetric parameterization space. Let $I_{\text{ref}}$ be the input image generated by a baseline model. We initialize our symmetric generator, denoted as a parameterized lattice representation $\mathcal{G}_\phi$, where $\phi$ represents the learnable parameters of the symmetric feature field. The lattice configuration is scaled (typically by a factor of 8) to accommodate high-resolution optimization ($1024 \times 1024$). We optimize the parameters $\phi$ such that the generated symmetric image $I_{\text{sym}} = \mathcal{G}_\phi$ approximates $I_{\text{ref}}$ by minimizing the MSE loss
\begin{equation*}
\mathcal{L}_{\text{MSE}} = \| \mathcal{G}_{\phi} - I_{\text{ref}} \|^2_2.
\end{equation*}

The optimization is conducted using the AdamW optimizer with a 0.1 learning rate and 500 steps. The quantitative MSE results are reported in \cref{fig:mse_detail}.

\subsubsection{Comparison with Other Symmetrization}

For both method, we adopt the SDS-based optimization and refinement pipeline described in \cref{sec:vis_detail}, with the following modifications. We conduct an ablation over the output resolution. In this experiment, we do not use a negative prompt and evaluate on the $p1$ group. The positive prompts are kept the same as the 17 prompts used in \cref{sec:mllm_setting_detail}.

We implement the projection operator by constructing a finite Fourier basis associated with the target lattice and applying the orthogonal
projector obtained from its QR factorization. Given the lattice parameters, we first compute the corresponding reciprocal lattice and enumerate the integer reciprocal-lattice points $(h,k)$ within the Nyquist region. These frequencies define a band-limited periodic subspace. For each sampled pixel, we convert its Cartesian coordinate to the natural lattice coordinate $(u,v)$ and evaluate the Fourier basis functions $1$, $\cos(2\pi(hu+kv))$ and $\sin(2\pi(hu+kv))$. To avoid redundant basis functions, we keep only one representative from each pair of opposite reciprocal frequencies. Specifically, we retain the frequencies satisfying $h>0$ or $h=0,\ k>0$.

Stacking the basis values over all sample points gives a basis matrix $\mathbf{\Phi}$. We then orthonormalize $\mathbf{\Phi}$ by QR decomposition, obtaining $\mathbf{\Phi} = \mathbf{QR}$. The projection of a flattened image $\mathbf{x}$ onto this periodic subspace is computed as $\mathbf{QQ^T x}$. The same projection is applied independently to each image channel. In practice, we compute the projection over the full sampling grid rather than only over the parallelogram cell, which reduces boundary discontinuities and alleviates ringing artifacts.

\subsection{Paper-Cutting Design}
\label{sec:paper_cutting_detail}

\textbf{Details of fine-tuning.}
Our dataset consists of 140 Chinese paper-cutting images, each annotated with a regional style label. Based on these labels, we construct the text prompt for each training sample using the following template: \textit{"traditional chinese papercut art, [regional style], high quality, detailed, artistic, traditional craftsmanship, paper cutting, chinese folk art, intricate patterns, cultural heritage"}. The prompts are used as the text conditioning input during training.

We fine-tuned SDXL using a LoRA-based adaptation implemented with Diffusers. All training images were resized to 1024 $\times$ 1024, randomly cropped, randomly horizontally flipped, and normalized to the range $[-1,1]$. During fine-tuning, the VAE, both SDXL text encoders, and the original U-Net weights were frozen, and only the inserted LoRA parameters were optimized. LoRA adapters were added to the U-Net attention projection layers, with rank $r=4$ by default. The model was trained with batch size 1 for 1000 epochs using AdamW, with a learning rate of $1\times10^{-4}$, a constant learning-rate scheduler, and gradient clipping with a maximum norm of 1.0.

\textbf{Details of paper-cutting design.}
We select four representative symmetry groups: $p2mm$, $p4mm$, $p3m1$, and $p6mm$. To demonstrate generalization, we apply a shared prompt, \textit{red Chinese paper cutting, flowers} across all four groups (shown in the top row of \cref{fig:exp_papercut}a). In contrast, the bottom row displays results generated using prompts tailored specifically to each group. The detailed prompts are listed below:

\begin{tcolorbox}[breakable, colback=TiffanyBlue!5!white, colframe=TiffanyBlue!75!black, width=\textwidth, title={Prompts Used in Paper-Cutting Design}]
\begin{description}
\setlength{\itemsep}{4pt}
\item[p2mm:] \textit{x Chinese paper cutting, lantern}
\item[p4mm:] \textit{red Chinese paper cutting, copper}
\item[p3m1:] \textit{red Chinese paper cutting, star}
\item[p6mm:] \textit{red Chinese paper cutting, snowflakes}
\end{description}
\end{tcolorbox}

During optimization, we choose the lattice parameters to reduce redundancy in the optimized pattern. We additionally place the highest-order rotation center at the image center, which generally produces more visually balanced and aesthetically pleasing paper-cutting patterns. Let \(n_{x}\) and \(n_{y}\) denote the image resolution along the two axes. We set lattice parameters
\begin{equation}
    a=b=\frac{3n_{x}}{4},
    \qquad
    \phi=\frac{\pi}{6},
    \qquad
    (x_0,y_0)
    =
    \left(
    \frac{n_{x}}{2},
    \frac{n_{y}}{2}
    \right),
\end{equation}
with
\begin{equation}
    \gamma=
    \begin{cases}
        \pi/2,  & \text{for } p2mm \text{ and } p4mm,\\
        2\pi/3, & \text{for } p3m1 \text{ and } p6mm.
    \end{cases}
\end{equation}

For the connectivity constraint, we utilize the Virtual Temperature Method (VTM) with $\Gamma = \{0\} \times [0, 1) \cup [0, 1) \times \{0\}$ in the coordination determined by fundamental translation $\mathbf{a}$ and $\mathbf{b}$. We solve the heat-conduction equation on a $2\times2$ supercell. The mesh edge lengths are $1/2\times 1/2$, with the included angle $\gamma$ between $\mathbf a$ and $\mathbf b$, and we use a $128\times128$ mesh. The heat source ranges from $10^{-8}$ to $10^{-4}$, and the thermal conductivity ranges from $10^{-4}$ to $1$. The SIMP penalty is set to $5$. To approximate the maximum temperature, we use a differentiable $p$-norm aggregation with exponent $20$. We adopt four-node bilinear quadrilateral (Q1) shape functions with $2\times2$ Gauss quadrature. Before segmentation, we apply density filtering with radius $2$ and step size $1$.

In the cases, we keep the coefficient of the SDS loss fixed to $1$. For generation, we fix the learning rate to $10^{-2}$, the target volume fraction $\rho_{0}$ to $0.35$, and the connectivity penalty $\lambda_{\mathrm{conn}}$ to $10^{2}$. To identify the optimal configuration for paper-cutting results, we perform a grid search over the volume penalty $\lambda_{\mathrm{vol}}\in\{10^{3},3\times10^{3},5\times10^{3},10^{4}\}$, the end binary ratio in $\{0.0,0.5\}$, and the number of optimization steps in $\{200,400,800\}$. The weight of the rendered $z_{\theta}^{\mathrm{bin}}$ is increased linearly from $0$ to the specified end binary ratio during optimization. We employ a linearly annealed CFG scale, which starts from $100$ and decays to $7.5$ over the course of optimization. We also apply a linear timestep annealing strategy, decreasing the timestep from $T$ to $0$. Accordingly, $c_{\mathrm{solid}}$ and $c_{\mathrm{void}}$ are the spatially averaged latent vectors obtained by encoding pure red and pure white images, respectively, with the SDXL VAE.

\subsection{Topology Design}
As discussed in \cref{sec:topology_results}, we utilize the set of 12 test prompts in~\citet{zhongTopologyOptimizationTextguided2023} with lattice length $a$ and $b$ set to half the image side length. Details are as follows:

\begin{tcolorbox}[breakable, colback=TiffanyBlue!5!white, colframe=TiffanyBlue!75!black, width=\textwidth, title={Prompts Used in Topology Design}]
\begin{description}
\setlength{\itemsep}{4pt}
\item[] \textit{golden, Baroque style}
\item[] \textit{rainbow-color, spider web style}
\item[] \textit{red, koi, Chinese paper cutting style}
\item[] \textit{Autumn branches}
\item[] \textit{wood appliques, simple}
\item[] \textit{kaleidoscope art}
\item[] \textit{modern, dream, wavy texture}
\item[] \textit{rosewood texture}
\item[] \textit{floral ornament}
\item[] \textit{Persian carpet style}
\item[] \textit{Art Deco}
\item[] \textit{Art Nouveau}
\end{description}
\end{tcolorbox}

To mitigate potential artifacts during SDS optimization, we append the suffixes \textit{, tessellation, pure white background} to the prompts. Additionally, a negative prompt is employed: \textit{black, shadow, lowres, bad anatomy, error, extra digit, fewer digits, worst quality, watermark, 3d, shadow, blur, artifact, deformed, distorted, noisy}. As our SDS optimization operates within the latent space, it necessitates a latent representation of the background color. To achieve this, we encode a pure white image using the VAE encoder and utilize the resulting latent code as the white background in the latent space.

Both the baseline and our method share a unified physical simulation environment to ensure a fair comparison. VTM setting follows \cref{sec:paper_cutting_detail}. For the mechanical constraint, we employ homogenized finite element analysis (FEA) on oblique lattice elements. The mesh edge lengths are $1\times1$, with the included angle $\gamma$ between $\mathbf a$ and $\mathbf b$; we use a $64\times64$ mesh during training and a $128\times128$ mesh during testing. We adopt the plane-stress constitutive matrix, with the solid Young's modulus normalized to $E_0=1$, the ersatz-material modulus set to $E_{\min}=10^{-6}$, Poisson's ratio $\nu=0.3$, and the SIMP penalty factor set to $p=10$. We again use Q1 quadrilateral elements with $2\times2$ Gauss quadrature. Before segmentation, we apply density filtering with radius $3$ and step size $1$. After segmentation, we apply a Heaviside projection before FEA:
$$
\bar{\rho}
=
\frac{\tanh(\beta \eta)+\tanh\!\big(\beta(\rho-\eta)\big)}
{\tanh(\beta \eta)+\tanh\!\big(\beta(1-\eta)\big)} ,
$$
where $\beta$ is linearly annealed from $1$ to $8$, and the threshold is $\eta=0.3$. The target volume fraction for all topology design experiments is fixed at $\rho_{0} = 0.45$. In all cases, we keep the coefficient of the mechanical loss fixed to $1$. The loss weights and optimization schedules differ between the baseline and our method as follows.

For baseline, we optimize the topology for $401$ steps. The loss weights are configured as: volume penalty $\lambda_{\text{vol}} = 3 \times 10^4$, connectivity penalty $\lambda_{\text{conn}} = 10^2$, and the semantic CLIP loss weight $\lambda_{\text{clip}} = 5 \times 10^3$. For our method, we have optimization process of $801$ steps to ensure convergence of the generative objective. The physical constraint weights remain consistent with the baseline ($\lambda_{\text{vol}} = 3 \times 10^4$, $\lambda_{\text{conn}} = 10^2$) to enforce comparable structural validity. The SDS loss weight is set to $\lambda_{\text{sds}} = 0.3$. We employ a linearly annealing CFG scale, starting at $50$ and decaying to $7.5$ over the optimization. We also employ a linear annealing strategy for the timestep, decreasing from $T$ to $0$. Following the SDS convergence, we apply a refinement stage identical to that described in \cref{sec:pattern_design_detail}. 

For topology optimization, $c_{\mathrm{void}}$ is chosen as the averaged latent vector across latent pixels obtained by feeding a pure white image into the SD 2.1 VAE encoder. Notably, this SD 2.1-derived void representation yields substantially better optimization performance than its SDXL counterpart. Nevertheless, an SDXL model fine-tuned on the paper-cutting dataset described in \cref{sec:paper_cutting_detail} may also facilitate the generation of planar patterns for topology design.

\subsection{Metamaterial Design}

\textbf{Data Generation.}
We generate the $p1$ metamaterial training set using homogenization-based topology optimization. Each unit cell is discretized on a $64\times 64$ square finite-element grid, where each element has a density variable. The material stiffness is interpolated by a SIMP-type model with penalization $p=5$ and $E_{\min}=10^{-6}$. We use periodic boundary conditions for numerical homogenization and optimize the density field to maximize the homogenized bulk modulus under a volume fraction constraint $0.5$. 

Each sample is initialized from a uniform density field with a softened circular region at the center follows \citet{xia2015design}, together with a small periodic Fourier perturbation to introduce diversity. We use Fourier modes up to $5$ and perturb the initial density with amplitude $0.03$. During optimization, sensitivities are smoothed by a periodic sensitivity filter with radius $3$, and the density variables are updated by the optimality criteria method with move limit $0.1$. Each design is optimized for at most $200$ iterations and terminated early when the maximum density change is below $0.01$. The final continuous density field is binarized by thresholding at $0.5$, resulting in a $64\times64$ binary unit-cell mask. No symmetry other than the basic translational periodicity is imposed during data generation.

\textbf{Diffusion model training.}
We train an unconditional diffusion model on the generated $p1$ unit-cell masks. All samples are represented as single-channel $64\times64$ images and normalized to $[-1,1]$. The denoising network is a convolutional U-Net with sinusoidal timestep embeddings, residual blocks, group normalization, SiLU activations, and encoder-decoder skip connections. The base channel width is set to $64$, with channel multipliers $(1,2,2,2)$. We use a standard noise-prediction objective with $T=1000$ diffusion steps and a linear noise schedule from $\beta_1=10^{-4}$ to $\beta_T=0.02$. At each training iteration, we uniformly sample a timestep $t$, perturb the clean image $\mathbf{x}_0$ as
\begin{equation}
    \mathbf{x}_t=\sqrt{\bar{\alpha}_t}\mathbf{x}_0+\sqrt{1-\bar{\alpha}_t}\epsilon, \epsilon\sim\mathcal{N}(0,I).
\end{equation}

Here
\begin{equation}
    \bar{\alpha}_t=\prod_{s=1}^{t}(1-\beta_s),
\end{equation}

is the cumulative signal-preserving coefficient. We train the U-Net to predict the added noise $\epsilon$ using an MSE loss. The model is optimized with Adam using a batch size of $128$ and a learning rate of $10^{-4}$ for $100$ epochs. No symmetry labels or symmetry-specific data are used during training. For visualization, we periodically generate samples using DDIM sampling with $100$ denoising steps.

\textbf{SDS optimization.}
For the generative tasks driven by SDS, we perform optimization in the parametric symmetric representation space. For each prescribed planar group, the representation is instantiated with one density channel and lattice parameters $a=b=64$ and $\gamma=\pi/2$, consistent with the square unit-cell domain. Each sample is optimized for $300$ steps using the AdamW optimizer with a learning rate of $1\times10^{-1}$, $\beta=(0.9,0.99)$, and $\epsilon=10^{-15}$. The diffusion guidance uses $T=1000$ timesteps, and the SDS timestep is sampled from the range $[0.02T,0.98T]$. During optimization, we pass a normalized progress ratio $i/(N-1)$ to the SDS objective, where $i$ denotes the current optimization step and $N=300$ is the total number of steps. The output of the symmetric representation is transformed by a $\tanh$ activation before being fed into the SDS loss. After optimization, the resulting continuous density field is clipped to $[-1,1]$, rescaled to $[0,1]$, and binarized using a threshold of $0.5$ to obtain the final unit-cell mask.

\begin{figure}[H]
    \centering
    \includegraphics[width=0.7\linewidth]{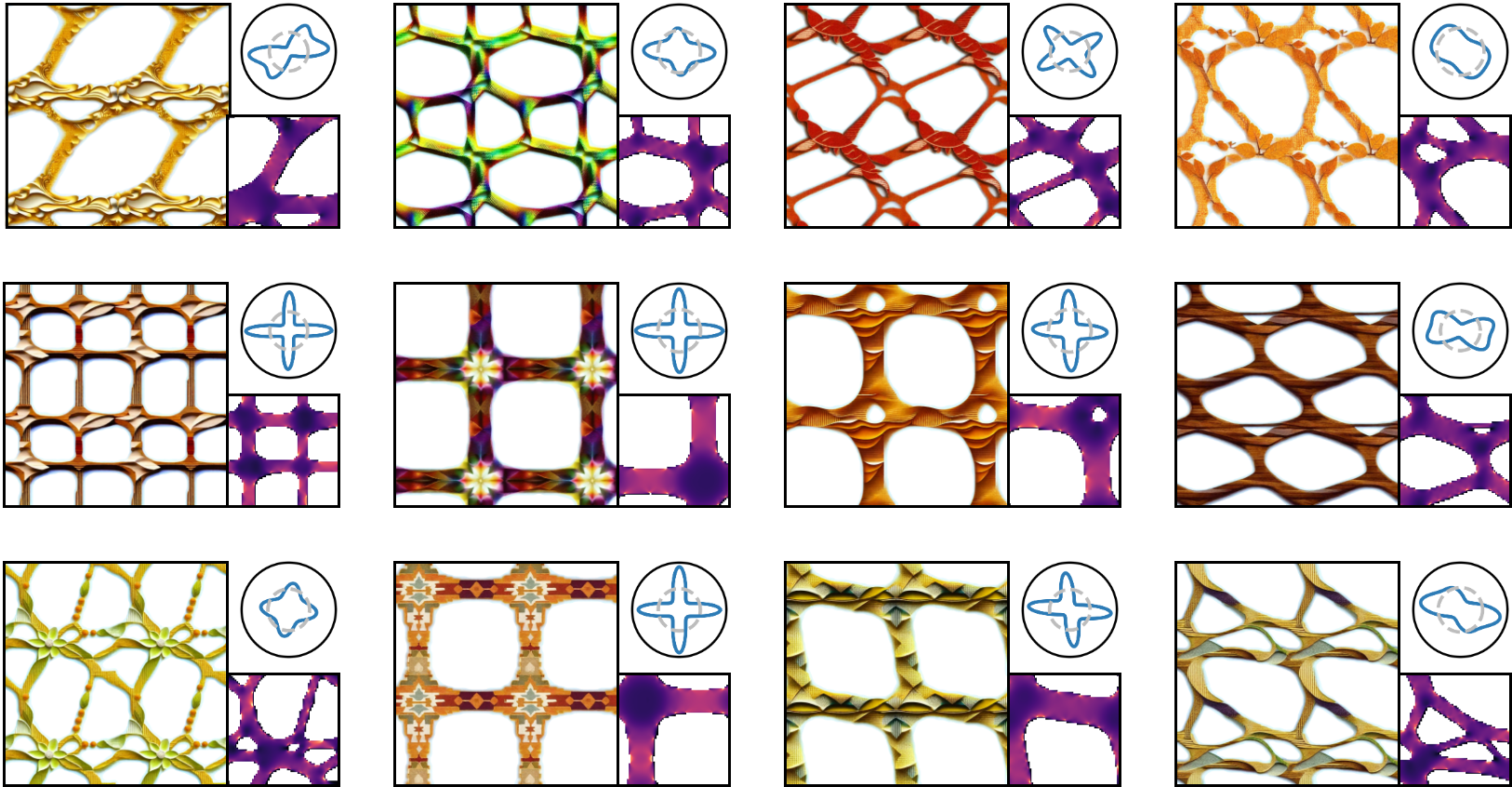}
    \caption[Visualization of Topology Design under $p1$ Symmetry]{\textbf{Visualization of topology design under $p1$ symmetry.} Left: topology-optimized designs. Top right: directional Young’s modulus, where blue curves indicate values across directions and the gray circle denotes the mean. Bottom right: energy distribution of the bulk modulus in unit cell.}
    \label{fig:sdto_p1}
    \vspace{0.5cm}
    \includegraphics[width=0.7\linewidth]{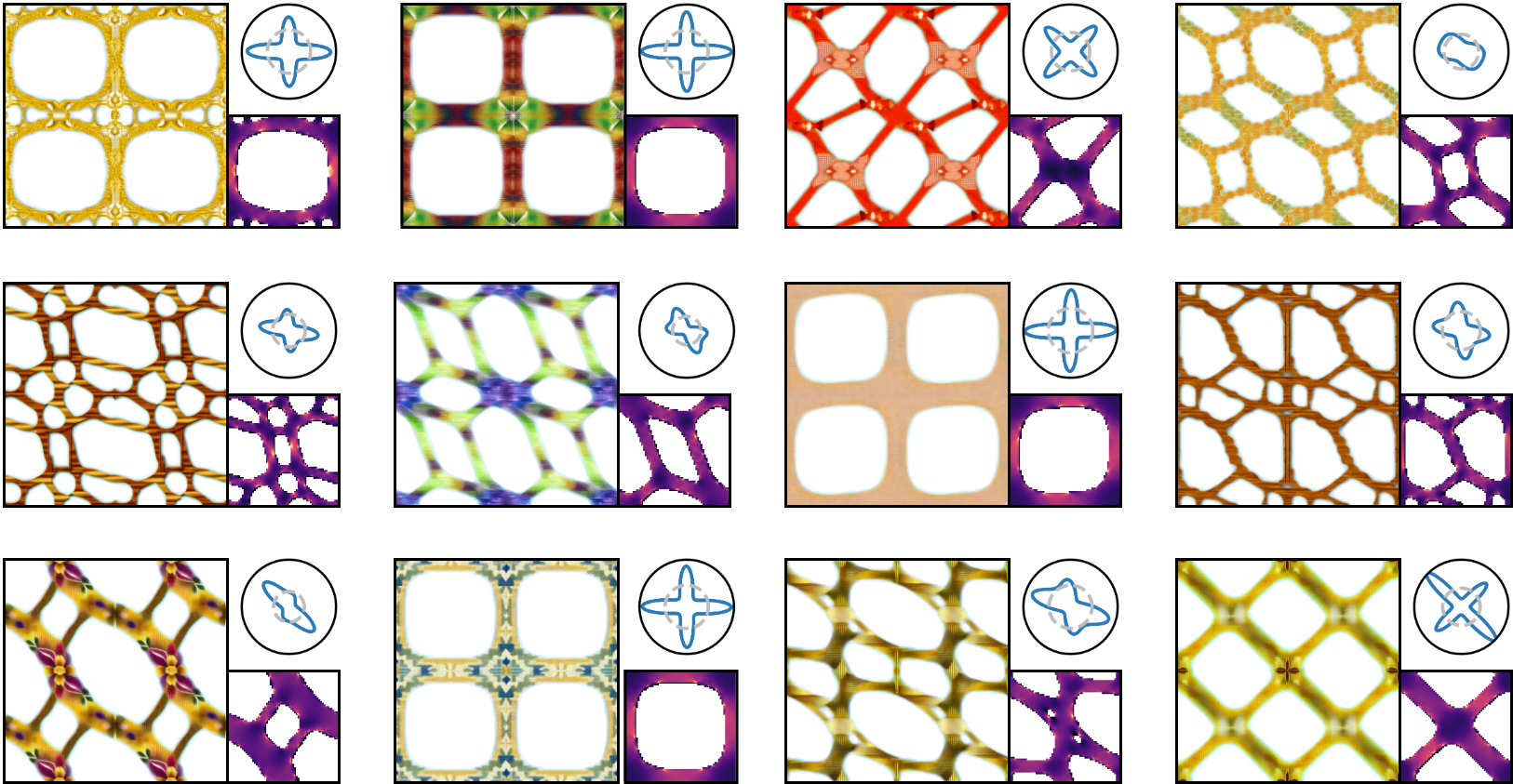}
    \caption[Visualization of Topology Design under $p2$ Symmetry]{\textbf{Visualization of topology design under $p2$ symmetry.} Visualization settings are the same as in \cref{fig:sdto_p1}.}
    \label{fig:sdto_p2}
    \vspace{0.5cm}
    \includegraphics[width=0.7\linewidth]{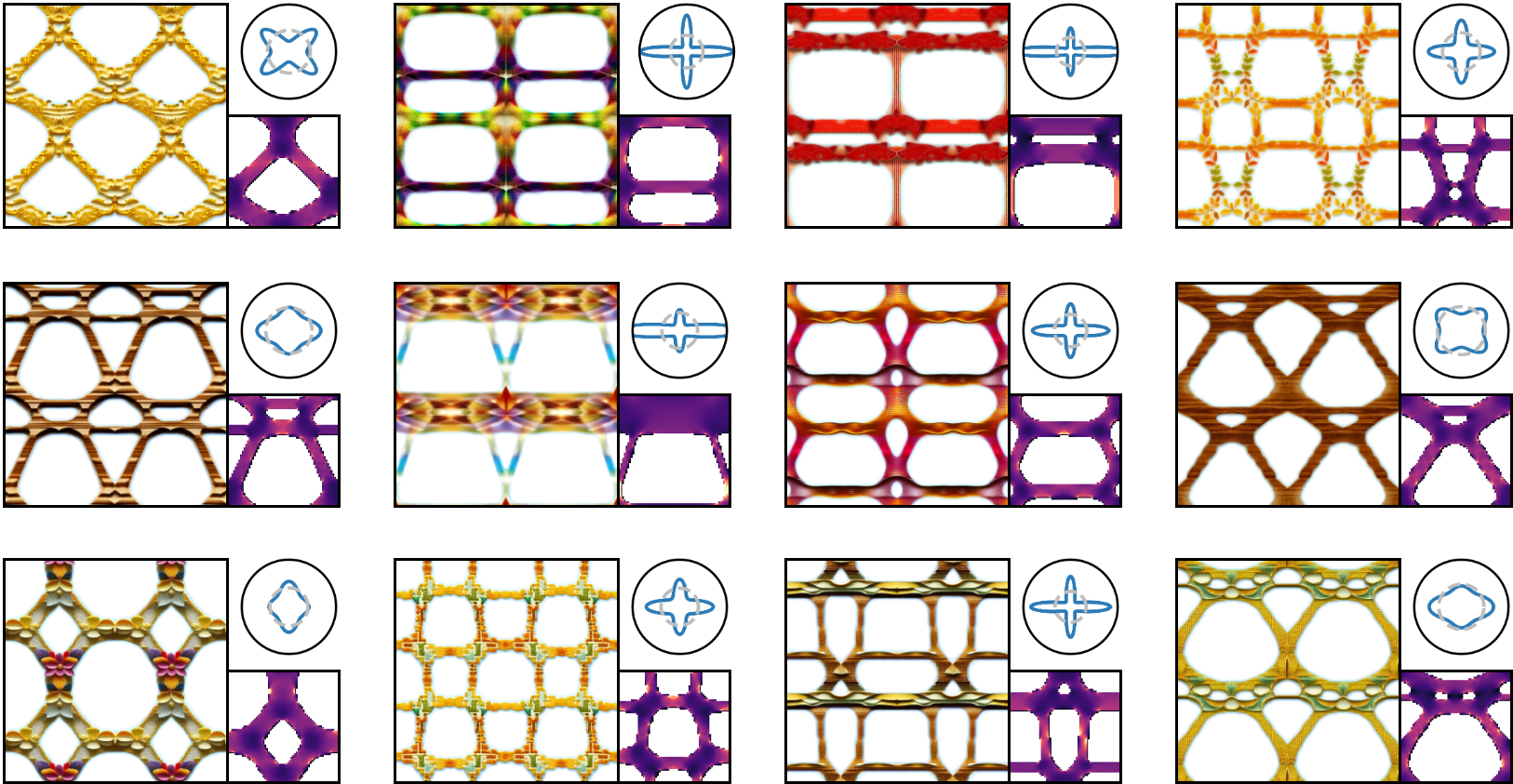}
    \caption[Visualization of Topology Design under $pm$ Symmetry]{\textbf{Visualization of topology design under $pm$ symmetry.} Visualization settings are the same as in \cref{fig:sdto_p1}.}
    \label{fig:sdto_pm}
\end{figure}

\clearpage
\section{Table of Affine Reflection Supergroups and Secondary Invariants}
\label{sec:sec_inv_table}

The symmetric continuous representation in \cref{sec:formulation} relies on the Hironaka-type decomposition: a $G$-invariant field can be written as a linear combination of a finite set of fixed $G$-invariant basis functions (secondary invariants) with coefficient fields that enjoy a higher affine reflection symmetry $W_a$ (cf. \cref{eq:hironaka_param,thm:hironaka_decomposition_func}). For practical use, the only group-dependent ingredient is the explicit choice of these basis functions $\{\eta_i\}_{i=1}^r$, where $r=[W_a\!:\!G]$. This section tabulates non-trivial $\eta_i$ for all planar groups (we omit trivial $\eta_1 = 1$), together with a compatible embedding $G\subset W_a$ and the associated lattice generators $(\mathbf a,\mathbf b)$ of $W_a$. The table serves as a plug-in recipe: once the target symmetry group $G$ and lattice are fixed, we directly obtain $(W_a,\mathbf a,\mathbf b)$ and the corresponding $\eta_i$, and then parameterize $G$-symmetric continuous fields via \cref{eq:hironaka_param}.

For the oblique groups $p1$ and $p2$, embedding into an affine reflection supergroup may require a linear coordinate transformation that maps the oblique lattice to a rectangular one, after which $W_a$ can be taken as $p2mm$. This transformation is used
only in the theoretical construction and is not applied to the
parameterization in our implementation. Instead, we retain the
original lattice coordinates and directly impose the corresponding
conjugated reflections on the coefficient fields, which act as skew
reflection symmetries in the original coordinates.

All notations in \cref{sec:sec_inv_table} follow the same conventions as in
\cref{sec:formulation}. In particular, $c_i$ and $s_i$ denote the cosine and sine associated with the lattice vectors. For the hexagonal lattice, the secondary invariants appearing in the table are defined as
\begin{align*}
\phi_1^{-+} &= s_1 + s_2 - (c_1 s_2 + c_2 s_1), \\
\phi_2^{--} &= s_1 - s_2 + c_1 s_2 - c_2 s_1 + 2(c_1 - c_2)(c_1 s_2 + c_2 s_1).
\end{align*}

\begin{table}[h]
\centering
\caption{Table of Affine Reflection Supergroups and Secondary Invariants}
\renewcommand{\arraystretch}{1.2}
\resizebox{0.9\linewidth}{!}{
\begin{tabular}{llllllllc}
\toprule
Lattice & $G$ & $W_a$ & $\mathbf{a}$ & $\mathbf{b}$ & $\eta_2$ & $\eta_3$ & $\eta_4$ & $r$ \\
\midrule
Oblique     & $p1$   & \multirow{9}{*}{$p2mm$} & $\mathbf{a}$     & $\mathbf{b}$     & $s_1$         & $s_2$                 & $s_1s_2$                           & $4$ \\
Oblique     & $p2$   &                         & $\mathbf{a}$     & $\mathbf{b}$     & $s_1s_2$       & $-$                   & $-$                                & $2$ \\
Rectangular & $pm$   &                         & $\mathbf{a}$     & $\mathbf{b}$     & $s_{1}$        & $-$                   & $-$                                & $2$ \\
Rectangular & $pg$   &                         & $\mathbf{a}$     & $\mathbf{b}/2$   & $s_{1}c_{2}$    & $s_{1}s_{2}$          & $c_{2}s_{2}$                       & $4$ \\
Rectangular & $cm$   &                         & $\mathbf{a}/2$   & $\mathbf{b}/2$   & $c_{1}c_{2}$    & $c_{1}s_{2}$          & $c_{2}s_{2}$                       & $4$ \\
Rectangular & $p2mm$ &                         & $\mathbf{a}$     & $\mathbf{b}$     & $-$            & $-$                   & $-$                                & $1$ \\
Rectangular & $p2mg$ &                         & $\mathbf{a}/2$   & $\mathbf{b}$     & $s_{1}s_{2}$    & $-$                   & $-$                                & $2$ \\
Rectangular & $p2gg$ &                         & $\mathbf{a}/2$   & $\mathbf{b}/2$   & $c_{1}c_{2}$    & $c_{1}s_{1}s_{2}$     & $s_{1}c_{2}s_{2}$                  & $4$ \\
Rectangular & $c2mm$ &                         & $\mathbf{a}/2$   & $\mathbf{b}/2$   & $c_{1}c_{2}$            & $-$          & $-$                                & $2$ \\
\midrule
Square      & $p4$   & \multirow{3}{*}{$p4mm$} & $\mathbf{a}$     & $\mathbf{b}$     & $(c_1-c_2)s_1s_2$ & $-$                 & $-$                                & $2$ \\
Square      & $p4gm$ &                         & $\mathbf{a}/2$   & $\mathbf{b}/2$   & $c_{1}c_{2}$    & $(c_{1}-c_{2})s_{1}s_{2}$ & $(c_{1}-c_{2})c_{1}c_{2}s_{1}s_{2}$ & $4$ \\
Square      & $p4mm$ &                         & $\mathbf{a}$     & $\mathbf{b}$     & $-$            & $-$                   & $-$                                & $1$ \\
\midrule
Hexagonal   & $p3$   & \multirow{5}{*}{$p6mm$} & $\mathbf{a}$     & $\mathbf{b}$     & $\phi_1^{-+}$   & $\phi_2^{--}$         & $\phi_3^{+-}=\phi_1^{-+}\phi_2^{--}$ & $4$ \\
Hexagonal   & $p3m1$ &                         & $\mathbf{a}$     & $\mathbf{b}$     & $\phi_1^{-+}$   & $-$                   & $-$                                & $2$ \\
Hexagonal   & $p31m$ &                         & $\mathbf{a}$     & $\mathbf{b}$     & $\phi_2^{--}$   & $-$                   & $-$                                & $2$ \\
Hexagonal   & $p6$   &                         & $\mathbf{a}$     & $\mathbf{b}$     & $\phi_3^{+-}$   & $-$                   & $-$                                & $2$ \\
Hexagonal   & $p6mm$ &                         & $\mathbf{a}$     & $\mathbf{b}$     & $-$            & $-$                   & $-$                                & $1$ \\
\bottomrule
\end{tabular}
}
\end{table}

\clearpage

\end{document}